\documentclass[11pt]{article}
\usepackage[letterpaper, margin=1in]{geometry}
\usepackage[parfill]{parskip}
\usepackage{amsmath,amsthm,amssymb,bbm}
\usepackage{mathtools}
\usepackage{cases}
\usepackage{dsfont}
\usepackage{microtype}
\usepackage{tablefootnote}

\allowdisplaybreaks

\usepackage{subfigure}
\usepackage{algorithm,algorithmic}
\usepackage{color}
\usepackage{appendix}

\usepackage{bm}
\usepackage{color}
\usepackage{booktabs}       
\usepackage{appendix}
\usepackage{authblk}
\usepackage{comment}

\usepackage{url}
\usepackage[authoryear]{natbib}
\usepackage{hyperref}
\usepackage{xcolor}
\hypersetup{
    colorlinks,
    linkcolor={blue!50!black},
    citecolor={blue!50!black},
}
\colorlet{linkequation}{blue}

\usepackage{cleveref}
\crefformat{equation}{(#2#1#3)}
\crefrangeformat{equation}{(#3#1#4) to~(#5#2#6)}
\crefname{equation}{}{}
\Crefname{equation}{}{}

\crefname{definition}{\textbf{definition}}{definitions}
\Crefname{definition}{Definition}{Definitions}
\crefname{assumption}{\textbf{assumption}}{assumptions}
\Crefname{assumption}{Assumption}{Assumptions}

\definecolor{maroon}{RGB}{192,80,77}
\definecolor{mypink3}{cmyk}{0, 0.7808, 0.4429, 0.1412}

\newcommand{\explain}[2]{\underset{\mathclap{\overset{\uparrow}{#2}}}{#1}}

\newtheorem{theorem}{Theorem}[section]
\newtheorem{lemma}[theorem]{Lemma}
\newtheorem{proposition}[theorem]{Proposition}

\newtheorem{corollary}[theorem]{Corollary}

\newtheorem{remark}[theorem]{Remark}
\newtheorem{assumption}[theorem]{Assumption}

\usepackage{amsmath}

\newcommand\norm[1]{\left\lVert#1\right\rVert}
\newcommand{\mypink}{\textcolor{mypink3}}

\newcommand{\argmax}{\mathop{\mathrm{argmax}}}

\newcommand{\defeq}{\mathrel{\mathop:}=}

\def\tin{\text{in}}
\def\PP{\textit{\textbf{P}}}
\def\OPDVR{\emph{\textbf{OPDVR}}}
\def\hpi{\widehat{\pi}}
\def\E{\mathbb{E}}
\def\P{\mathbb{P}}

\def\Var{\mathrm{Var}}

\def\R{\mathbb{R}}
\def\cA{\mathcal{A}}

\def\cD{\mathcal{D}}

\def\cS{\mathcal{S}}

\def\d{\textit{\textbf{d}}}
\def\PP{\textit{\textbf{P}}}

\begin{document}


\title{Near-Optimal Offline Reinforcement Learning via \\ Double Variance Reduction}
	

\author[1,3]{Ming Yin}
\author[2]{Yu Bai}
\author[3]{Yu-Xiang Wang}
\affil[1]{Department of Statistics and Applied Probability, UC Santa Barbara}
\affil[2]{Salesforce Research}
\affil[3]{Department of Computer Science, UC Santa Barbara}
\affil[ ]{\texttt{ming\_yin@ucsb.edu}  \quad \texttt{yu.bai@salesforce.com} \quad
	\texttt{yuxiangw@cs.ucsb.edu}}

\date{}
\maketitle


\begin{abstract}


We consider the problem of offline reinforcement learning (RL)  --- a well-motivated setting of RL that aims at policy optimization using only historical data. Despite its wide applicability, theoretical understandings of offline RL, such as its optimal sample complexity, remain largely open even in basic settings such as \emph{tabular} Markov Decision Processes (MDPs). 
 In this paper, we propose \emph{Off-Policy Double Variance Reduction} (OPDVR), a new variance reduction based algorithm for offline RL. Our main result shows that OPDVR provably identifies an $\epsilon$-optimal policy with $\widetilde{O}(H^2/d_m\epsilon^2)$ episodes of offline data in the finite-horizon \emph{stationary transition} setting, where $H$ is the horizon length and $d_m$ is the minimal marginal state-action distribution induced by the behavior policy. This improves over the best known upper bound by a factor of $H$. Moreover, we establish an information-theoretic lower bound of $\Omega(H^2/d_m\epsilon^2)$ which certifies that OPDVR is optimal up to logarithmic factors.  
Lastly, we show that OPDVR also achieves rate-optimal sample complexity under alternative settings such as the finite-horizon MDPs with non-stationary transitions and the infinite horizon MDPs with discounted rewards.

\end{abstract}
      

\section{Introduction}
\label{sec:introduction}




Offline reinforcement learning (offline RL, also known as batch RL) aims at learning the near-optimal policy by using a static offline dataset that is collected by a certain behavior policy $\mu$~\citep{lange2012batch}. As offline RL agent works without needing to interact with the environment, it is more widely applicable to problems where online interaction is infeasible, \emph{e.g.} when trials-and-errors are expensive (robotics, education), risky (autonomous driving) or even unethical (healthcare) \citep[see,e.g., a recent survey ][]{levine2020offline}. 

Despite its practical significance, a precise theoretical understanding of offline RL has been lacking. Previous sample complexity bounds for RL has primarily focused on the online setting~\citep{azar2017minimax,jin2018q,bai2019provably,zanette2019tighter,simchowitz2019non,efroni2019tight,dann2015sample,cai2019provably} or the generative model (simulator) setting~\citep{azar2013minimax,sidford2018near,sidford2018variance,yang2019sample,agarwal2019optimality,wainwright2019variance,lattimore2019learning}, both of which assuming interactive access to the environment and not applicable to offline RL. On the other hand, the sample complexity of offline RL remains unsettled even for environments with finitely many state and actions, a.k.a, the tabular MDPs  (Markov Decision Processes).
 One major line of work is concerned with the off-policy evaluation (OPE) problem~\citep{li2015toward,jiang2016doubly,liu2018breaking,kallus2019double,kallus2019efficiently,uehara2019minimax,xie2019towards,yin2020asymptotically,duan2020minimax}. These works provide sample complexity bounds for evaluating the performance of a fixed policy, and do not imply guarantees for policy optimization. Another line of work studies the sample complexity of offline policy optimization in conjunction with function approximation~\citep{chen2019information,xie2020q,xie2020batch,jin2020pessimism}. These results apply to offline RL with general function classes, but when specialized to the tabular setting, they give rather loose sample complexity bounds with suboptimal dependencies on various parameters \footnote{See Table~\ref{table} for a clear comparison.}. 

The recent work of~\citet{yin2020near} showed that the optimal sample complexity for finding an $\epsilon$-optimal policy in offline RL is $\widetilde{O}(H^3/d_m\epsilon^2)$ in the finite-horizon \emph{non-stationary}\footnote{This is also known as the finite-horizon \emph{time-inhomogeneous} or \emph{time-varying} setting, where the transition dynamics and rewards could differ by time steps. 
The stationary case is expected to be easier in the information-theoretical sense, but is more challenging to analyze due to the more complex dependence structure in the observed data.
} setting (with matching upper and lower bounds), where $H$ is the horizon length and $d_m$ is a constant related to the data coverage of the behavior policy in the given MDP. However, the optimal sample complexity in alternative settings such as stationary transition or infinite-horizon settings remains unknown. 
Further, the $\widetilde{O}(H^3/d_m\epsilon^2)$ sample complexity is achieved by an off-policy evaluation + uniform convergence type algorithm; other more practical algorithms including (stochastic) optimal planning algorithms such as Q-Learning are not well understood in offline RL. This motivates us to ask:
\begin{quote}
  \emph{What algorithm achieves the optimal sample complexity for offline RL in tabular MDPs?}
\end{quote}
\vspace{-1em}




\paragraph{Our Contributions}
In this paper, we propose an algorithm \OPDVR{} (Off-Policy Doubled Variance Reduction) for offline reinforcement learning based on an extension of the variance reduction technique initiated in~\citep{sidford2018near,yang2019sample}. \OPDVR{} performs stochastic (minibatch style) value iterations using the available offline data, and can be seen as a version of stochastic optimal planning that interpolates value iteration and Q-learning. Our main contributions are summarized as follows.

\begin{itemize}
	
	\item 
	We show that \OPDVR{} finds an $\epsilon$-optimal policy with high probability using $\widetilde{O}(H^2/d_m\epsilon^2)$ episodes of offline data (Section~\ref{sec:stationary}).
	This improves upon the best known sample complexity by an $H$ factor and to the best of our knowledge is the first that achieves an $O(H^2)$ horizon dependence  offlinely, thus formally separating the stationary case with the non-stationary case for offline RL. 
	
	\item We establish a sample (episode) complexity lower bound $\Omega(H^2/d_m\epsilon^2)$ for offline RL in the finite-horizon stationary setting (Theorem~\ref{thm:learn_low_bound_s}), showing that the sample complexity of \OPDVR{} is optimal up to logarithmic factors.
	
	\item 	In the finite-horizon non-stationary setting,  and infinite horizon $\gamma$-discounted setting, we show that \OPDVR{} achieves $\widetilde{O}(H^3/d_m\epsilon^2)$ sample (episode) complexity  (Section~\ref{sec:finite_horizon1}) and $\widetilde{O}((1-\gamma)^{-3}/d_m\epsilon^2)$ sample (step) complexity (Section~\ref{sec:infinite}) respectively.  They are both optimal up to logarithmic factors and our infinite-horizon result improves over the best known results, e.g., those derived for the fitted Q-iteration style algorithms \citep{xie2020q}.
	
	
	
	
	
	\item On the technical end, our algorithm presents a sharp analysis of offline RL with stationary transitions, and uses a doubling technique
	to resolve the initialization dependence in the original variance reduction algorithm (e.g. of~\citep{sidford2018near}), both of which could be of broader interest.
\end{itemize}

\begin{table*}
	\caption{Comparison of sample complexities for tabular offline RL interpretation. }
	\label{table}
	\centering
	\resizebox{\linewidth}{!}{%
		\begin{tabular}{llp{1.7in}l}
			\toprule
			\cmidrule(r){1-4}
			Method/Analysis     &Setting& Assumptions    & Sample complexity$^a$ \\
			\midrule
			BFVT \citep{xie2020batch} & $\infty$-horizon& only realizability $+$ MDP concentrability$^b$ & $\tilde{O}((1-\gamma)^{-8}C^2/\epsilon^4)$ \\
			MBS-PI/QI \citep{liu2020provably} &$\infty$-horizon &completeness$+$bounded density estimation error&$\tilde{O}((1-\gamma)^{-8}C^2/\epsilon^2)$\\
			\cite{le2019batch}     & $\infty$-horizon &Full Concentrability & $\tilde{O}((1-\gamma)^{-6}\beta_\mu/\epsilon^2)$\\

			FQI \citep{chen2019information}   & $\infty$-horizon &Full Concentrability & $\tilde{O}((1-\gamma)^{-6}C/\epsilon^2)$ \\
			MSBO/MABO \citep{xie2020q} & $\infty$-horizon &Full Concentrability  & $\widetilde{O}((1-\gamma)^{-4}C_\mu/ \epsilon^2)$   \\
			OPEMA \citep{yin2020near} &$H$-horizon non-stationary &Full Concentrability &$\widetilde{O}(H^3/d_m \epsilon^2)$   \\
			OPDVR (Section~\ref{sec:finite_horizon1}) &$H$-horizon non-stationary &Weak Coverage &$\widetilde{O}(H^3/d_m \epsilon^2)$   \\
			OPDVR (Section~\ref{sec:finite_horizon2}) &$H$-horizon stationary &Weak Coverage &$\widetilde{O}(H^2/d_m \epsilon^2)$   \\
			OPDVR (Section~\ref{sec:infinite})  &$\infty$-horizon &Weak Coverage &$\widetilde{O}((1-\gamma)^{-3}/d_m \epsilon^2)$   \\
			\bottomrule
		\end{tabular}
	}
\begin{flushleft}
\footnotesize{$^a$ Number of episodes in the finite horizon setting and number of steps in the infinite horizon.}\\
\footnotesize{$^b$ To compare concentrability parameters, use $\beta_\mu=C \geq  C_\mu \geq  1/d_m$. See Assumption~\ref{assu2} and also Section~\ref{sec:comparison_dm} for discussion.} 
\end{flushleft}
\end{table*}

\paragraph{Related work.} There is a large and growing body of work on the theory of offline RL and RL in general.  We could not hope to provide a comprehensive survey, thus will instead highlight the few prior work that we depend upon on the technical level.  The variance reduction techniques that we use in this paper builds upon  the work of \citep{sidford2018near} in the generative model setting, though it is nontrivial in adapting their techniques to the offline setting; and our two-stage variance reduction appears essential for obtaining optimal rate for $\epsilon>1$ (see Section~\ref{sec:discussion} and Appendix~\ref{sec:sidford_proof} for more detailed discussions). We also used a fictitious estimator technique that originates from the OPE literature\citep{xie2019towards,yin2020asymptotically}, but extended it to the stationary-transition case, and to the policy optimization problem.  As we mentioned earlier, the optimal sample complexity in offline RL in the tabular MDPs with stationary transitions was not settled. The result of \citep{yin2020near} is optimal in the non-stationary case, but is suboptimal by a factor of $H$ in the stationary case.  Our lower bound is a variant of the construction of \citep{yin2020near} that applies to the stationary case.  Other existing work on offline RL has even weaker parameters (sometimes due to their setting being more general, see details in Table~\ref{table}).  We defer more detailed discussion  related to  the OPE literature and online RL / generative model literature  to Appendix~\ref{sec:related} due to space constraint.

\paragraph{Additional paper organization}
We present the problem setup in Section~\ref{sec:preliminaries}, present some discussions related to our algorithm in Section~\ref{sec:discussion}, and conclude in Section~\ref{sec:conclusion}. Proofs and some additional technical materials are deferred to the Appendix.

\section{Preliminaries}
\label{sec:preliminaries}


We consider reinforcement learning problems modeled by finite Markov Decision Processes (MDPs) (we focus on the finite-horizon episodic setting, and defer the infinite-horizon discounted setting to Section~\ref{sec:infinite}.) An MDP is denoted by a tuple $M=(\mathcal{S},\mathcal{A},r,T,d_1,H)$, where $\mathcal{S}$ and $\mathcal{A}$ are the state and action spaces with finite cardinality $|\mathcal{S}|=S$ and $|\mathcal{A}|=A$. $P_t : \mathcal{S}\times \mathcal{A}\times \mathcal{S} \rightarrow [0,1]$ is the transition kernel with $P_t (s^\prime |s, a)$ be the probability of entering state $s'$ after taking action $a$ at state $s$. We consider both the stationary and non-stationary transition setting: The stationary transition setting (\emph{e.g.}~\citet{dann2017unifying}) assumes $P_t\equiv P$ is identical at different time steps, and the non-stationary transition setting~\citep{jiang2017contextual,xie2019towards} allows $P_t$ to be different for different $t$. $r_t : \mathcal{S} \times \mathcal{A} \rightarrow [0,1]$ is the reward function which we assume to be deterministic\footnote{This is commonly assumed in the RL literature. The randomness in the reward will only cause a lower order error (than the randomness in the transition) for learning.}. $d_1$ is the initial state distribution, and $H$ is the time horizon. A (non-stationary) policy $\pi : \mathcal{S}  \rightarrow \P_\mathcal{A}^H$ assigns to each state $s_t \in \mathcal{S}$ a distribution over actions at each time $t$, \emph{i.e.} $\pi_t(\cdot|s_t)$ is a probability distribution with dimension $S$.\footnote{Note even for stationary transition setting, the policy itself can be non-stationary.} We use $d^\pi_t(s,a)$ or $d^\pi_t(s)$ to denote the marginal state-action/state distribution induced by policy $\pi$ at time $t$, i.e.
$
  d^\pi_t(s) \defeq \P^\pi(s_t = s)~~~{\rm and}~~~d^\pi_t(s,a) \defeq \P^\pi(s_t = s, a_t=a).
$
\paragraph{$Q$-value and Bellman operator.} For any policy $\pi$ and any fixed time $t$, the value function $V_t^\pi(\cdot)\in\R^S$ and $Q$-value function $Q_t^\pi(\cdot,\cdot)\in\R^{S\times A}$, $\forall s,a$ is defined as:
{\scriptsize
\begin{align*}
V_t^\pi(s)=\E\left[\sum_{i=t}^H r_i\middle|s_t=s\right],
Q_t^\pi(s,a)=\E\left[\sum_{i=t}^H r_i\middle|s_t,a_t=s,a\right]
\end{align*} }\vspace{-0.5cm}

For the ease of exposition, we always enumerate $Q^\pi$ as a column vector and similarly for $P_t(\cdot|s,a)$. Moreover, for any vector $Q\in\R^{S\times A}$, the induced value vector and policy is defined in the greedy way:
$
\forall s_t\in\mathcal{S}, V_Q(s_t)=\max_{a_t\in\mathcal{A}}Q(s_t,a_t),$  $\pi_Q(s_t)=\argmax_{a_t\in\mathcal{A}}Q(s_t,a_t).
$
Given an MDP, for any vector $V\in\mathbb{R}^\mathcal{S}$ and any deterministic policy $\pi$, $\forall t\in[H]$ the Bellman operator $\mathcal{T}_t^\pi:\mathbb{R}^\mathcal{S}\rightarrow\mathbb{R}^\mathcal{S}$ is defined as: $[\mathcal{T}_t^\pi(V)](s):=r(s,\pi_t(s))+P_t^\top(\cdot|s,\pi_t(s))V$, and the corresponding Bellman optimality operator $\mathcal{T}_t:\mathbb{R}^\mathcal{S}\rightarrow\mathbb{R}^\mathcal{S}$, $[\mathcal{T}_t(V)](s):=\max_{a\in\mathcal{A}}[r(s,a)+P_t^\top(\cdot|s,a)V]$. Lastly, for a given value function $V_{t}$, we define backup function $z_t(s_t,a_t) := P_t^\top(\cdot|s_t,a_t)V_{t+1}$ and the one-step variance as $\sigma_{V_{t+1}}(s_{t},a_{t}):=\Var_{s_{t+1}}[V_{t+1}(s_{t+1})|s_{t},a_{t}]$.

\subsection{Offline learning problem}
In this paper we investigate the offline learning problem, where we do not have interactive access to the MDP, and can only observe a static dataset $\mathcal{D}=\left\{(s^{(i)}_t,a^{(i)}_t,r^{(i)}_t,s^{(i)}_{t+1})\right\}_{i\in[n]}^{t\in[H]}$.We assume that $\mathcal{D}$ is obtained by executing a pre-specified \emph{behavior policy} $\mu$ (also known as the \emph{logging policy}) for $n$ episodes and collecting the trajectories $\tau^{(i)}=(s_1^{(i)},a_1^{(i)}, r_1^{(i)},\dots,s_H^{(i)},a_H^{(i)},r_H^{(i)},s_{H+1}^{(i)})$, where each episode is rendered in the form: $s_1^{(i)}\sim d_1$, $a_t^{(i)}\sim \mu_t(\cdot|s_t^{(i)})$, $r_t^{(i)}=r(s_t^{(i)}, a_t^{(i)})$, and $s_{t+1}^{(i)}\sim P_t(\cdot|s_t^{(i)}, a_t^{(i)})$.
Given the dataset $\mathcal{D}$, our goal is to find an {\bf $\epsilon$-optimal policy} $\pi_\text{out}$, in the sense that $||V_1^{\pi^\star}-V_1^{\pi_\text{out}}||_\infty<\epsilon$.


\paragraph{Assumption on data coverage}
Due to the curse of distributional shift, efficient offline RL is only possible under certain data coverage properties for the behavior policy $\mu$. Throughout this paper we assume the following:
\begin{assumption}[Weak coverage]
  \label{assu2}
  The behavior policy $\mu$ satisfies the following: There exists \emph{some} optimal policy $\pi^\star$ such that $d_{t'}^\mu(s_{t'},a_{t'})>0$ if there exists $t<t'$ such that $d^{\pi^\star}_{t:t'}(s_{t'},a_{t'}|s_t,a_t)>0$, where $d^{\pi^\star}_{t:t'}(s_{t'},a_{t'}|s_t,a_t)$ is the conditional multi-step transition probability from step $t$ to $t'$.
\end{assumption}
Intuitively, Assumption~\ref{assu2} requires $\mu$ to ``cover'' certain optimal policy $\pi^\star$, in the sense that any $s_{t'},a_{t'}$ is reachable by $\mu$ if it is attainable from a previous state-action pair by $\pi^\star$.  It is similar to \citep[Assumption 1]{liu2019off}. Note that this is weaker than the standard ``concentrability'' assumption~\citep{munos2003error,le2019batch,chen2019information}: Concentrability defines $\beta_\mu:=\sup_{\pi\in\Pi}||d^\pi(s_t,a_t)/d^\mu(s_t,a_t)||_\infty<\infty$ (cf.~\citep[Assumption 1 \& Example 4.1]{le2019batch}), which requires the sufficient exploration for tabular case\footnote{Note \citet{xie2020q} has a tighter concentration coefficient with $C_\mu:=\max_{\pi\in\Pi}\norm{w_{d_\pi/\mu}}^2_{2,\mu}$ but it still requires full exploration when $\Pi$ contains all policies.} since we optimize over all policies (see Section~\ref{sec:comparison_dm} for a discussion). In contrast, our assumption only requires $\mu$ to ``trace'' one single optimal policy.\footnote{Nevertheless, we point out that function approximation$+$concentrability assumption is powerful for handling realizability/agnostic case and related concepts (\emph{e.g.} inherent Bellman error) and easier to scale up to general settings.} 


With Assumption~\ref{assu2}, we define 
{\small
\begin{align}
  \label{eqn:definition-dm}
  d_m:=\min_{t,s_t,a_t}\{d^\mu_t(s_t,a_t):d^\mu_t(s_t,a_t)>0\},
\end{align}
}which is decided by the behavior policy $\mu$ and is an intrinsic quantity required by offline learning (see Theorem~G.2 in \citet{yin2020near}). Our sample complexity bounds will depend on $1/d_m$ and in general $d_m$ is unknown. Yet, we assume $d_m$ is \emph{known} for the moment and will utilize the knowledge of $d_m$ in our algorithms. Indeed, in Lemma~\ref{thm:est_d_m}, we show that estimating $d_m$ (using on-policy Monte Carlo estimator) up to a multiplicative factor only requires $\widetilde{O}(1/d_m)$ episodes of offline data; replacing the exact $d_m$ with this estimator suffices for our purpose and, importantly, will not affect our downstream sample complexities.

\section{Variance reduction for offline RL}
\label{sec:finite_horizon1}


In this section, we introduce our main algorithm Off-Policy Double Variance Reduction (\OPDVR{}), and present its theoretical guarantee in the finite-horizion non-stationary setting.

\subsection{Review: variance reduction for RL}
We begin by briefly reviewing the variance reduction algorithm for online reinforcement learning, where we have the interactive access to the environment.

Variance reduction (VR) initially emerged as a technique for obtaining fast convergence in large scale optimization problems, for example in the Stochastic Variance Reduction Gradient method (SVRG,~\citep{johnson2013accelerating,zhang2013linear}). This technique is later brought into reinforcement learning for handling policy evaluation~\citep{du2017stochastic} and policy optimization problems~\citep{sidford2018variance,sidford2018near,yang2019sample,wainwright2019variance,sidford2020solving,li2020sample,zhang2020almost}.

In the case of policy optimization, VR is an algorithm that approximately iterating the Bellman optimality equation, using an inner loop that performs an approximate value (or Q-value) iteration using fresh interactive data to estimate $V^\star$, and an outer loop that performs multiple steps of such iterations to refine the estimates. Concretely, to obtain an reliable $Q_t(s,a)$ for some step $t\in[H]$, by the Bellman equation $Q_t(s,a)=r(s,a)+P_t^\top(\cdot|s,a)V_{t+1}$, we need to estimate $P_t^\top(\cdot|s,a)V_{t+1}$ with sufficient accuracy. VR handles this by decomposing:
{\small\begin{equation}\label{eq:VR_decomposition}
P_t^\top(\cdot|s,a)V_{t+1}=P_t^\top(\cdot|s,a)(V_{t+1}-V_{t+1}^{\text{in}})+P_t^\top(\cdot|s,a)V_{t+1}^{\text{in}},
\end{equation}
}where $V_{t+1}^{\text{in}}$ is a \emph{reference} value function obtained from previous calculation (See line~4,13 in the inner loop of Algorithm~\ref{alg:OPDVRT}) and $P_t^\top(\cdot|s,a)(V_{t+1}-V_{t+1}^{\text{in}})$, $P_t^\top(\cdot|s,a)V_{t+1}^{\text{in}}$ are estimated separately at different stages. This technique can help in reducing the ``effective variance'' along the learning process (see~\citet{wainwright2019variance} Section~2 for a discussion).

In addition, in order to translate the guarantees from learning values to learning policies\footnote{Note in general, direct translation of learning a $\epsilon$-optimal value to $\epsilon$-optimal policy will cause additional suboptimal complexity dependency of $H$. }, we build on the following ``monotonicity property'': For any policy $\pi$ that satisfies the monotonicity condition $V_t\leq \mathcal{T}_{\pi_t}V_{t+1}$ for all $t\in[H]$, the performance of $\pi$ is sandwiched as $V_t\leq V^\pi_t\leq V^\star_t$, i.e. $\pi$ is guaranteed to perform the same or better than $V_t$. This property is first captured by~\citep{sidford2018near} (for completeness we provide a proof in Lemma~\ref{lem:mono}), and later reused by \citet{yang2019sample,sidford2020solving} under different settings. We rely on this property in our offline setting as well for providing policy optimization guarantees.


\subsection{OPDVR: variance reduction for offline RL}
We now explain how we design the VR algorithm in the offline setting. Even though our primary novel contribution is for the stationary case (Theorem~\ref{thm:main_stationary}), we begin with non-stationary setting for the ease of explaining algorithmic design. We let $\iota:=\log(HSA/\delta)$ as a short hand.

\begin{algorithm}[tb]
	\caption{OPVRT: A Prototypical Off-Policy Variance Reduction Template}
	\label{alg:OPDVRT}
	\small{
		\begin{algorithmic}[1]
			\STATE {\bfseries Functional input:}  Integer valued function $\mathbf{m}:  \R_+ \rightarrow \mathbb{N}$.   Off-policy estimator $\mathbf{z}_t,\mathbf{g}_t$ in function forms that provides lower confidence bounds (LCB) of the two terms in the bootstrapped value function \eqref{eq:VR_decomposition}.
			
			\STATE {\bfseries Static input:} Initial value function $V_t^{(0)}$ and $\pi_t^{(0)}$  (which satisfy $V^{(0)}_t\leq \mathcal{T}_{\pi_t^{(0)}}V_{t+1}^{(0)}$ and $V_{H+1}^{(0)}\equiv 0$).
			A scalar $u^{(0)}$ satisfies $u^{(0)}\geq \sup_t||V_t^\star-V_t^{(0)}||_\infty$.   Outer loop iterations $K$. 
			Offline dataset $\mathcal D = \{\{s_t^{(i)},a_t^{(i)},r_t^{(i)}\}_{t = 1}^{H }\}_{i = 1}^{n}$ from the behavior policy $\mu$ as a data-stream where
			$n \geq \sum_{i=1}^K 2 \cdot \mathbf{m}(u^{(0)} \cdot 2^{-(i-1)}). $
			
			\STATE 			------------------\textsc{Inner loop} ---------------------\\
			\FUNCTION{\textsc{QVI-VR} ($\cD_1,\cD_2$, $V_t^{\text{in}}, \pi^{\text{in}}, \mathbf{z}_t,\mathbf{g}_t, u^\text{in}$)}
			\STATE  \mypink{$\diamond$ Computing reference with $\cD_1$:}
			\STATE Initialize $Q_t\leftarrow \mathbf{0}\in\R^{\mathcal{S}\times\mathcal{A}}$ for $t\in[H+1]$.\label{asdad}
			\FOR{$t\in[H]$ and each pair $(s_t,a_t)\in\mathcal{S}\times\mathcal{A}$}
			\STATE \mypink{$\diamond$ Compute an LCB of $P_t^\top(\cdot|s_t,a_t) V_{t+1}^{in}$:}
			\STATE $z_t\leftarrow  \mathbf{z}_t(\cD_1, V^{\text{in}}_{t+1}, u^\text{in} )$
			\ENDFOR
			\STATE \mypink{$\diamond$ Value Iterations with $\cD_2$:}
			\FOR{$t=H+1,H,...,1$}
			\STATE \mypink{$\diamond$ Update value function:}  $V_t=\max(V_{Q_{t}},V^{\text{in}}_t)$, 
			\STATE \mypink{$\diamond$ Update policy according to value function:} 
			\STATE $\forall s_t$, if $V_t(s_t)=V^{\text{in}}_t(s_t)$ set $\pi_t(s_t)=\pi_t^{\text{in}}(s_t)$; else set $\pi_t(s_t)=\pi_{Q_{t}}(s_t)$. 
			\IF{$t\geq 1$}
			\STATE \mypink{$\diamond$  LCB of $P^\top(\cdot|s_{t-1},a_{t-1})[V_{t}-V^{\text{in}}_{t}]$:} 
			\STATE $g_{t-1} \leftarrow  \mathbf{g}_{t-1}(\cD_2,V_{t}, V^{\text{in}}_{t},u^\text{in})$. 
			\STATE \mypink{$\diamond$ Update $Q$ function:}  $Q_{t-1}\leftarrow r+z_{t-1}+g_{t-1}$  
			\ENDIF
			\ENDFOR 
			\STATE {\bfseries Return:} $V_1,...,V_H$ and $\pi$
			\ENDFUNCTION
			\STATE 			------------------\textsc{outer loop} ---------------------\\
			\FOR{$i=1,...,K$}
			\STATE$m^{(i)} \rightarrow \mathbf{m}(u^{(i-1)})$
			\STATE Get $\cD_1$ and $\cD_2$ both of size $m^{(i)}$ from the stream $\cD$.
			\STATE $V^{(i)},\pi^{(i)}\leftarrow$\textsc{QVI-VR}($\cD_1,\cD_2,V_t^{(i-1)},\pi^{(i-1)},\mathbf{z}_t,\mathbf{g}_t,u^{(i-1)}$).
			\STATE $u^{(i)}\leftarrow u^{(i-1)}/2$.
			\ENDFOR		
			\STATE {\bfseries Output: $V^{(K)}$, $\pi^{(K)}$ } 
			
		\end{algorithmic}
	}
\end{algorithm}
\normalsize{}

\paragraph{Prototypical offline VR} 
We first describe a prototypical version of our offline VR algorithm in Algorithm~\ref{alg:OPDVRT}, which we will instantiate with different parameters  
\emph{twice} (hence the name``Double'') in each of the three settings of interest.

Algorithm~\ref{alg:OPDVRT} takes estimators $\mathbf{z}_t$ and $\mathbf{g}_t$ that produce lower confidence bounds (LCB) of the two terms in \eqref{eq:VR_decomposition} using offline data. 
Specifically, we assume $\mathbf{z}_t,\mathbf{g}_t$ are both available in \emph{function forms} in that they take an offline dataset (with an arbitrary size), fixed value function $V_{t+1}, V^{\text{in}}_{t+1}$ and an external scalar input $u$ then return $z_t, g_t\in \R^{\cS\times \cA}$. 
 $z_t,g_t$ satisfies that 
 \begin{align*}
 z_t(s_t,a_t) &\leq P^\top(\cdot|s_t,a_t)V^{\text{in}}_{t+1},\\
 g_t(s_t,a_t) &\leq P^\top(\cdot|s_{t},a_{t})[V_{t+1}-V^{\text{in}}_{t+1}],
 \end{align*}
 uniformly for all $s_t,a_t$  with high probability. 
 
Algorithm~\ref{alg:OPDVRT} then proceeds by taking the input offline dataset as a stream of iid sampled trajectories and use an exponentially increasing-sized batches of independent data to pass in $\mathbf{z}_t$ and $\mathbf{g}_t$ while updating the estimated $Q$ value function by applying the Bellman backup operator except that the update is based on a conservative and variance reduced estimated values. Each inner loop iteration backs up from the last time-step and update all $Q_t$ for $t=H,...,1$; and each outer loop iteration passes a new batch of data into the inner loop while ensuring reducing the suboptimality gap from the optimal policy by a factor of 2 in each outer loop iteration. 

Now let us introduce our estimators $\mathbf{z}_t$ and $\mathbf{g}_t$  in the finite-horizon non-stationary case  (the choices for the stationary case and the infinite-horizon case will be introduced later).   

Given an offline dataset $\mathcal{D}$, we define LCB $z_t(s_t,a_t) = \tilde{z_t}(s_t,a_t) - e(s_t,a_t)$ where $\tilde{z_t}(s_t,a_t)$ is an unbiased estimator and $e(s_t,a_t) = O(\sqrt{\tilde{\sigma}_{V^{\text{in}}_{t+1}}(s_t,a_t)})$ is an ``error bar'' that depends on $\tilde{\sigma}_{V^{\text{in}}_{t+1}}(s_t,a_t)$ ---  an estimator of the variance of $\tilde{z_t}(s_t,a_t)$.   $\tilde{z_t}(s_t,a_t)$ and $\tilde{\sigma}_{V^{\text{in}}_{t+1}}(s_t,a_t)$ are plug-in estimators at $(s_t, a_t)$ that use the available offline data $(r_t, s_{t+1}')$  to estimate the transition and rewards \emph{only if} the number of visitations to $(s_t, a_t)$ (denoted by $n_{s_t,a_t}$) is greater than a statistical threshold. Let $m$ be the episode budget, we write:
\begin{equation}\label{eqn:off_z}
\begin{aligned}
&\tilde{z}_t(s_t,a_t)=P^\top_t(\cdot|s_t,a_t) V^{\text{in}}_{t+1}\cdot \mathbf{1}(E^c_{m,t})
+ \frac{1}{n_{s_t,a_t}}\sum_{i=1}^m V^{\text{in}}_{t+1}(s^{(i)}_{t+1})\cdot\mathbf{1}_{[s^{(i)}_t,a^{(i)}_t=s_t,a_t]}\cdot \mathbf{1}(E_{m,t}),\\
&\tilde{\sigma}_{V^{\text{in}}_{t+1}}(s_t,a_t)={\sigma}_{V^{\text{in}}_{t+1}}(s_t,a_t)\mathbf{1}(E^c_{m,t})
+\Big[\frac{1}{n_{s_t,a_t}}\sum_{i=1}^m[V^{\text{in}}_{t+1}(s^{(i)}_{t+1})]^2\cdot\mathbf{1}_{[s^{(i)}_t,a^{(i)}_t=s_t,a_t]}-\tilde{z}_t^2(s_t,a_t)\Big]\mathbf{1}(E_{m,t}),\\
&e(s_t,a_t) = \sqrt{\frac{4\tilde{\sigma}_{V^{\tin}_{t+1}} \iota}{m d^\mu_t(s_t, a_t)}}  + 2\sqrt{6}V_{\max}\big(\frac{\iota}{m {d^\mu_t(s_t, a_t)}}\big)^{3/4} + 16 \frac{V_{\max} \iota}{m {d^\mu_t(s_t, a_t)}}\\
\end{aligned}
\end{equation}
where $E_{m,t}=\{n_{s_t,a_t}> \frac{1}{2} m\cdot d^\mu_t(s_t,a_t)\}$ and $n_{s_t,a_t}$ is the number of episodes visited $(s_t,a_t)$ at time $t$.
We also note that we only aggregate the data at the same time step $t$, so that the observations are from different episodes and thus independent\footnote{This is natural for the non-stationary transition setting; for the stationary transition setting we have an improved way for defining this estimators. See Section~\ref{sec:finite_horizon2}.}.

Similarly,  our estimator $g_t(s_t, a_t) = \tilde{g}_t(s_t,a_t) - f(s_t,a_t)$ where  $\tilde{g}_t(s_t,a_t) $ estimates $P_t^\top(\cdot|s_{t},a_{t})[V_{t+1}-V^{\text{in}}_{t+1}]$ using $l$ independent episodes (let $n'_{s_t,a_t}$ denote the visitation count from these $l$ episodes) and $f(s_t,a_t)$ is an error bar:
{
  \begin{equation}\label{eqn:off_g}
  \begin{aligned}
    \tilde{g}_{t}(s_{t},a_{t})=&P^\top_t(\cdot|s_{t},a_{t})[V_{t+1}-V^{\text{in}}_{t+1}]\cdot\mathbf{1}(E^c_{l,t})\\
      &+\frac{1}{n'_{s_{t},a_{t}}}\sum_{j=1}^l[V_{t+1}(s'^{(j)}_{t+1})-V_{t+1}^{\text{in}}(s'^{(j)}_{t+1})]\cdot\mathbf{1}_{[s'^{(j)}_{t},a'^{(j)}_{t}=s_{t},a_{t}]}\mathbf{1}(E_{l,t})
      \end{aligned}
  \end{equation}
}
Here, $E_{l,t} =\{n'_{s_{t},a_{t}}> \frac{1}{2} l\cdot d^\mu_{t}(s_{t},a_{t})\}$ and $f(s_t, a_t,u) \defeq 4u \sqrt{{\iota}/{l  d^\mu_t(s_t, a_t)}}$. Notice that  $f(s_t,a_t,u)$ depends on the additional input $u^{\text{in}}$ which measures the certified suboptimality of the input.  

\paragraph{Fictitious vs. actual estimators.} Careful readers must have noticed that that the above estimators $\mathbf{z}_t$ and $\mathbf{g}_t$ are \emph{infeasible} to implement as they require the unobserved (population level-quantities) in some cases. We call them \emph{fictitious} estimators as a result. Readers should rest assured since by the following proposition we can show their practical implementations (summarized in Figure~\ref{fig:practical_estimators}) are \emph{identical} to these \emph{fictitious} estimators with high probability:
\begin{proposition}[Summary of Section~\ref{sec:app_practical}]\label{prop:practical}
	Under the condition of Theorem~\ref{thm:main_nonstationary}, we have 
	\[
	\P\left[\bigcup_{i\in[K_1],t\in[H]}\left(E^{(i)c}_{l,t}\cup E^{(i)c}_{m,t}\right)\bigcup_{j\in[K_2],t\in[H]}\left(E^{(j)c}_{l,t}\cup E^{(j)c}_{m,t}\right)\right]\leq \delta/2,
	\]
	this means with high probability $1-\delta/2$, fictitious estimators $\tilde{z}_t,\tilde{g}_t,\tilde{\sigma}$ are all identical to their practical versions (summarized in Figure~\ref{fig:practical_estimators}).  Moreover, under the same high probability events,  the empirical version of the ``error bars'' $e_t(s_t,a_t)$ and $f(s_t,a_t,u)$ are at most twice as large than their fictitious versions that depends on the unknown $d^\mu_t(s_t, a_t)$.
\end{proposition}

These  \emph{fictitious} estimators, however, are easier to analyze and they are central to our extension of the Variance Reduction framework  previously used in the generative model setting \citep{sidford2018near} to the offline setting.  The idea is that it \emph{replaces} the low-probability, but pathological cases due to random $n_{s_t,a_t}$ with ground truths. Another challenge of the offline setting is due to the dependence of data points within a single episode. Note that the estimators are only aggregating the data at the same time steps. Since the data pair at the same time must come from different episodes, then conditional independence (given data up to current time steps, the states transition to the next step are independent of each other) can be recovered by this design \eqref{eqn:off_z}, \eqref{eqn:off_g}.

\paragraph{The doubling procedure}
It turns out that Algorithm~\ref{alg:OPDVRT} alone does not yield a tight sample complexity guarantee, due to its \emph{suboptimal dependence on the initial optimality gap} $u^{(0)}\ge \sup_t \|V^\star_t - V^{(0)}_t\|_\infty$ (recall $u^{(0)}$ is the initial parameter in the outer loop of Algorithm~\ref{alg:OPDVRT}). This is captured in the following:
\begin{proposition}[Informal version of Lemma~\ref{lem:complexity_OPVRT}]
	\label{prop:suboptimal-h}
	Suppose $\epsilon\in(0,1]$ is the final target accuracy. Algorithm~\ref{alg:OPDVRT} outputs the $\epsilon$-optimal policy with episode complexity:
	\begin{align*}
	&\bullet\;\tilde{O}(H^4/d_m\epsilon^2),\quad  \text{If}\; u^{(0)}> \sqrt{H};  \qquad\qquad \\
	&\bullet \;\tilde{O}(H^3/d_m\epsilon^2)  ,\quad \text{If}\; u^{(0)}\leq \sqrt{H}.
	\end{align*}
\end{proposition}
Proposition~\ref{prop:suboptimal-h} suggests that Algorithm~\ref{alg:OPDVRT} may have a suboptimal sample complexity when the initial optimality gap $u^{(0)}>\sqrt{H}$. Unfortunately, this is precisely the case for standard initializations such as $V_t^{(0)}\defeq \mathbf{0}$, for which we must take $u^{(0)}=H$.
We overcome this issue by designing a two-stage \emph{doubling} procedure: At stage $1$, we use Algorithm~\ref{alg:OPDVRT} to obtain 
$V_t^{\text{intermediate}}$, $\pi^{\text{intermediate}}$ that are $\epsilon'=\sqrt{H}\epsilon$ accurate; At stage $2$, we then use Algorithm~\ref{alg:OPDVRT} again with $V_t^{\text{intermediate}}$, $\pi^{\text{intermediate}}$ as the input and further reduce the error from $\epsilon'$ to $\epsilon$. The main take-away of this doubling procedure is that the episode complexity of both stage is only $\tilde{O}(H^3/d_m\epsilon^2)$, therefore the total sample complexity optimality is preserved. 

\paragraph{Full algorithm description}
We describe our full algorithm \OPDVR{} in Algorithm~\ref{alg:OPDVR}.
\begin{algorithm}[H]
	\caption{(OPDVR) Off-Policy Doubled Variance Reduction}
	\label{alg:OPDVR}
	\small{
		%
		\begin{algorithmic}[1]
			\INPUT Offline Dataset $\cD$ of size $n$ as a stream.  Target accuracy $\epsilon,\delta$ such that the algorithm does not use up $\cD$.
			\INPUT Estimators $\mathbf{z}_t, \mathbf{g}_t $ in function forms, $m'_1, m'_2, K_1, K_2$. 
			\STATE  \mypink{$\diamond$ Stage $1$. coarse learning: a ``warm-up'' procedure } 
			\STATE Set initial values $V_t^{(0)}:=\mathbf{0}$ and any policy $\pi^{(0)}$. 
			\STATE Set initial $u^{(0)}:=H$.
			\STATE Set $\mathbf{m}(u)=m_1'\log(16HSAK_1)/u^2$. 
			\STATE Run  Algorithm~\ref{alg:OPDVRT} with   $\mathbf{m}, \mathbf{z}_t, \mathbf{g}_t,  V_t^{(0)}, \pi^{(0)},u^{(0)}, K_1,\cD$ and return
			$V_t^{\text{intermediate}},\pi^{\text{intermediate}}$. 
			\STATE \mypink{$\diamond$ Stage $2$. fine learning: reduce error to given accuracy }			
			\STATE Reset initial values $V_t^{(0)}:=V_t^{\text{intermediate}}$ and policy $\pi^{(0)}:=\pi^{\text{intermediate}}$. Set $u^{(0)}:=\sqrt{H}$.
			\STATE Reset $\mathbf{m}(u)$ by replacing $m'_1$ with $m'_2$, $K_1$ with $K_2$. 
			\STATE Run  Algorithm~\ref{alg:OPDVRT} with   $\mathbf{m}, \mathbf{z}_t, \mathbf{g}_t,  V_t^{(0)}, \pi^{(0)},u^{(0)}, K_2,\cD$ and return $V_t^{\text{final}},\pi^{\text{final}}$.
			\OUTPUT  $V_t^{\text{final}},\pi^{\text{final}}$
		\end{algorithmic}
	}
%
%
\end{algorithm}
\normalsize{}

\begin{figure*}[t]\label{fig:practical_estimators}
	\caption{The \emph{implementable} ``plug-in'' lower confidence bound estimators $\mathbf{z}_t$ and $\mathbf{g}_t$.} 
	\resizebox{\linewidth}{!}{%
		\begin{tabular}{c|c|c}
			\hline
			Setting &  $\mathbf{z}_t(\cD_1, V_{t+1}^{\text{in}}, u)$
			& $\mathbf{g}_t(\cD_2, V_{t+1}, V_{t+1}^{\text{in}}, u)$\\
			\hline
			Non-stationary & $\frac{1}{n_{s_t,a_t}}\sum_{i=1}^m V^{\text{in}}_{t+1}(s^{(i)}_{t+1})\cdot\mathbf{1}_{[s^{(i)}_t,a^{(i)}_t=s_t,a_t]} - e_t(s_t,a_t)$ & $\frac{1}{n'_{s_{t},a_{t}}}\sum_{j=1}^l[V_{t+1}(s'^{(j)}_{t+1})-V_{t+1}^{\text{in}}(s'^{(j)}_{t+1})]\cdot\mathbf{1}_{[s'^{(j)}_{t},a'^{(j)}_{t}=s_{t},a_{t}]}- f_t(s_t,a_t,u)$
			\\
			Stationary &  $\frac{1}{n_{s,a}}\sum_{i=1}^m \sum_{u=1}^H V^{\text{in}}_{t+1}(s^{(i)}_{u+1})\cdot\mathbf{1}_{[s^{(i)}_u=s,a^{(i)}_u=a]} - e_t(s,a)$ &$\frac{1}{n'_{s,a}}\sum_{j=1}^l\sum_{u=1}^H[V_{t+1}(s'^{(j)}_{u+1})-V_{t+1}^{\text{in}}(s'^{(j)}_{u+1})]\cdot\mathbf{1}_{[s'^{(j)}_{u},a'^{(j)}_{u}=s,a]}- f_t(s,a,u)$\\
			$\infty$-Horizon &$\frac{1}{n_{s,a}}\sum_{i=1}^m V^{\text{in}}(s^{\prime(i)})\cdot\mathbf{1}_{[s^{(i)}=s,a^{(i)}=a]} -e(s,a)$& $\frac{1}{n'_{s_{t},a_{t}}}\sum_{j=1}^l[V^{(i)}(s'^{(j)}_{t+1})-V^{\text{in}}(s'^{(j)})]\cdot\mathbf{1}_{[s'^{(j)},a'^{(j)}=s,a]}- f(s,a,u)$\\
			\hline
		\end{tabular}
	}
	
	\resizebox{\linewidth}{!}{
		\begin{tabular}{c|c|c|c}
			\hline
			Setting & $\tilde{\sigma}(s_t,a_t)$ &  $e_t(s_t, a_t)$   & $f_t(s_t,a_t,u)$\\
			\hline
			Non-stationary &  $\frac{1}{n_{s_t,a_t}}\sum_{i=1}^m[V^{\text{in}}_{t+1}(s^{(i)}_{t+1})]^2\cdot\mathbf{1}_{[s^{(i)}_t,a^{(i)}_t=s_t,a_t]}-\tilde{z}_t^2(s_t,a_t)$  &
			$\sqrt{\frac{4\tilde{\sigma}_{V^{\tin}_{t+1}} \iota}{n_{s_t, a_t}}}  + 2\sqrt{6}V_{\max}\big(\frac{\iota}{n_{s_t, a_t}}\big)^{3/4} + 16 V_{\max}\frac{\iota}{n_{s_t, a_t}}$
			&  $4u \sqrt{\frac{\iota}{n'_{s_t, a_t}}}$
			\\
			Stationary & $\frac{1}{n_{s,a}}\sum_{i=1}^m\sum_{u=1}^H[V^{\text{in}}_{t+1}(s^{(i)}_{u+1})]^2\cdot\mathbf{1}_{[s^{(i)}_u=s,a^{(i)}_u=a]}-\tilde{z}_t^2(s,a)$& $\sqrt{\frac{4\tilde{\sigma}_{V^{\tin}_{t+1}} \iota}{n_{s, a}}}  + 2\sqrt{6}V_{\max}\big(\frac{\iota}{n_{s,a}}\big)^{3/4} + 16 V_{\max}\frac{\iota}{n_{s, a}}$
			&  $4u \sqrt{\frac{\iota}{n'_{s, a}}}$
			\\
			$\infty$-Horizon &$\frac{1}{n_{s,a}}\sum_{i=1}^m[V^{\text{in}}(s'^{(i)})]^2\cdot\mathbf{1}_{[s^{(i)}=s,a^{(i)}=a]}-\tilde{z}^2(s,a)$&$\sqrt{\frac{4\cdot\tilde{\sigma}_{V^{\tin}}\cdot\iota}{n_{s,a}}}+2\sqrt{6}\cdot V_{\max}\cdot\left(\frac{\iota}{n_{s,a}}\right)^{3/4}+\frac{16V_{\max}\iota}{3n_{s,a}}$& $4u\sqrt{\frac{\log(2RSA/\delta)}{n'_{s,a}}}$\\
			\hline
		\end{tabular}
	}
	\footnotesize{$^*$  $m,l$ are the number of episodes in $\cD_1,\cD_2$ respectively.  $\iota$ is a logarithmic factor in $HSA/\delta$ in the finite horizon case and  $SA/\delta$ in the infinite horizon cases.  $n_{s_t,a_t}$ is the number of times $s_t,a_t$ appears at time $t$ in $\cD_1$; and $n'(s_t,a_t)$ is the that for $\cD_2$.  In the case when $n_{s_t,a_t} = 0$, we simply output $0$ for all quantities above.}
\end{figure*}

\subsection{\emph{OPDVR} for non-stationary transition settings}
We now state our main theoretical guarantee for the \OPDVR{} algorithm in the finite-horizon non-stationary transition setting.

\begin{theorem}[Sample complexity of \OPDVR{} in finite-horizon non-stataionary setting]
	\label{thm:main_nonstationary}
	For the $H$-horizon non-stationary setting, there exist universal constants $c_1, c_2, c_3>0$ such that if we set {\small$m'_1=c_1H^4/d_m$} for Stage $1$,  {\small$m'_2=c_2H^3/d_m$} for Stage $2$, set $K_1=K_2=\log_2(\sqrt{H}/\epsilon)$, take $\mathbf{g}_t$ and $\mathbf{z}_t$ according to Figure~\ref{fig:practical_estimators}, then \OPDVR{} (Algorithm~\ref{alg:OPDVR}) with probability $1-\delta$ outputs an $\epsilon$-optimal policy $\hat{\pi}$ provided that the number of episodes in the offline data $\cD$ exceeds: 
		\[
		\frac{c_3 \max[\frac{m'_1}{H},m'_2]}{\epsilon^2}\log(32HSA\log_2(\sqrt{H}/\epsilon)/\delta)\log_2(\sqrt{H}/\epsilon) = \widetilde{O}\left(\frac{H^3}{d_m\epsilon^2}\right).
		\]
%
\end{theorem}

\paragraph{Optimality of sample complexity}
Theorem~\ref{thm:main_nonstationary} shows that our \OPDVR{} algorithm can find an $\epsilon$-optimal policy with $\widetilde{O}(H^3/d_m\epsilon^2)$ episodes of offline data. Compared with the sample complexity lower bound $\Omega(H^3/d_m\epsilon^2)$ for offline learning~(Theorem~G.2. in \citet{yin2020near}), we see that our \OPDVR{} algorithm matches the lower bound up to logarithmic factors. The same rate was achieved previously by the local uniform convergence argument of~\citet{yin2020near} under a stronger assumption of full data coverage.
\begin{proof}[Proof sketch of Theorem~\ref{thm:main_nonstationary}.]
By Proposition~\ref{prop:practical}, it suffices to analyze the performance of \OPDVR{} instantiated with fictitious estimators \eqref{eqn:off_g} and \eqref{eqn:off_z}.
Theorem~\ref{thm:main_nonstationary} relies on first analyzing the 
the prototypical \OPDVR{} (Algorithm~\ref{alg:OPDVRT}) and then connecting the result to the practical version using Multiplicative Chernoff bound (Section~\ref{sec:app_practical}). 
In particular, both off-policy estimators $z_t$ and $g_t$ use lower confidence update
to avoid over-optimism and the $\max$ operator in $V_t=\max(V_{Q_{t}},V^{\text{in}}_t)$ helps prevent pessimism. By doing so the update $V_t$ in Algorithm~\ref{alg:OPDVRT} always satisfies $0\leq V_t\leq V^\star_t$, which is always within valid range. The doubling procedure of Algorithm~\ref{alg:OPDVR} then first decreases the accuracy to a coarse level $\epsilon'=\sqrt{H}\epsilon$, and further lowers it to the given accuracy $\epsilon$. The key technical lemma for achieving optimal dependence in $H$ is Lemma~\ref{lem:H3}, which bounds the term {\small$\sum_{u=t}^{H}\E^{\pi^\star}_{s_u,a_u}\left[{\Var}[{V^{\star}_{u+1}}(s_{u+1})\middle|s_u,a_u]\right]$} by $O(H^2)$ instead of the naive $O(H^3)$. 
The full proof of Theorem~\ref{thm:main_nonstationary} can be found in Appendix~\ref{sec:proof_n}.
\end{proof}

\section{\emph{OPDVR} for stationary transition settings}
\label{sec:finite_horizon2}

In this section, we switch gears to the \emph{stationary} transition setting, in which the transition probabilities are identical at all time steps: $P_t(s'|s,a):\equiv P(s'|s,a)$. We will consider both the (a) finite-horizon case where each episode is consist of $H$ steps; and (b) the infinite-horizon case where the reward at the $t$-th step is discounted by $\gamma^t$, where $\gamma \in (0,1)$ is a discount factor. 
%

These settings encompass additional challenges compared with the non-stationary case, as in theory the transition probabilities can now be estimated more accurately due to the shared information across time steps, and we would like our sample complexity to reflect such an improvement.

\subsection{Finite-horizon stationary setting}
\label{sec:stationary}
We begin by considering the finite-horizon stationary setting. As this is a special case of the non-stationary setting, Theorem~\ref{thm:main_nonstationary} implies that \OPDVR{} achieves $\widetilde{O}(H^3/d_m\epsilon^2)$ sample complexity. However, similar as in online RL~\citep{azar2017minimax}, this result may be potentially loose by an $O(H)$ factor, as the algorithm does not take into account the stationarity of the transitions. This motivates us to design an algorithm that better leverages the stationarity by aggregating state-action pairs across different time steps. Indeed, we modify the fictitious estimators ~\eqref{eqn:off_z} and~\eqref{eqn:off_g} 
into the following:
  \begin{equation}\label{eqn:off_z_s}
    \begin{aligned}
      &\tilde{z}_t(s,a)=P^\top(\cdot|s,a) V^{\text{in}}_{t+1}\cdot\mathbf{1}(E_m^c) + \frac{1}{n_{s,a}}\sum_{i=1}^m \sum_{u=1}^H V^{\text{in}}_{t+1}(s^{(i)}_{u+1})\cdot\mathbf{1}_{[s^{(i)}_u=s,a^{(i)}_u=a]}\mathbf{1}(E_m),\\
      &\tilde{\sigma}_{V^{\text{in}}_{t+1}}(s,a)={\sigma}_{V^{\text{in}}_{t+1}}(s,a)\mathbf{1}(E_m^c)+ [\frac{1}{n_{s,a}}\sum_{i=1}^m\sum_{u=1}^H[V^{\text{in}}_{t+1}(s^{(i)}_{u+1})]^2\cdot\mathbf{1}_{[s^{(i)}_u=s,a^{(i)}_u=a]}-\tilde{z}_t^2(s,a)]\mathbf{1}(E_m),\\
      &e_t(s,a) = \sqrt{\frac{4\tilde{\sigma}_{V^{\tin}_{t+1}} \iota}{m \sum_{t=1}^H d^\mu_t(s, a)}}  + 2\sqrt{6}V_{\max}\big(\frac{\iota}{m {\sum_{t=1}^H d^\mu_t(s, a)}}\big)^{3/4} + 16 V_{\max}\frac{\iota}{m \sum_{t=1}^H{d^\mu_t(s, a)}},
    \end{aligned}
  \end{equation}
where $E_m = \{n_{s,a}> \frac{1}{2} m\cdot \sum_{t=1}^H d^\mu_t(s,a)\}$
and $n_{s,a}=\sum_{i=1}^m\sum_{t=1}^H\mathbf{1}{[s^{(i)}_t=s,a^{(i)}_t=a]}$ 
is the number of data pieces visited $(s,a)$ over all $m$ episodes. Moreover, $f_t(s,a,u) = 4u \sqrt{\frac{\iota}{l} \sum_{t=1}^H{d^\mu_t(s, a)}}$ and
	\begin{equation}\label{eqn:off_g_s}
	\begin{aligned}
	&\tilde{g}_{t}(s,a)=P^\top(\cdot|s,a)[V_{t+1}-V^{\text{in}}_{t+1}]\mathbf{1}(E^c_l)
	+
	\frac{1}{n'_{s,a}}\sum_{j=1}^l\sum_{u=1}^H[V_{t+1}(s'^{(j)}_{u+1})-V_{t+1}^{\text{in}}(s'^{(j)}_{u+1})]\cdot\mathbf{1}_{[s'^{(j)}_{u},a'^{(j)}_{u}=s,a]}\mathbf{1}(E^c_l).
	\end{aligned}
    \end{equation}

\begin{theorem}[Sample complexity of~\OPDVR{} in finite-horizon stationary setting]
  \label{thm:main_stationary}
  In the $H$-horizon stationary transition setting, 
  there exists universal constants $c'_1,c'_2, c'_3$ such that if we set $m'_1=c'_1H^3/d_m$, $m'_2=c'_2 H^2/d_m$ for Stage $1$ and $2$, set $K_1=K_2=\log_2(\sqrt{H}/\epsilon)$, and take $\mathbf{z}_t$ and $\mathbf{g}_t$ according to Figure~\ref{fig:practical_estimators}, 
  then with probability $1-\delta$, Practical \OPDVR{} finds an $\epsilon$-optimal policy provided that the number of episodes in the offline data $\cD$ exceeds:
  	\[	\frac{c'_3 \max[\frac{m'_1}{H},m'_2]}{\epsilon^2}\log(32HSA\log_2(\sqrt{H}/\epsilon)/\delta)\log_2(\sqrt{H}/\epsilon)  = \widetilde{O}\left(\frac{H^2}{d_m\epsilon^2}\right).
    \]
\end{theorem}
Theorem~\ref{thm:main_stationary} encompasses our main technical contribution, as
the compact data aggregation among different time steps make analyzing the estimators~\eqref{eqn:off_z_s} and~\eqref{eqn:off_g_s} knotty due to data-dependence (unlike the non-stationary transition setting where estimators are designed using data at specific time so the conditional independence remains). In particular, we need to fully exploit the property that transition $P$ is identical across different times in a pinpoint way to obtain the $H^2$ dependence in the sample complexity bound. 

\begin{proof}[Proof sketch]We design the martingale {\small$
X_k=\sum_{i=1}^m\sum_{u=1}^{k-1} \left(V^{\text{in}}_{t+1}(s^{(i)}_{u+1})-P^\top(\cdot|s,a)V^{\text{in}}_{t+1}\right)\cdot\mathbf{1}[s^{(i)}_u=s,a^{(i)}_u=a].
$} under the filtration $\mathcal{F}_k:=\{s_u^{(i)},a_u^{(i)}\}_{i\in[m]}^{u\in[k]}$ for bounding $ z_t(s_t,a_t) \leq P^\top(\cdot|s_t,a_t)V^{\text{in}}_{t+1}$. The conditional variance sum {\scriptsize$\sum_{k=1}^H\operatorname{Var}\left[X_{k+1} \mid \mathcal{F}_{k}\right]=\sum_{k=1}^H\sum_{i=1}^{m} \mathbf{1}\left[s_{k}^{(i)}, a_{k}^{(i)}=s,a\right] \operatorname{Var}\left[V_{t+1}^{\mathrm{in}}\left(s_{k+1}^{(i)}\right) \mid s_{k}^{(i)},a_{k}^{(i)}=s,a\right]$}. For stationary case, {\scriptsize$s_{k+1}^{(i)}\sim P(\cdot|s_{k}^{(i)},a_{k}^{(i)}=s,a)$} is irrelevant to time $k$ so above equals {\scriptsize $\sum_{k=1}^H\sum_{i=1}^{m} \mathbf{1}\left[s_{k}^{(i)}, a_{k}^{(i)}=s,a\right] \newline\sigma_{V_{t+1}^{\mathrm{in}}}(s, a)=n_{s,a}\cdot\sigma_{V_{t+1}^{\mathrm{in}}}(s, a)$}, where $V_{t+1}^{\mathrm{in}}$ is later replaced by $V_{t+1}^{\star}$ and $\sum_{t=1}^H\E^{\pi^\star}_{s,a}[\sigma_{V_{t}^{\star}}(s,a)]$ can be bounded by $H^2$ which is tight. In contrast, in non-stationary regime $P_t$ is varying across time so we can only obtain {\scriptsize$\sum_{k=1}^H\operatorname{Var}\left[X_{k+1} \mid \mathcal{F}_{k}\right]\leq n_{s,a} \max_t\sigma_{V_{t}^{\mathrm{in}}}(s, a)$}, which is later translated into $\sum_{t=1}^H\E^{\pi^\star}_{s,a}[\max_t\sigma_{V_{t}^{\star}}(s,a)]$, which in general has order $H^3$. To sum, the fact that $P$ is identical is carefully leveraged multiple times for obtaining $\widetilde{O}(H^2/d_m\epsilon^2)$ rate. The detailed proof of Theorem~\ref{thm:main_stationary} can be found in Appendix~\ref{sec:proof_stationary}.
\end{proof}




\paragraph{Improved dependence on $H$}
Theorem~\ref{thm:main_stationary} shows that \OPDVR{} achieves a sample complexity upper bound $\widetilde{O}(H^2/d_m\epsilon^2)$ in the stationary setting. To the best of our knowledge, this is the first result that achieves an $H^2$ dependence for offline RL with stationary transitions, and improves over the vanilla $H^3$ dependence in either the vanilla (non-stationary) \OPDVR{} (Theorem~\ref{thm:main_nonstationary}) or the ``off-policy evaluation + uniform convergence'' algorithm of~\citet{yin2020near}. 
We emphasize that we exploit specific properties of \OPDVR{} in our techniques for knocking off a factor of $H$ and 
there seems to be no direct ways in applying the same techniques in improving the uniform convergence-style results for the stationary-transition setting.
 



\paragraph{Optimality of $\widetilde{O}(H^2/d_m\epsilon^2)$}
We accompany Theorem~\ref{thm:main_stationary} by a establishing a sample complexity lower bound for this setting, showing that our algorithm achieves the optimal dependence of all parameters up to logarithmic factors.
\begin{theorem}
  \label{thm:learn_low_bound_s}
  For all $0<d_m\leq\frac{1}{SA}$, let the family of problem be $\mathcal{M}_{d_m} :=\big\{(\mu,M) \; \big| \;\min_{t,s_t,a_t} d_t^\mu(s_t,a_t) \newline \geq d_m\big\}$. There exists universal constants $c_1, c_2, c, p$ (with $H, S, A \geq c_1$ and $0 <\epsilon <c_2 $) such that when $n\leq cH^2/d_m\epsilon^2$, we always have 
  \[
    \inf_{{v}^{\pi_{alg}}}\sup_{(\mu,M)\in\mathcal{M}_{d_m}}\P_{\mu,M}\left(v^*-v^{\pi_{alg}}\geq \epsilon\right)\geq p.
  \]
\end{theorem}
The proof of Theorem~\ref{thm:learn_low_bound_s} builds on modifying the $\widetilde{O}(H^3/d_m\epsilon^2)$ sample complexity lower bound in the non-stationary case~\citep{yin2020near}, and can be found in Appendix~\ref{sec:lower}.

\subsection{Infinite-horizon discounted setting}
\label{sec:infinite}

We now consider the infinite-horizon discounted setting. 

\paragraph{Setup} An infinite-horizon discounted MDP is denoted by $(\mathcal{S},\mathcal{A},P,r,\gamma,d_0)$, where $\gamma$ is discount factor  and $d_0$ is initial state distribution. Given a policy $\pi$, the induced trajectory $s_0,a_0,r_0$, $s_1,a_1,r_1,...$ follows: $s_0\sim d_0$, $a_t\sim\pi(\cdot|s_t)$, $r_t=r(s_t,a_t)$, $s_{t+1}\sim P(\cdot|s_t,a_t)$. The corresponding value function (or state-action value) function is defined as: $V^\pi(s)=\E_\pi[\sum_{t=0}^\infty \gamma^t r_t|s_0=s]$, $Q^\pi(s)=\E_\pi[\sum_{t=0}^\infty \gamma^t r_t|s_0=s,a_0=a]$. Moreover, define the normalized marginal state distribution as $d^\pi(s):=(1-\gamma)\sum_{t=0}^\infty \gamma^t \P[s_t=s|s_0\sim d_0,\pi]$ and the state-action counterpart follows $d^\pi(s,a):=d^\pi(s)\pi(s|a)$. For the offline/batch learning problem, we adopt the same protocol of \cite{chen2019information,xie2020q} that data $\mathcal{D}=\{s^{(i)},a^{(i)},r^{(i)},s^{\prime(i)}\}_{i\in[n]}$ are i.i.d off-policy pieces with $(s,a)\sim d^\mu$ and $s'\sim P(\cdot|s,a)$.

Algorithm~\ref{alg:OPDVRT} and \ref{alg:OPDVR} are slighted modified slightly modified to cater to the infinite horizon setting (detailed pseudo-code in Algorithm~\ref{alg:OPDVRT_in} and \ref{alg:OPDVR_in} in the appendix).
Our result is stated as follows. The proof can be found in Appendix~\ref{sec:proof_i}.

\begin{theorem}[Sampe complexity of~\OPDVR{} in infinite-horizon discounted setting]
  \label{thm:main_infinite}
  Consider Algorithm~\ref{alg:OPDVR_in}. There are constants $c'_1, c'_2,c'_3$, such that if we set $m'_1=O((1-\gamma)^{-4}/d_m),m'_2 = O((1-\gamma)^{-3}/d_m)$ (see more precise expressions in Lemma~\ref{lem:complexity_OPVRT_in}), $K_1=\log_2({(1-\gamma)^{-1}}/\epsilon), K_2=\log_2(\sqrt{(1-\gamma)^{-1}}/\epsilon)$ , $R=\log (4/\epsilon (1-\gamma))$, and choose LCB estimators $\mathbf{z}$ and $\mathbf{g}$ as in Figure~\ref{fig:practical_estimators}, then 
  with probability $1-\delta$,  the infinite horizon version of \textbf{OPDVR} (Algorithm~\ref{alg:OPDVR_in}) outputs an $\epsilon$-optimal policy provided that in offline data $\cD$ has number of samples exceeding
      \begin{align*}
 \frac{  c'_3\max[\frac{m'_1}{(1-\gamma)^{-1}},m'_2]}{\epsilon^2} \cdot \iota'= \widetilde{O}\left[(1-\gamma)^{-3}/d_m\epsilon^2\right].
  \end{align*}
  where $\iota' := R\cdot(\log(32(1-\gamma)^{-1}RSA/\delta) + \log\log_2(\sqrt{(1-\gamma)^{-1}}/\epsilon))\cdot\log_2(\sqrt{(1-\gamma)^{-1}}/\epsilon)$.
\end{theorem}
The proof, deferred to Appendix~\ref{sec:proof_i},  again rely on the analyzing fictitious versions of Algorithm~\ref{alg:OPDVRT_in} and Algorithm~\ref{alg:OPDVR_in} with similar techniques described in our proof sketch of Theorem~\ref{thm:main_stationary}.

We note again that for the infinite horizon case, the sample-complexity measures the number of steps, while in the finite horizon case our sample complexity measures the number of episodes (each episode is $H$ steps) thus $(1-\gamma)^{-3}$ is comparable to the $H^2$ dependence.
To the best of our knowledge, Theorem~\ref{thm:main_stationary} and Theorem~\ref{thm:main_infinite} are the first results that achieve $H^2$, $(1-\gamma)^{-3}$ dependence in the offline regime respectively for stationary transition and infinite horizon setting, see Table~\ref{table}. Although we note that our result relies on the tabular structure whereas these prior algorithms work for general function classes, their bounds do not improve when reduced to tabular case. In particular, \citet{chen2019information,xie2020q,xie2020batch} consider using function approximation for exactly tabular problem. Lastly, in the tabular regime, our Assumption~\ref{assu2} is much weaker than the $\beta_\mu,C$ considered in these prior work; see Appendix~\ref{sec:comparison_dm} for a discussion.

		



\section{Discussions}
\label{sec:discussion}

\paragraph{Estimating $d_m$.}

It is worth mentioning that the input of \OPDVR{} depends on unknown system quantity $d_m$. Nevertheless, $d_m$ is only one-dimensional scalar and thus it is plausible (from a statistical perspective) to leverage standard parameter-tuning tools (\emph{e.g.} cross validation \citep{varma2006bias}) for obtaining a reliable estimate in practice. On the theoretical side, we provide the following result to show plug-in on-policy estimator $\widehat{d}^\mu_t(s_t,a_t)=n_{s_t,a_t}/n$ and $\widehat{d}_m:=\min_{t,s_t,a_t}\{n_{s_t,a_t}/n:n_{s_t,a_t}>0\},$ is sufficient for accurately estimating $d^\mu_t, d_m$ simultaneously.
\begin{lemma}\label{thm:est_d_m}
	For the finite-horizon setting (either stationary or non-stationary), there exists universal constant $c$,  s.t. when $n\geq c \cdot 1/d_m\cdot \log(HSA/\delta)$, then w.p. $1-\delta$, we have $\forall t, s_t,a_t$, 
	$
	\frac{1}{2}d^\mu_t(s_t,a_t)\leq \widehat{d}^\mu_t(s_t,a_t)\leq \frac{3}{2}d^\mu_t(s_t,a_t)
	$
	and, in particular, $
	\frac{1}{2}d_m\leq \widehat{d}_m\leq \frac{3}{2}d_m.
	$
\end{lemma}
Lemma~\ref{thm:est_d_m} ensures one can replace $d^\mu_t$ by $\widehat{d}^\mu_t$ ($d_m$ by $\widehat{d}_m$) in \OPDVR{} and we obtain a fully data-adaptive algorithm. Note that the requirement on $n$ does not affect our near-minimax complexity bound in either Theorem~\ref{thm:main_nonstationary} and~\ref{thm:main_stationary}---we only require $n\approx \tilde{\Theta}(1/d_m)$ additional episodes to estimate $d_m$ and it is of lower order compared to our upper bound $\widetilde{O}(H^3/d_m)$ or $\widetilde{O}(H^2/d_m)$). See Appendix~\ref{sec:data_adaptive} for the proof of Lemma~\ref{thm:est_d_m}.

%

\paragraph{Computational and memory cost.}
\OPDVR{} can be implemented as a streaming algorithm that uses only one pass of the dataset. Its computational cost is $\widetilde{O}(H^4/d_m\epsilon^2)$
— the same as its sample complexity in steps ($H$ steps is an episode), and the memory cost is $O(HSA)$ for the episodic
case and $O(SA)$ for the stationary or infinite horizon case. In particular, the double variance reduction technique does
not introduce additional overhead beyond constant factors.


\paragraph{Improvement over variance reduction under generative models.}
This work may be considered as an extension of the variance reduction framework for RL in the generative model setting (e.g.~\citep{sidford2018near,yang2019sample}), as some of proving techniques such as VR and monotone preserving $V_t\leq \mathcal{T}_{\pi_t}V_{t+1}$ are inherited from previous works. However, two improvements are made. First, the data pieces collected in offline case are highly dependent (in contrast for generative model setting each simulator call is independent) therefore how to disentangle the dependent structure and analyze tight results makes the offline setting inherently more challenging. Second, our doubling mechanism always guarantee the minimax rate with any initialization and the single VR procedure does not have this property (see Appendix~\ref{sec:sidford_proof} for a more detailed discussion). 
Lastly, on the other hand it is not very surprising technique (like VR) from generative model setting can be leveraged for offline RL. While generative model assumes access to the strong simulator $P(s'|s,a)$, for offline RL $\mu$ serves as a surrogate simulator where we can simulate episodes. If we can treat the distributional shift appropriately and decouple the data dependence in a pinpoint manner, it is hopeful that generic ideas still work.



\section{Conclusion}
\label{sec:conclusion}

This paper proposes OPDVR (off-policy double variance reduction), a new variance reduction algorithm for offline reinforcement learning.
 We show that OPDVR achieves tight sample complexity for offline RL in tabular MDPs; in particular, ODPVR is the first algorithm that acheives the optimal sample complexity for offline RL in the stationary transition setting. On the technical end, we present a sharp analysis under stationary transitions, and use the doubling technique to resolve the initialization dependence in variance reduction, both of which could be of broader interest. We believe this paper leads to some interesting next steps. For example, can our understandings about the variance reduction algorithm shed light on other commonly used algorithms (such as Q-Learning) for offline RL? How can we better deal with insufficient data coverage? We would like to leave these as future work.


\section*{Acknowledgment}

The authors would like to thank Lin F. Yang for the discussions about~\citep{sidford2018near}. 

\bibliographystyle{apa-good.bst}
\bibliography{sections/stat_rl}

\appendix

\clearpage
\begin{center}
	{\LARGE Appendix}
\end{center}

\section{More Discussion on Related Work}
\label{sec:related}

\paragraph{Offline reinforcement learning}
There is a growing body of work on offline RL recently~\citep{levine2020offline}  in both offline policy evaluation (OPE) and offline policy optimization. OPE (also known as Off-Policy Evaluation \citep{li2015toward}) requires estimating the value of a \emph{target policy} $\pi$ from an offline dataset that is often generated using another behavior policy $\mu$. A variety of algorithms and theoretical guarantees have been established for offline policy evaluation~\citep{li2015toward,jiang2016doubly,liu2018breaking,kallus2019double,kallus2019efficiently,uehara2019minimax,xie2019towards,yin2020asymptotically,duan2020minimax,liu2020understanding,liu2020provably,feng2020accountable}. The majority of these work uses (vanilla or more advanced versions of) importance sampling to correct for the distribution shift, or uses minimax formulation to approximate the task and solves the questions through convex/non-convex optimization.


Meanwhile, the offline policy optimization problem needs to find a near-optimal policy given the offline dataset. The study of offline policy optimization can be dated back to the classical Fitted Q-Iteration algorithm~\citep{antos2008fitted,antos2008learning}. The sample complexity for offline policy optimization is studied in a line of recent work~\citep{chen2019information,le2019batch,xie2020q,xie2020batch,liu2020provably}. The focus on these work is on the combination of offline RL and function approximation; when specialized to the tabular setting, the sample complexities have a rather suboptimal dependence on the horizon $H$ (or $(1-\gamma)^{-1}$ in the discounted setting). In particular, \citet{chen2019information,le2019batch} first established finite sample guarantees with complexity $\widetilde{O}((1-\gamma)^{-6}\beta_\mu/\epsilon^2)$ (where $\beta_\mu$ is the concentration coefficient) and it is later improved to $\widetilde{O}((1-\gamma)^{-4}\beta_\mu/\epsilon^2)$ by \citet{xie2020q} with a finer analysis. Later, \citet{xie2020batch} considers offline RL under weak \emph{realizability} assumption and \citet{liu2020provably} considers offline RL \emph{without} good exploration. Those are challenging offline settings but their dependence on horizon $(1-\gamma)^{-1}$ (or $H$) is very suboptimal. The recent work of~\citet{yin2020near} (OPE + uniform convergence) first achieves the sample complexity $\widetilde{O}(H^3/d_m\epsilon^2)$ in the finite-horizon non-stationary transition setting for tabular offline RL, and establishes a lower bound $\Omega(H^3/d_m\epsilon^2)$ (where $d_m$ is a constant related to the data coverage of the behavior policy in the given MDP that is similar to the concentration coefficient $\beta_\mu$) that matches the upper bound up to logarithmic factors. Compared with these work, we analyze a new variance reduction algorithm for offline policy learning, which is a more generic approach as it adapts to all typical settings (finite horizon stationary/non-stationary transition, infinite horizon setting) with optimal sample complexity while the technique in \citet{yin2020near} only works for non-stationary setting and cannot directly reduce to $\widetilde{O}(H^2/d_m\epsilon^2)$ when the transition becomes stationary. Concurrent to this work,~\citet{jin2020pessimism} study pessimism-based algorithms for offline policy optimization under insufficient coverage of the data and \cite{wang2020statistical,zanette2020exponential} provide some negative results (exponential lower bound) for offline RL with linear MDP structure.





\paragraph{Reinforcement learning in online settings}
In online RL (where one has interactive access to the environment), the model-based UCBVI algorithm achieves the minimax regret of $\tilde{O}(\sqrt{HSAT})$~\citep{azar2017minimax} and is later improved by~\citep{dann2018policy}. Later this minimax rate is also achieved by EULER with stronger problem-dependent expressions \citep{zanette2019tighter,simchowitz2019non}. 
Model-free algorithms such a \textit{Q-learning} is able to achieve a $\sqrt{H}$-suboptimal regret comparing to lower bound~\citep{jin2018q} and this gap is recently closed by an improved model-free algorithm in~\citep{zhang2020almost}.

In the generative model setting (where one has a simulator that samples $(r_t, s_{t+1})$ from any $(s_t, a_t)$,~\citep{azar2013minimax,wainwright2019variance} prove sample complexity $\widetilde{O}((1-\gamma)^{-3}SA/\epsilon^2)$ is sufficient for the output $Q$-function to be $\epsilon$-optimal, \emph{i.e.} $||Q^\star-Q^{\text{out}}||_\infty<\epsilon$, however this does not imply $\epsilon$-optimal policy with the same sample complexity. The most related to our work among this line is~\citep{sidford2018near}, which designs an \emph{variance reduction} algorithm that overcomes the above issue and obtains $\widetilde{O}((1-\gamma)^{-3}SA/\epsilon^2)$ sample complexity or finding the optimal policy as well.  Later~\citep{yang2019sample} again uses VR to obtain the sample optimality under the linear transition models. The design of our algorithm builds upon the variance reduction technique; our doubling technique and analysis in the offline setting can be seen as a generalization of~\citep{sidford2018near}; see Section~\ref{sec:discussion} and Appendix~\ref{sec:sidford_proof} for more detailed discussions.
\section{Proofs for finite-horizon non-stationary setting }\label{sec:proof_n}

The roadmap of our analysis in this section consists of first doing concentration analysis, then iteratively reasoning using induction, analyzing the doubling procedure and proving from prototypical version to the practical version. At a high level, we arrange the proving pipeline to be similar to that of \cite{sidford2018near} and let exquisite readers feel the difference between generative model setting and offline setting in tabular case (and why VR works under weak offline Assumption~\ref{assu2}). We also address one defect in \cite{sidford2018near} later (see Section~\ref{sec:sidford_proof}) to contrast that our doubling VR procedure is necessary. 

 Even before that, let us start with the simple \emph{monotone preservation} lemma.

\begin{lemma}\label{lem:mono}
	Suppose $V$ and $\pi$ is any value and policy satisfy $V_t\leq \mathcal{T}_{\pi_t}V_{t+1}$ for all $t\in[H]$. Then it holds $V_t\leq V^\pi_t\leq V^\star_t$, for all $t\in[H]$.
	
\end{lemma}
\begin{proof}
	We only need to show $V_t\leq V^\pi_t$. Since $V_t\leq \mathcal{T}_{\pi_t}V_{t+1}$, we can use it repeatedly to obtain

	\begin{equation}\label{eqn:iter_v}
	V_t\leq \mathcal{T}_{\pi_t} V_{t+1}\leq \mathcal{T}_{\pi_t}(\mathcal{T}_{\pi_{t+1}}V_{t+2})\leq...\leq \mathcal{T}_{\pi_t} \circ \mathcal{T}_{\pi_{t+1}}\circ...\circ\mathcal{T}_{\pi_H} V_{H+1}
	\end{equation}
	where ``$\circ$'' denotes operator composition. Note by default $V_{H+1}=V_{H+1}^\pi = V_{H+1}^\star=\mathbf{0}$, therefore 
	\begin{equation}\label{eqn:iter_v_2}
	\mathcal{T}_{\pi_{t+1}}\circ...\circ\mathcal{T}_{\pi_H} V_{H+1}=\mathcal{T}_{\pi_{t+1}}\circ...\circ\mathcal{T}_{\pi_H} V_{H+1}^{\pi}=\mathcal{T}_{\pi_{t+1}}\circ...\circ\mathcal{T}_{\pi_{H-1}} V_{H}^{\pi}=...=\mathcal{T}_{\pi_t} V_{t+1}^{\pi}=V_{t}^{\pi}
	\end{equation}
	where we use the definition of Bellman equation that $V_{t}^{\pi}=\mathcal{T}_{\pi_t} V_{t+1}^{\pi}$ for all $t$. Combining \eqref{eqn:iter_v} and \eqref{eqn:iter_v_2} gives the stated result. 
\end{proof}

\subsection{Concentration analysis for non-stationary transition setting}  
Recall $\tilde{z}_t$, $\tilde{\sigma}_{V^{\text{in}}_{t+1}}$ \eqref{eqn:off_z} and $g_t$ \eqref{eqn:off_g} are three quantities deployed in Algorithm~\ref{alg:OPDVRT} that use off-policy data $\mathcal{D}$. We restate their definition as follows:

{\small
	\begin{align*}
	\tilde{z}_t(s_t,a_t)&=\begin{cases}P^\top_t(\cdot|s_t,a_t) V^{\text{in}}_{t+1}, &{if} \;n_{s_t,a_t}< \frac{1}{2} m\cdot d^\mu_t(s_t,a_t),\\ 
	\frac{1}{n_{s_t,a_t}}\sum_{i=1}^m V^{\text{in}}_{t+1}(s^{(i)}_{t+1})\cdot\mathbf{1}_{[s^{(i)}_t=s_t,a^{(i)}_t=a_t]},&{if} \;n_{s_t,a_t}\geq \frac{1}{2} m\cdot d^\mu_t(s_t,a_t).\end{cases}\\
	\tilde{\sigma}_{V^{\text{in}}_{t+1}}(s_t,a_t)&=\begin{cases}{\sigma}_{V^{\text{in}}_{t+1}}(s_t,a_t),\qquad\qquad\qquad\qquad\qquad\,\qquad\quad {if} \;n_{s_t,a_t}< \frac{1}{2} m\cdot d^\mu_t(s_t,a_t),\\ \frac{1}{n_{s_t,a_t}}\sum_{i=1}^m[V^{\text{in}}_{t+1}(s^{(i)}_{t+1})]^2\cdot\mathbf{1}_{[s^{(i)}_t=s_t,a^{(i)}_t=a_t]}-\tilde{z}_t^2(s_t,a_t), \;\;{otherwise}.
	\end{cases}
	\end{align*}
}

{\small\[
	g_{t}(s_{t},a_{t})=\begin{cases}P^\top(\cdot|s_{t},a_{t})[V_{t+1}-V^{\text{in}}_{t+1}]-f(s_{t},a_{t}),\hfill {if}\quad n_{s_{t},a_{t}}< \frac{1}{2} l\cdot d^\mu_{t}(s_{t},a_{t}),
	\\\frac{1}{n'_{s_{t},a_{t}}}\sum_{j=1}^l[V_{t+1}(s'^{(j)}_{t+1})-V_{t+1}^{\text{in}}(s'^{(j)}_{t+1})]\cdot\mathbf{1}_{[s'^{(j)}_{t},a'^{(j)}_{t}=s_{t},a_{t}]}-f(s_{t},a_{t}),\hfill o.w.\end{cases}
	\]}and recall $f(s_t,a_t)=4u^{\text{in}}\sqrt{{\log(2HSA/\delta)}/{l d^\mu_t(s_t,a_t)}}$. Also, we use bold letters to represent matrices, \emph{e.g.} $\PP_t\in\R^{SA\times S}$ satisfies $\PP_t[(s_t,a_t),s_{t+1}]=P_t(s_{t+1}|s_t,a_t)$. The following Lemmas~\ref{lem:con_z},\ref{lem:con_sig},\ref{lem:con_g} provide their concentration properties.

\begin{lemma}\label{lem:con_z}
Let $\tilde{z_t}$ be defined as \eqref{eqn:off_z} in Algorithm~\ref{alg:OPDVRT}, where $\tilde{z_t}$ is the off-policy estimator of $P_t^\top(\cdot|s_t,a_t)V^{\text{in}}_{t+1}$ using $m$ episodic data. Then with probability $1-\delta$, we have 

\begin{equation}\label{eqn:tilde_z}
\left|\tilde{z_t}-\PP_t V^{\text{in}}_{t+1}\right|\leq \sqrt{\frac{4\cdot\sigma_{V^{\text{in}}_{t+1}}\cdot\log(HSA/\delta)}{m\cdot d_t^\mu}}+\frac{4V_{\max}}{3m\cdot d^\mu_t}\log(HSA/\delta), \quad \forall t\in[H]
\end{equation}
here $\tilde{z_t},\PP_t V^{\text{in}}_{t+1},\sigma_{V^{\text{in}}_{t+1}},d^\mu_t\in\mathbb{R}^{S\times A}$ are $S\times A$ column vectors and $\sqrt{\cdot}$ is elementwise operation.
\end{lemma}

\begin{proof}
	First fix $s_t,a_t$. Let $E_t :=\{n_{s_t,a_t}\geq \frac{1}{2}m\cdot d^\mu_t(s_t,a_t)\}$, then by definition, 
	\[
	\tilde{z_t}(s_t,a_t)-P^\top_t(\cdot|s_t,a_t)V^{\text{in}}_{t+1} = \left(\frac{1}{n_{s_t,a_t}}\sum_{i=1}^m V^{\text{in}}_{t+1}(s^{(i)}_{t+1})\cdot\mathbf{1}[s^{(i)}_t=s_t,a^{(i)}_t=a_t]-P^\top_t(\cdot|s_t,a_t)V^{\text{in}}_{t+1}\right)\cdot\mathbf{1}(E_t).
	\]
	Next we conditional on $n_{s_t,a_t}$. Then from above expression and Bernstein inequality \ref{lem:bernstein_ineq} we have with probability at least $1-\delta$
	\begin{align*}
	&\left|\tilde{z_t}(s_t,a_t)-P^\top_t(\cdot|s_t,a_t)V^{\text{in}}_{t+1}\right| \\
	= &\left|\frac{1}{n_{s_t,a_t}}\sum_{i=1}^{n_{s_t,a_t}}V^{\text{in}}_{t+1}(s^{(i)}_{t+1}|s_t,a_t)-P^\top_t(\cdot|s_t,a_t)V^{\text{in}}_{t+1}\right|\cdot\mathbf{1}(E_t)\\
	\leq & \left(\sqrt{\frac{2\cdot\sigma_{V^{\text{in}}_{t+1}}(s_t,a_t)\cdot\log(1/\delta)}{n_{s_t,a_t}}}+\frac{2V_{\max}}{3n_{s_t,a_t}}\log(1/\delta)\right)\cdot\mathbf{1}(E_t)\\
	\leq & \sqrt{\frac{4\cdot\sigma_{V^{\text{in}}_{t+1}}(s_t,a_t)\cdot\log(1/\delta)}{m\cdot d_t^\mu(s_t,a_t)}}+\frac{4V_{\max}}{3m\cdot d^\mu_t(s_t,a_t)}\log(1/\delta)
	\end{align*}
	where we use shorthand notation $V^{\text{in}}_{t+1}(s^{(i)}_{t+1}|s_t,a_t)$ to denote the value of $V^{\text{in}}_{t+1}(s^{(i)}_{t+1})$ given $s^{(i)}_t=s_t$ and $a^{(i)}=a_t$. The condition $V^\text{in}_t\leq V_{\max}$ is guaranteed by Lemma~\ref{lem:mono}.
	Now we get rid of the conditional on $n_{s_t,a_t}$. Denote $A=\{\tilde{z_t}(s_t,a_t)-P^\top(\cdot|s_t,a_t)V^{\text{in}}_{t+1}\leq \sqrt{4\cdot\sigma_{V^{\text{in}}_{t+1}}(s_t,a_t)\cdot\log(1/\delta)/m\cdot d_t^\mu(s_t,a_t)}+\frac{4V_{\max}}{3m\cdot d^\mu_t(s_t,a_t)}\log(1/\delta)\}$, then equivalently we can rewrite above result as $\P(A|n_{s_t,a_t})\geq 1-\delta$. Note this is the same as $\E[\mathbf{1}(A)|n_{s_t,a_t}]\geq 1-\delta$, therefore by law of total expectation we have
	\[
	\P(A)=\E[\mathbf{1}(A)]=\E[\E[\mathbf{1}(A)|n_{s_t,a_t}]]\geq \E[1-\delta]=1-\delta,
	\]
	\textit{i.e.} for fixed $(s_t,a_t)$ we have with probability at least $1-\delta$, 
	\[
	\left|\tilde{z_t}(s_t,a_t)-P^\top(\cdot|s_t,a_t)V^{\text{in}}_{t+1}\right|\leq \sqrt{\frac{4\cdot\sigma_{V^{\text{in}}_{t+1}}(s_t,a_t)\cdot\log(1/\delta)}{m\cdot d_t^\mu(s_t,a_t)}}+\frac{4V_{\max}}{3m\cdot d^\mu_t(s_t,a_t)}\log(1/\delta)
	\]
	 Apply the union bound over all $t,s_t,a_t$, we obtain
	 \[
	 \left|\tilde{z_t}-\PP_t V^{\text{in}}_{t+1}\right|\leq \sqrt{\frac{4\cdot\sigma_{V^{\text{in}}_{t+1}}\cdot\log(HSA/\delta)}{m\cdot d_t^\mu}}+\frac{4V_{\max}}{3m\cdot d^\mu_t}\log(HSA/\delta),
	 \]
	 where the inequality is  element-wise and this is \eqref{eqn:tilde_z}.
\end{proof}

\begin{remark}
	Exquisite reader might notice under the Assumption~\ref{assu2} it is likely for some $(s_t,a_t)$ the corresponding $d^\mu_t(s_t,a_t)=0$, then the result \eqref{eqn:tilde_z} may fail to be meaningful (since less than infinity is trivial). However, in fact for those entries it is legitimate to set the right hand side of \eqref{eqn:tilde_z} equal to $0$. The reason comes from our construction in \eqref{eqn:off_z} that when $d^\mu_t(s_t,a_t)=0$, it must holds $n_{s_t,a_t}=0$, so in this case $\tilde{z_t}(s_t,a_t)=P^\top(\cdot|s_t,a_t)V^{\text{in}}_{t+1}$. Therefore, we keep writing in this fashion only for the ease of illustration.
\end{remark}

\begin{lemma}\label{lem:con_sig}
	Let $\tilde{\sigma}_{V^{\text{in}}_{t+1}}$ be defined as \eqref{eqn:off_z} in Algorithm~\ref{alg:OPDVRT}, i.e. the off-policy estimator of $\sigma_{V^{\text{in}}_{t+1}}(s_t,a_t)$ using $m$ episodic data. Then with probability $1-\delta$, we have 
	
	\begin{equation}\label{eqn:tilde_sig}
	\left|\tilde{\sigma}_{V^{\text{in}}_{t+1}}-\sigma_{V_{t+1}^{\text{in}}} \right|\leq 6 V_{\max}^2\sqrt{\frac{\log(4HSA/\delta)}{m\cdot d^\mu_t}}+\frac{4V_{\max}^2\log(4HSA/\delta)}{m\cdot d^\mu_t},\quad \forall t=1,...,H.
	\end{equation}
	
\end{lemma}

\begin{proof}
	From the definition we have for fixed $(s_t,a_t)$
	\begin{align*}
	\tilde{\sigma}_{V^{\text{in}}_{t+1}}(s_t,a_t)-\sigma_{V_{t+1}^{\text{in}}}(s_t,a_t) =& \left(\frac{1}{n_{s_t,a_t}}\sum_{i=1}^{n_{s_t,a_t}}V^{\text{in}}_{t+1}(s^{(i)}_{t+1}|s_t,a_t)^2-P^\top(\cdot|s_t,a_t)(V^{\text{in}}_{t+1})^2\right)\mathbf{1}(E_t)\\
	+&\left(\left[\frac{1}{n_{s_t,a_t}}\sum_{i=1}^{n_{s_t,a_t}}V^{\text{in}}_{t+1}(s^{(i)}_{t+1}|s_t,a_t)\right]^2-\left[P^\top(\cdot|s_t,a_t)V^{\text{in}}_{t+1}\right]^2\right)\mathbf{1}(E_t)\\
	\end{align*}
	By using the same conditional on $n_{s_t, a_t}$ as in Lemma~\ref{lem:con_z}, applying Hoeffding's inequality and law of total expectation, we obtain with probability $1-\delta/2$, the first term in above is bounded by
	\begin{equation}\label{eqn:first_diff}
	\begin{aligned}
	&\left(\frac{1}{n_{s_t,a_t}}\sum_{i=1}^{n_{s_t,a_t}}V^{\text{in}}_{t+1}(s^{(i)}_{t+1}|s_t,a_t)^2-P^\top(\cdot|s_t,a_t)(V^{\text{in}}_{t+1})^2\right)\mathbf{1}(E_t)\\
	&\leq V_{\max}^2\sqrt{\frac{2\log(4/\delta)}{n_{s_t,a_t}}}\cdot\mathbf{1}(E_t)\leq 2V_{\max}^2\sqrt{\frac{\log(4/\delta)}{m\cdot d^\mu_t(s_t,a_t)}},
	\end{aligned}
	\end{equation}
	and similarly with probability $1-\delta/2$, 
	\begin{equation}\label{eqn:sig_diff}
	\left(\frac{1}{n_{s_t,a_t}}\sum_{i=1}^{n_{s_t,a_t}}V^{\text{in}}_{t+1}(s^{(i)}_{t+1}|s_t,a_t)-P^\top(\cdot|s_t,a_t)V^{\text{in}}_{t+1}\right)\mathbf{1}(E_t)\leq 2V_{\max}\sqrt{\frac{\log(4/\delta)}{m\cdot d^\mu_t(s_t,a_t)}}.
	\end{equation}
	Note for $a,b,c>0$, if $|a-b|\leq c$, then $|a^2-b^2|=|a-b|\cdot|a+b|\leq |a-b|\cdot(|a|+|b|)\leq|a-b|\cdot(2|b|+c)\leq c\cdot(2|b|+c)=2bc+c^2 $, therefore by \eqref{eqn:sig_diff} we have 
	\begin{equation}\label{eqn:second_diff}
	\begin{aligned}
	&\left(\left[\frac{1}{n_{s_t,a_t}}\sum_{i=1}^{n_{s_t,a_t}}V^{\text{in}}_{t+1}(s^{(i)}_{t+1}|s_t,a_t)\right]^2-\left[P^\top(\cdot|s_t,a_t)V^{\text{in}}_{t+1}\right]^2\right)\mathbf{1}(E_t)\\
	\leq &4P^\top(\cdot|s_t,a_t)V^{\text{in}}_{t+1}\cdot V_{\max}\sqrt{\frac{\log(4/\delta)}{m\cdot d^\mu_t(s_t,a_t)}}+\frac{4V_{\max}^2\log(4/\delta)}{m\cdot d^\mu_t(s_t,a_t)}\\
	\leq &4 V_{\max}^2\sqrt{\frac{\log(4/\delta)}{m\cdot d^\mu_t(s_t,a_t)}}+\frac{4V_{\max}^2\log(4/\delta)}{m\cdot d^\mu_t(s_t,a_t)}
	\end{aligned}
	\end{equation}
	where the last inequality comes from $|P^\top(\cdot|s_t,a_t)V^{\text{in}}_{t+1}|\leq ||P(\cdot|s_t,a_t)||_1 ||V^{\text{in}}_{t+1}||_\infty\leq V_{\max}$. Combining \eqref{eqn:first_diff}, \eqref{eqn:second_diff} and a union bound, we have with probability $1-\delta$, 
	\[
	\left|\tilde{\sigma}_{V^{\text{in}}_{t+1}}(s_t,a_t)-\sigma_{V_{t+1}^{\text{in}}}(s_t,a_t) \right|\leq 6 V_{\max}^2\sqrt{\frac{\log(4/\delta)}{m\cdot d^\mu_t(s_t,a_t)}}+\frac{4V_{\max}^2\log(4/\delta)}{m\cdot d^\mu_t(s_t,a_t)},
	\]
	apply again the union bound over $t,s_t,a_t$ gives the desired result.
	
\end{proof}

\begin{lemma}\label{lem:con_g}
	Fix time $t\in[H]$. Let $g_t$ be the estimator in \eqref{eqn:off_g} in Algorithm~\ref{alg:OPDVRT}. Then if $||V_{t+1}-V^{\text{in}}_{t+1}||_\infty\leq 2u^{\text{in}}$, then with probability $1-\delta/H$, 
	\[
	 \mathbf{0}\leq \PP_t[V_{t+1}-V^{\text{in}}_{t+1}]-g_t\leq 8u^{\text{in}}\sqrt{\frac{\log(2HSA/\delta)}{l d^\mu_t}}
	\]
\end{lemma}

\begin{proof}
	Recall $g_t,d^\mu_t$ are vectors. By definition of $g_t(s_t,a_t)$, applying Hoeffding's inequality we obtain with probability $1-\delta/H$
	\begin{align*}
	&g_t(s_t,a_t)+f(s_t,a_t)-P^\top(\cdot|s_t,a_t)[V_{t+1}-V^{\text{in}}_{t+1}]\\
	=&\left(\frac{1}{n'_{s_t,a_t}}\sum_{j=1}^l\left[V_{t+1}(s'^{(j)}_{t+1}|s_t,a_t)-V_{t+1}^{\text{in}}(s'^{(j)}_{t+1}|s_t,a_t)\right]-P^\top(\cdot|s_t,a_t)[V_{t+1}-V^{\text{in}}_{t+1}]\right)\cdot\mathbf{1}(E_t)\\
	\leq & \left(||V_{t+1}-V_{t+1}^{\text{in}}||_\infty\sqrt{\frac{2\log(2H/\delta)}{n'_{s_t,a_t}}}\right)\cdot\mathbf{1}(E_t)\\
	\leq & ||V_{t+1}-V_{t+1}^{\text{in}}||_\infty\sqrt{\frac{4\log(2H/\delta)}{l\cdot d^\mu_t(s_t,a_t)}}
	\end{align*}
	Now use assumption $||V_{t+1}-V_{t+1}^{\text{in}}||_\infty\leq 2u^{\text{in}}$ and a union bound over $s_t,a_t$, we have with probability $1-\delta/H$,
	\begin{equation}
	\left|g_t+f-\PP[V_{t+1}-V^{\tin}_{t+1}]\right|\leq 4u^{\text{in}}\sqrt{\frac{\log(2HSA/\delta)}{l d^\mu_t}}
	\end{equation}
	use $f=4u^{\text{in}}\sqrt{{\log(2HSA/\delta)}/{l d^\mu_t}}$, we obtain the stated result.
\end{proof}

\begin{remark}
	The marginal state-action distribution $d^\mu_t$ entails the hardness in off-policy setting. If the current logging policy $\mu$ satisfies there exists some $s_t,a_t$ such that $d^\mu_t(s_t,a_t)$ is very small, then learning the MDP using this off-policy data will be generically hard, unless $d^{\pi^*}_t(s_t,a_t)$ is also relatively small for this $s_t,a_t$, see analysis in the following sections.
\end{remark}

\subsection{Iterative update analysis}\label{sec:iterative_non}

The goal of iterative update is to obtain the recursive relation: $Q^\star_t-Q_t\leq \PP^{\pi^\star}_t[Q^\star_{t+1}-Q_{t+1}]+{\xi}_t$, where $\PP^{\pi^\star}_t\in\R^{SA\times SA}$ is a matrix. We control the error propagation term $\xi_t$  to be small enough.

\begin{lemma}\label{lem:VR_incremental}
	Let $Q^\star$ be the optimal $Q$-value satisfying $Q^\star_t=r+\PP_t V^\star_{t+1}$ and $\pi^\star$ is one optimal policy satisfying Assumption~\ref{assu2}. Let $\pi$ and $V_t$ be the \textbf{Return} of inner loop in Algorithm~\ref{alg:OPDVRT}, and recall $V_{H+1}=\mathbf{0}\in\R^{S}$, $Q_{H+1}=\mathbf{0}\in\R^{S\times A}$. We have with probability $1-\delta$, for all $t\in[H]$,
	\begin{align*}
		V^{\text{in}}_t\leq V_t\leq\mathcal{T}_{\pi_t}V_{t+1}\leq V^\star_t,\quad Q_t\leq r+\PP_t V_{t+1},\quad \text{and}\quad Q^\star_t-Q_t\leq \PP^{\pi^\star}_t[Q^\star_{t+1}-Q_{t+1}]+{\xi}_t,
	\end{align*} 
	where \begin{align*}
	{\xi}_t\leq&8u^{\tin}\sqrt{\frac{\log(2HSA/\delta)}{l d^\mu_t}}+\sqrt{\frac{16\cdot{\sigma}_{V^{\star}_{t+1}}\cdot\log(4HSA/\delta)}{m\cdot d_t^\mu}}+\sqrt{\frac{16\cdot\log(4HSA/\delta)}{m\cdot d_t^\mu}}\cdot u^{\tin}\\
	+& V_{\max}\left[8\sqrt{6}\cdot\left(\frac{\log(16HSA/\delta)}{m\cdot d^\mu_t}\right)^{3/4}+\frac{56\log(16HSA/\delta)}{3m\cdot d^\mu_t}\right].
	\end{align*} Here $\PP^{\pi^\star}\in\R^{S\cdot A\times S\cdot A}$ with $\PP^{\pi^\star}_{(s_t,a_t),(s_{t+1},a_{t+1})}=d^{\pi^\star}(s_{t+1},a_{t+1}|s_t,a_t)$.
\end{lemma}

\begin{proof}
	\textbf{Step1:} For any $a,b\geq0$, we have the basic inequality $\sqrt{a+b}\leq\sqrt{a}+\sqrt{b}$, and apply to Lemma~\ref{lem:con_sig} we have with probability $1-\delta/4$,
	\begin{equation}\label{eqn:diff_sig}
	\sqrt{\left|\tilde{\sigma}_{V^{\text{in}}_{t+1}}-\sigma_{V_{t+1}^{\text{in}}} \right|}\leq  V_{\max}\cdot\left(\frac{36\log(16HSA/\delta)}{m\cdot d^\mu_t}\right)^{1/4}+2V_{\max}\cdot\sqrt{\frac{\log(16HSA/\delta)}{m\cdot d^\mu_t}},\quad \forall t=1,...,H.
	\end{equation}
	Next, similarly for any $a,b\geq 0$, we have $\sqrt{a}\leq\sqrt{|a-b|}+\sqrt{b}$,  conditional on above then apply to Lemma~\ref{lem:con_z} (with probability $1-\delta/4$) and we obtain with probability $1-\delta/2$,
	\begin{align*}
	&\left|\tilde{z_t}-\PP_t V^{\text{in}}_{t+1}\right|\\
	\leq &\sqrt{\frac{4\cdot\sigma_{V^{\text{in}}_{t+1}}\cdot\log(4HSA/\delta)}{m\cdot d_t^\mu}}+\frac{4V_{\max}}{3m\cdot d^\mu_t}\log(4HSA/\delta)\\
	\leq&\left(\sqrt{\tilde{\sigma}_{V^{\text{in}}_{t+1}}}+\sqrt{\left|\tilde{\sigma}_{V^{\text{in}}_{t+1}}-\sigma_{V^{\text{in}}_{t+1}}\right|}\right)\sqrt{\frac{4\cdot\log(4HSA/\delta)}{m\cdot d_t^\mu}}+\frac{4V_{\max}}{3m\cdot d^\mu_t}\log(4HSA/\delta)\\
	= &\sqrt{\frac{4\cdot\tilde{\sigma}_{V^{\text{in}}_{t+1}}\cdot\log(4HSA/\delta)}{m\cdot d_t^\mu}}+\left(\sqrt{\left|\tilde{\sigma}_{V^{\text{in}}_{t+1}}-\sigma_{V^{\text{in}}_{t+1}}\right|}\right)\sqrt{\frac{4\cdot\log(4HSA/\delta)}{m\cdot d_t^\mu}}+\frac{4V_{\max}}{3m\cdot d^\mu_t}\log(4HSA/\delta)\\
	\leq &\sqrt{\frac{4\cdot\tilde{\sigma}_{V^{\text{in}}_{t+1}}\cdot\log(4HSA/\delta)}{m\cdot d_t^\mu}}+2\sqrt{6}\cdot V_{\max}\cdot\left(\frac{\log(16HSA/\delta)}{m\cdot d^\mu_t}\right)^{3/4}+\frac{16V_{\max}}{3m\cdot d^\mu_t}\log(16HSA/\delta).
	\end{align*}
	Since $e = \sqrt{4\cdot\tilde{\sigma}_{V^{\tin}_{t+1}}\cdot\log(4HSA/\delta)/(m\cdot d_t^\mu)}+2\sqrt{6}\cdot V_{\max}\cdot\left(\log(16HSA/\delta)/(m\cdot d^\mu_t)\right)^{3/4}+16V_{\max}\log(16HSA/\delta)/(3m\cdot d^\mu_t)$, from above we have 
	\begin{equation}\label{eqn:upper_z}
	z_t=\tilde{z}_t-e\leq \PP_t V^{\tin}_{t+1},
	\end{equation}
	and 
	\begin{equation}\label{eqn:lower_z}
	z_t\geq \PP_t V^{\tin}_{t+1}-2e.
	\end{equation}
	
	Next note $\sqrt{\sigma_{(\cdot)}}$ is a norm, so by norm triangle inequality (for the second inequality) and $\sqrt{a}\leq\sqrt{b}+\sqrt{|b-a|}$ with \eqref{eqn:diff_sig} (for the first inequality) we have 
	\begin{align*}
	\sqrt{\tilde{\sigma}_{V^{\tin}_{t+1}}}\leq&\sqrt{\sigma_{V_{t+1}^{\tin}}} +  V_{\max}\left[\left(\frac{36\log(16HSA/\delta)}{m\cdot d^\mu_t}\right)^{1/4}+\sqrt{\frac{4\log(16HSA/\delta)}{m\cdot d^\mu_t}}\right]\\
	\leq&\sqrt{\sigma_{V_{t+1}^{\star}}} +\sqrt{\sigma_{V_{t+1}^{\star}-V_{t+1}^{\tin}}} +  V_{\max}\left[\left(\frac{36\log(16HSA/\delta)}{m\cdot d^\mu_t}\right)^{1/4}+\sqrt{\frac{4\log(16HSA/\delta)}{m\cdot d^\mu_t}}\right]\\
	\leq&\sqrt{\sigma_{V_{t+1}^{\star}}} +\sqrt{\PP_t(V_{t+1}^{\star}-V_{t+1}^{\tin})^2} +  V_{\max}\left[\left(\frac{36\log(16HSA/\delta)}{m\cdot d^\mu_t}\right)^{1/4}+\sqrt{\frac{4\log(16HSA/\delta)}{m\cdot d^\mu_t}}\right]\\
	\leq&\sqrt{\sigma_{V_{t+1}^{\star}}} +||V_{t+1}^{\star}-V_{t+1}^{\tin}||_\infty\cdot\mathbf{1} +  V_{\max}\left[\left(\frac{36\log(16HSA/\delta)}{m\cdot d^\mu_t}\right)^{1/4}+\sqrt{\frac{4\log(16HSA/\delta)}{m\cdot d^\mu_t}}\right]\\
	\leq&\sqrt{\sigma_{V_{t+1}^{\star}}} +u^{\tin}\cdot\mathbf{1} +  V_{\max}\left[\left(\frac{36\log(16HSA/\delta)}{m\cdot d^\mu_t}\right)^{1/4}+\sqrt{\frac{4\log(16HSA/\delta)}{m\cdot d^\mu_t}}\right]\\
	\end{align*}
	Plug this back to \eqref{eqn:lower_z} we get
	\begin{equation}\label{eqn:lower_zt}
	\begin{aligned}
	z_t\geq& \PP_t V^{\tin}_{t+1}-\sqrt{\frac{16\cdot{\sigma}_{V^{\star}_{t+1}}\cdot\log(4HSA/\delta)}{m\cdot d_t^\mu}}-\sqrt{\frac{16\cdot\log(4HSA/\delta)}{m\cdot d_t^\mu}}\cdot u^{\tin}\\
	-& V_{\max}\left[8\sqrt{6}\cdot\left(\frac{\log(16HSA/\delta)}{m\cdot d^\mu_t}\right)^{3/4}+\frac{56\log(16HSA/\delta)}{3m\cdot d^\mu_t}\right].
	\end{aligned}
	\end{equation}
	To sum up, so far we have shown that \eqref{eqn:upper_z}, \eqref{eqn:lower_zt} hold with probability $1-\delta/2$ and we condition on that. 
	
	\textbf{Step2:} Next we prove 
	\begin{equation}\label{eqn:inter}
	 Q_t\leq r+\PP_t V_{t+1},\quad V^{\tin}_t\leq V_t\leq V^\star_t,\quad\forall t\in[H]
	\end{equation}
	using backward induction. 
	
	First of all, $V^\star_{H+1}=V_{H+1}=V_{H+1}^{\tin}=0$ implies $V^\star_{H+1}\leq V_{H+1}\leq V_{H+1}^{\tin}$ and 
	\[
	Q_H=r+z_H+g_H=r+(\mathbf{0}-e)+(\mathbf{0}-f)\leq r=r+\PP_H^\top\mathbf{0}=r+\PP_H^\top V_{H+1},
	\]  
	so the results hold for the base case. 
	
	Now for certain $t$, using induction assumption we can assume with probability at least $1-(H-t-1)\delta/H$, for all $t'=t+1,...,H$,  
	\begin{equation}\label{eqn:intermediate}
	V^{\tin}_{t'}\leq V_{t'}\leq V^\star_{t'},\quad Q_{t'}\leq r+\PP_{t'} V_{{t'}+1}.
	\end{equation}
	In particular, since $V^{\tin}_{t+1}\leq V_{t+1}^\star\leq V^{\tin}_{t+1}+u^{\tin}\mathbf{1}$, so combine this and \eqref{eqn:intermediate} for $t'=t+1$ we get
	\[
	V^\star_{t+1}-V_{t+1}\leq V^\star_{t+1}- V^{\tin}_{t+1}\leq u^{\tin}\mathbf{1}.
	\]
	By Lemma~\ref{lem:con_g}, with probability $1-\delta/H$,
	\begin{equation}\label{eqn:diff_g_t}
	\PP_t[V_{t+1}-V^{\tin}_{t+1}]- 8u^{\tin}\sqrt{\frac{\log(2HSA/\delta)}{l d^\mu_t}}\leq g_t\leq \PP_t[V_{t+1}-V^{\tin}_{t+1}].
	\end{equation}
	By the right hand side of above and \eqref{eqn:upper_z} we acquire with probability $1-(H-t)\delta/H$,
	\[
	Q_t=r+z_t+g_t\leq  r+ \PP_t V^{\tin}_{t+1}+ \PP_t[V_{t+1}-V^{\tin}_{t+1}] =  r+\PP_t V_{t+1}\leq r+\PP_t V_{t+1}^\star= Q_t^\star
	\]
	where the second equality already gives the proof of the first part of claim~\eqref{eqn:inter} and the second inequality is by induction assumption. Moreover, above $Q_t\leq Q_t^\star$ also implies $V_{Q_t}\leq V_{Q_t^\star}=V_t^\star$, so together with Lemma~\ref{lem:mono} (note $V_t^{\tin}\leq\mathcal{T}_{\pi_t^{\tin}}V_{t+1}^{\tin}$) we have 
	\[
	V_t=\max(V_{Q_t},V^{\tin}_t)\leq V^\star_t,
	\]	
	this completes the proof of the second part of claim~\eqref{eqn:inter}. 
	
 \textbf{Step3:} Next we prove $V_t\leq\mathcal{T}_{\pi_t}V_{t+1}$. 
 
 For a particular $s_t$, on one hand, if $\pi_t(s_t)=\argmax_{a_t} Q_t(s_t,a_t)$, by $Q_t\leq r+\PP_t V_{t+1}$ we have in this case:  
	\[
	V_t(s_t)=\max_{a_t} Q_t(s_t,a_t)=Q_t(s_t,\pi_t(s_t))\leq r(s_t,\pi_t(s_t))+P^\top(\cdot|s_t,\pi_t(s_t))V_{t+1}=(\mathcal{T}_{\pi_t}V_{t+1})(s_t),
	\]
	where the first equal sign comes from the definition of $V_t$ when $V_{Q_t}(s_t)\geq V^{\tin}_t(s_t)$ and the first inequality is from Step2.
	
	On the other hand, if $\pi_t(s_t)=\pi^{\tin}(s_t)$, then
	\[
	V_t(s_t)=V^{\tin}_t(s_t)\leq (\mathcal{T}_{\pi^{\tin}_t}V_{t+1}^{\tin})(s_t)\leq(\mathcal{T}_{\pi^{\tin}_t}V_{t+1})(s_t)=(\mathcal{T}_{\pi_t}V_{t+1})(s_t)
	\]
	where the first inequality is the property of input $V^{\tin}$, $\pi^{\tin}$ and the second inequality is from Step2.

\textbf{Step4:}	It remains to prove $Q^\star_t-Q_t\leq \PP^{\pi^\star}_t[Q^\star_{t+1}-Q_{t+1}]+{\xi}_t$. Indeed, using the construction of $Q_t$, we have
	\begin{equation}\label{eqn:diff_q_t}
	\begin{aligned}
	&Q^\star_t-Q_t=Q^\star_t-r-z_t-g_t=\PP_t V_{t+1}^\star-z_t-g_t\\
	=&\PP_t V_{t+1}^\star-\PP_t(V_{t+1}-V^{\tin}_{t+1})-\PP_t V_{t+1}^{\tin}+{\xi}_t=\PP_t V_{t+1}^\star-\PP_t V_{t+1}+{\xi_t},
	\end{aligned}
	\end{equation}
	where the second equation uses Bellman optimality equation and the third equation uses the definition of ${\xi}_t=\PP_t(V_{t+1}-V^{\tin}_{t+1})-g_t+\PP_t V_{t+1}^{\tin}-z_t$. By \eqref{eqn:lower_zt} and \eqref{eqn:diff_g_t},
	\begin{align*}
	{\xi}_t\leq&8u^{\tin}\sqrt{\frac{\log(2HSA/\delta)}{l d^\mu_t}}+\sqrt{\frac{16\cdot{\sigma}_{V^{\star}_{t+1}}\cdot\log(4HSA/\delta)}{m\cdot d_t^\mu}}+\sqrt{\frac{16\cdot\log(4HSA/\delta)}{m\cdot d_t^\mu}}\cdot u^{\tin}\\
	+& V_{\max}\left[8\sqrt{6}\cdot\left(\frac{\log(16HSA/\delta)}{m\cdot d^\mu_t}\right)^{3/4}+\frac{56\log(16HSA/\delta)}{3m\cdot d^\mu_t}\right].
	\end{align*}
	Lastly, note 
	$
	\PP_t V_{t+1}^\star=\PP^{\pi^\star}_t Q^\star_{t+1}
	$
	and by definition $V_{t+1}\geq V_{Q_{t+1}}$, so we have $\PP_t V_{t+1}\geq\PP_t V_{Q_{t+1}}=\PP^{\pi_{Q_{t+1}}}_t Q_{t+1}\geq \PP^{\pi^\star}_t Q_{t+1}$, the last inequality holds true since $\pi_{Q_{t+1}}$ is the greedy policy over $Q_{t+1}$. Threfore \eqref{eqn:diff_q_t} becomes $Q^\star_t-Q_t=\PP_t V_{t+1}^\star-\PP_t V_{t+1}+{\xi_t}\leq\PP^{\pi^\star}_t Q^\star_{t+1}-\PP^{\pi^\star}_t Q_{t+1}+{\xi_t}$. This completes the proof.

\end{proof}

\begin{lemma}\label{lem:complexity}
	Suppose the input $V^{\tin}_t$, $t\in[H]$ of Algorithm~\ref{alg:OPDVRT} satisfies $V^{\tin}_t\leq\mathcal{T}_{\pi^{\tin}_t}V_{t+1}^{\tin}$ and $V^{\tin}_t\leq V^\star_t\leq V^{\tin}_t+u^{\tin}\mathbf{1}$. Let $V_t$, $\pi$ be the return of inner loop of Algorithm~\ref{alg:OPDVRT} and choose $m=l:=m' \cdot \log(16HSA)/(u^{\tin})^2$, where $m'$ is a parameter will be decided later. Then in addition to the results of Lemma~\ref{lem:VR_incremental}, we have with probability $1-\delta$, 
	\begin{itemize}
		\item if $u^{\tin}\in[\sqrt{H},H]$, then: 
		\begin{align*}
		&\mathbf{0}\leq V^\star_t-V_t\leq\\
		\leq &\bigg(\frac{12H^2}{\sqrt{m' }}\norm{\d^{\pi^\star}_{t:t'}\sqrt{\frac{1}{d^\mu_{t'}}}}_{\infty,H}+\frac{4}{\sqrt{m'}}\norm{\sum_{t'=t}^{H}\d^{\pi^\star}_{t:t'}\sqrt{\frac{{\sigma}_{V^{\star}_{t'+1}}}{d_{t'}^\mu}}}_{\infty} + \frac{8\sqrt{6}H^{\frac{10}{4}}}{(m')^{3/4}}\norm{\d^{\pi^\star}_{t:t'}\sqrt{\frac{1}{d^\mu_{t'}}}}_{\infty,H}\\
		+& \frac{56 H^{3}}{3m'}   \norm{\d^{\pi^\star}_{t:t'}\frac{1}{d^\mu_{t'}}}_{\infty,H}\bigg)u^{\tin}\cdot\mathbf{1}.
		\end{align*}
		\item if $u^{\tin}\leq \sqrt{H}$, then 
		\begin{align*}
		&\mathbf{0}\leq V^\star_t-V_t\leq\\
		\leq &\bigg(\frac{12\sqrt{H^3}}{\sqrt{m' }}\norm{\d^{\pi^\star}_{t:t'}\sqrt{\frac{1}{d^\mu_{t'}}}}_{\infty,H}+\frac{4}{\sqrt{m'}}\norm{\sum_{t'=t}^{H}\d^{\pi^\star}_{t:t'}\sqrt{\frac{{\sigma}_{V^{\star}_{t'+1}}}{d_{t'}^\mu}}}_{\infty} + \frac{8\sqrt{6}H^{\frac{9}{4}}}{(m')^{3/4}}\norm{\d^{\pi^\star}_{t:t'}\sqrt{\frac{1}{d^\mu_{t'}}}}_{\infty,H}\\
		+& \frac{56 H^{\frac{5}{2}}}{3m'}   \norm{\d^{\pi^\star}_{t:t'}\frac{1}{d^\mu_{t'}}}_{\infty,H}\bigg)u^{\tin}\cdot\mathbf{1}.
		\end{align*}
	\end{itemize}

where  $\d^{\pi^\star}_{t:t'}\in\R^{S\cdot A\times S\cdot A}$ is a matrix represents the multi-step transition from time $t$ to $t'$, \emph{i.e.} 
$\d^{\pi^\star}_{(s_t,a_t),(s_{t'},a_{t'})}=d^{\pi^\star}_{t:t'}(s_{t'},a_{t'}|s_t,a_t)$ and recall $1/d^\mu_{t'}$ is a vector. $\d^{\pi^\star}_{t:t'}\frac{1}{d^\mu_{t'}}$ is a matrix-vector multiplication. For a vector $d_t\in \R^{S\times A}$, norm $||\cdot||_{\infty,H}$ is defined as $||d_t||_{\infty,H}=\max_{t,s_t,a_t}d_t(s_t,a_t)$. 
\end{lemma}

\begin{remark}
	Note if $u^{\tin}\geq\sqrt{H}$, the first term in $V_t-V_t^\star$ requires sample $m'$ of order $O(H^4)$, which is suboptimal. This is the main reason why we need the doubling procedure in Algorithm~\ref{alg:OPDVR} to keep the whole algorithm optimal.
\end{remark}

\begin{proof}
	By Lemma~\ref{lem:VR_incremental}, we have with probability $1-\delta$, for all $t\in[H]$,
	\begin{align*}
	V^{\text{in}}_t\leq V_t\leq\mathcal{T}_{\pi_t}V_{t+1}\leq V^\star_t,\quad Q_t\leq r+\PP_t V_{t+1},\quad \text{and}\quad Q^\star_t-Q_t\leq \PP^{\pi^\star}_t[Q^\star_{t+1}-Q_{t+1}]+{\xi}_t,
	\end{align*} 
	where \begin{align*}
	{\xi}_t\leq&8u^{\tin}\sqrt{\frac{\log(2HSA/\delta)}{l d^\mu_t}}+\sqrt{\frac{16\cdot{\sigma}_{V^{\star}_{t+1}}\cdot\log(4HSA/\delta)}{m\cdot d_t^\mu}}+\sqrt{\frac{16\cdot\log(4HSA/\delta)}{m\cdot d_t^\mu}}\cdot u^{\tin}\\
	+& V_{\max}\left[8\sqrt{6}\cdot\left(\frac{\log(16HSA/\delta)}{m\cdot d^\mu_t}\right)^{3/4}+\frac{56\log(16HSA/\delta)}{3m\cdot d^\mu_t}\right].
	\end{align*} 
	
	Applying the recursion repeatedly, we obtain
	\[
	Q^\star_t-Q_t\leq\sum_{t'=t}^{H}\left[\prod_{i=t}^{t'-1}\PP^{\pi^\star}_i\right]{\xi}_{t'}
	\]
	Note $\prod_{i=t}^{t'-1}\PP^{\pi^\star}_i\in\R^{S\cdot A\times S\cdot A}$ represents the multi-step transition from time $t$ to $t'$, \emph{i.e.} 
	$(\prod_{i=t}^{t'-1}\PP^{\pi^\star}_i)_{(s_t,a_t),(s_{t'},a_{t'})}=d^{\pi^\star}_{t:t'}(s_{t'},a_{t'}|s_t,a_t)$. Therefore 
	{
		\begin{equation}\label{eqn:q_diff_decomp}
		\begin{aligned}
		&Q^\star_t-Q_t\leq\sum_{t'=t}^{H}\left[\prod_{i=t}^{t'-1}\PP^{\pi^\star}_i\right]{\xi}_{t'}=\sum_{t'=t}^{H}\d^{\pi^\star}_{t:t'}{\xi}_{t'}\\
		\leq&\sum_{t'=t}^{H}\d^{\pi^\star}_{t:t'}\bigg(8u^{\tin}\sqrt{\frac{\log(2HSA/\delta)}{l d^\mu_{t'}}}+\sqrt{\frac{16\cdot{\sigma}_{V^{\star}_{t'+1}}\cdot\log(4HSA/\delta)}{m\cdot d_{t'}^\mu}}+\sqrt{\frac{16\cdot\log(4HSA/\delta)}{m\cdot d_{t'}^\mu}}\cdot u^{\tin}\\
		+& V_{\max}\left[8\sqrt{6}\cdot\left(\frac{\log(16HSA/\delta)}{m\cdot d^\mu_{t'}}\right)^{3/4}+\frac{56\log(16HSA/\delta)}{3m\cdot d^\mu_{t'}}\right]\bigg)
		\end{aligned}
		\end{equation}
	}
	\normalsize{}
	Now by our choice of $m=l:=m' \cdot \log(16HSA/\delta)/(u^{\tin})^2$, then \eqref{eqn:q_diff_decomp} further less than

	\begin{equation}\label{eqn:decomp_genereal}
	\begin{aligned}
	\leq &\sum_{t'=t}^{H}\d^{\pi^\star}_{t:t'}\left(\frac{12u^{\tin}}{\sqrt{m' d^\mu_{t'}}}u^{\tin}+\sqrt{\frac{16\cdot{\sigma}_{V^{\star}_{t'+1}}}{m'\cdot d_{t'}^\mu}}u^{\tin}
	+ V_{\max}\left[8\sqrt{6}\cdot\left(\frac{(u^{\tin})^{ 2/3}}{m'\cdot d^\mu_{t'}}\right)^{3/4}+\frac{56u^{\tin}}{3m'\cdot d^\mu_{t'}}\right]\cdot u^{\tin}\right)\\
	\end{aligned}
	\end{equation}
	
	\textbf{Case1.} If $u^{\tin}\leq \sqrt{H}$, then \eqref{eqn:decomp_genereal} is less than 
	\begin{equation}\label{eqn:decomp_genereal_case1}
	\begin{aligned}
	\leq &\sum_{t'=t}^{H}\d^{\pi^\star}_{t:t'}\left(\frac{12\sqrt{H}}{\sqrt{m' d^\mu_{t'}}}+\sqrt{\frac{16\cdot{\sigma}_{V^{\star}_{t'+1}}}{m'\cdot d_{t'}^\mu}}
	+ V_{\max}\left[8\sqrt{6}\cdot\left(\frac{H^{ 1/3}}{m'\cdot d^\mu_{t'}}\right)^{3/4}+\frac{56H^{1/2}}{3m'\cdot d^\mu_{t'}}\right] \right)u^{\tin}\\
	\leq &\bigg(\frac{12\sqrt{H^3}}{\sqrt{m' }}\norm{\d^{\pi^\star}_{t:t'}\sqrt{\frac{1}{d^\mu_{t'}}}}_{\infty,H}+\frac{4}{\sqrt{m'}}\norm{\sum_{t'=t}^{H}\d^{\pi^\star}_{t:t'}\sqrt{\frac{{\sigma}_{V^{\star}_{t'+1}}}{d_{t'}^\mu}}}_{\infty} + \frac{8\sqrt{6}H^{\frac{9}{4}}}{(m')^{3/4}}\norm{\d^{\pi^\star}_{t:t'}\left[\frac{1}{d^\mu_{t'}}\right]^{\frac{3}{4}}}_{\infty,H}\\
	+& \frac{56 H^{\frac{5}{2}}}{3m'}   \norm{\d^{\pi^\star}_{t:t'}\frac{1}{d^\mu_{t'}}}_{\infty,H}\bigg)u^{\tin}\cdot\mathbf{1}.
	\end{aligned}
	\end{equation}
	
	\textbf{Case2.} If $u^{\tin}\geq \sqrt{H}$, then \eqref{eqn:decomp_genereal} is less than 
	\begin{equation}\label{eqn:decomp_genereal_case2}
	\begin{aligned}
	\leq &\sum_{t'=t}^{H}\d^{\pi^\star}_{t:t'}\left(\frac{12{H}}{\sqrt{m' d^\mu_{t'}}}+\sqrt{\frac{16\cdot{\sigma}_{V^{\star}_{t'+1}}}{m'\cdot d_{t'}^\mu}}
	+ V_{\max}\left[8\sqrt{6}\cdot\left(\frac{H^{ 2/3}}{m'\cdot d^\mu_{t'}}\right)^{3/4}+\frac{56H}{3m'\cdot d^\mu_{t'}}\right] \right)u^{\tin}\\
	\leq &\bigg(\frac{12H^2}{\sqrt{m' }}\norm{\d^{\pi^\star}_{t:t'}\sqrt{\frac{1}{d^\mu_{t'}}}}_{\infty,H}+\frac{4}{\sqrt{m'}}\norm{\sum_{t'=t}^{H}\d^{\pi^\star}_{t:t'}\sqrt{\frac{{\sigma}_{V^{\star}_{t'+1}}}{d_{t'}^\mu}}}_{\infty} + \frac{8\sqrt{6}H^{\frac{10}{4}}}{(m')^{3/4}}\norm{\d^{\pi^\star}_{t:t'}\left[\frac{1}{d^\mu_{t'}}\right]^{\frac{3}{4}}}_{\infty,H}\\
	+& \frac{56 H^{3}}{3m'}   \norm{\d^{\pi^\star}_{t:t'}\frac{1}{d^\mu_{t'}}}_{\infty,H}\bigg)u^{\tin}\cdot\mathbf{1}.
	\end{aligned}
	\end{equation}

\end{proof}

\normalsize{}
\subsection{The doubling procedure}

Before we explain the doubling procedure, let us first finish the proof the Algorithm~\ref{alg:OPDVRT}.
	
	\begin{lemma}\label{lem:complexity_OPVRT}
		For convenience, define:
		\[
		A_{\frac{1}{2}}= \norm{\d^{\pi^\star}_{t:t'}\sqrt{\frac{1}{d^\mu_{t'}}}}_{\infty,H}, \;A_2=\norm{\sum_{t'=t}^{H}\d^{\pi^\star}_{t:t'}\sqrt{\frac{{\sigma}_{V^{\star}_{t'+1}}}{d_{t'}^\mu}}}_{\infty},\;A_{\frac{3}{4}}=\norm{\d^{\pi^\star}_{t:t'}\left[\frac{1}{d^\mu_{t'}}\right]^{\frac{3}{4}}}_{\infty,H},\;\;A_1=\norm{\d^{\pi^\star}_{t:t'}\frac{1}{d^\mu_{t'}}}_{\infty,H}.
		\] 
		Recall $\epsilon$ is the target accuracy in the outer loop of Algorithm~\ref{alg:OPDVRT}. Then:
		\begin{itemize}
			\item If $u^{(0)}\leq \sqrt{H}$, then choose $m^{(i)}=l^{(i)}=B\log(16HSAK/\delta)/(u^{(i-1)})^2$, where 
			\[
			B=\max\left[ 96^2H^3A_{\frac{1}{2}}^2,32^2A_2^2,\left(64\sqrt{6}A_{\frac{3}{4}}\right)^{\frac{4}{3}}H^3,\frac{448}{3}H^{5/2}A_1\right], \quad K=\log_2(\sqrt{H}/\epsilon),
			\]
			
			\item If $u^{(0)}> \sqrt{H}$, then choose $m^{(i)}=l^{(i)}=B\log(16HSAK/\delta)/(u^{(i-1)})^2$, where 
			\[
			B=\max\left[ 96^2H^4 A_{\frac{1}{2}}^2,32^2A_2^2,\left(64\sqrt{6}A_{\frac{3}{4}}\right)^{\frac{4}{3}}H^{\frac{10}{3}},\frac{448}{3}H^{3}A_1\right], \quad K=\log_2({H}/\epsilon),
			\]
		\end{itemize}
	Then Algorithm~\ref{alg:OPDVRT} guarantees with probability $1-\delta$, the output $\pi^{(K)}$ is a $\epsilon$-optimal policy, i.e. $||V_1^\star-V_1^{\pi^{(K)}}||_\infty <\epsilon$ with total episode complexity:
	\[
	\frac{2B\log(16HSAK/\delta)}{\epsilon^2}K
	\]
	for both cases. Moreover, $B$ can be simplified as:
	\begin{itemize}
		\item If $u^{(0)}\leq \sqrt{H}$,then $B\leq cH^3/d_m$;
		\item If $u^{(0)}> \sqrt{H}$, then $B\leq c H^4/d_m$.
	\end{itemize}
	\end{lemma}

\begin{proof}[Proof of Lemma~\ref{lem:complexity_OPVRT}]
\textbf{Step1: proof in general.}	First, using induction it is easy to show for all $0<a_1,...,a_n<1$, it follows
	\[
	(1-a_1)\cdot(1-a_2)\cdot ...\cdot(1-a_n)\geq1-(a_1+...+a_n).
	\]
	and this directly implies $(1-\frac{\delta}{K})^K\geq 1-\delta$. By the choice of $m'$ and $K$, for both situation by Lemma~\ref{lem:complexity} we always have $||V_t^\star-V_t^{\pi^{(i)}}||_\infty<u^{(i-1)}/2=u^{(i)}$ with probability $1-\delta/K$ (this is because we choose $m^{(i)}=l^{(i)}=B\log(16HSAK/\delta)/(u^{(i-1)})^2$ instead of $B\log(16HSA/\delta)/(u^{(i-1)})^2$).

	Therefore by chain rule of probability, 
	\begin{align*}
	&\P\left(\forall i\in[K],t\in[H],\;V^\star_t-V^{\pi^{(i)}}_t\leq u^{(i)}\right)\\
	=&\prod_{j=2}^K\P\left(\forall t\in[H],\;V^\star_t-V^{\pi^{(j)}}_t\leq u^{(j)}\middle|\forall i\in[j-1],t\in[H],\;V^\star_t-V^{\pi^{(i)}}_t\leq u^{(i)}\right)\\
	\times&\P\left(\forall t\in[H],\;V^\star_1-V^{\pi^{(1)}}_1\leq u^{(1)}\right)\\
	\geq& (1-\frac{\delta}{K})^K\geq 1-\delta.
	\end{align*}
	In particular, in both situation\footnote{The last equal sign holds since if $u^{(0)}\leq\sqrt{H}$ (or $u^{(0)}\leq {H}$), you can always reset $u^{(0)}=\sqrt{H}$ (or $u^{(0)}= {H}$).}
	\[
	\forall t\in[H],\;V^\star_t-V^{\pi^{(K)}}_t\leq u^{(K)}=u^{(0)}\cdot 2^{-K} = \epsilon,
	\]
	with total number of budget to be 
	\[
	\sum_{i=1}^K (m^{(i)}+l^{(i)})=2\sum_{i=1}^K \frac{B\log(16HSAK/\delta)}{(u^{(i-1)})^2}\leq 2\sum_{i=1}^K \frac{B\log(16HSAK/\delta)}{(u^{(0)}\cdot 2^{-K} )^2}=\frac{2B\log(16HSAK/\delta)}{\epsilon^2}K
	\] 
	
	\textbf{Step2: simplified expression for $m'$.} Indeed,
	\begin{align*}
		\d^{\pi^\star}_{t:t'}\sqrt{\frac{1}{d^\mu_{t'}}}\leq \d^{\pi^\star}_{t:t'}\sqrt{\frac{1}{d_m}\cdot\mathbf{1}}\leq \sqrt{\frac{1}{d_m}}||\d^{\pi^\star}_{t:t'}||_1 \cdot ||\mathbf{1}||_\infty\cdot\mathbf{1}=\sqrt{\frac{1}{d_m}}\cdot\mathbf{1}&\Rightarrow A_{\frac{1}{2}}\leq\sqrt{\frac{1}{d_m}};\\
		\d^{\pi^\star}_{t:t'}\left[\frac{1}{d^\mu_{t'}}\right]^\frac{3}{4}\leq \d^{\pi^\star}_{t:t'}\left[\frac{1}{d_m}\right]^\frac{3}{4}\cdot\mathbf{1}\leq \left[\frac{1}{d_m}\right]^\frac{3}{4}||\d^{\pi^\star}_{t:t'}||_1 \cdot ||\mathbf{1}||_\infty\cdot\mathbf{1}=\left[\frac{1}{d_m}\right]^\frac{3}{4}\cdot\mathbf{1}&\Rightarrow A_{\frac{3}{4}}\leq\left[\frac{1}{d_m}\right]^\frac{3}{4};\\
		\d^{\pi^\star}_{t:t'}\frac{1}{d^\mu_{t'}}\leq \d^{\pi^\star}_{t:t'}\frac{1}{d_m}\cdot\mathbf{1}\leq \frac{1}{d_m}||\d^{\pi^\star}_{t:t'}||_1 \cdot ||\mathbf{1}||_\infty\cdot\mathbf{1}=\frac{1}{d_m}\cdot\mathbf{1}&\Rightarrow A_{{1}}\leq\frac{1}{d_m};\\
	\end{align*}
	and 
	\begin{align*}
&\sum_{t'=t}^{H}\sum_{s_{t'},a_{t'}}d^{\pi^\star}_{t:t'}(s_{t'},a_{t'}|s_t,a_t)\sqrt{\frac{{\sigma}_{V^{\star}_{t'+1}}(s_{t'},a_{t'})}{d_{t'}^\mu(s_{t'},a_{t'})}}\\
=&\sum_{t'=t}^{H}\sum_{s_{t'},a_{t'}}\sqrt{d^{\pi^\star}_{t:t'}(s_{t'},a_{t'}|s_t,a_t)}\sqrt{\frac{{\sigma}_{V^{\star}_{t'+1}}(s_{t'},a_{t'})d^{\pi^\star}_{t:t'}(s_{t'},a_{t'}|s_t,a_t)}{d_{t'}^\mu(s_{t'},a_{t'})}}\\
\leq&\sqrt{\frac{1}{d_m}}\sum_{t'=t}^{H}\sum_{s_{t'},a_{t'}}\sqrt{d^{\pi^\star}_{t:t'}(s_{t'},a_{t'}|s_t,a_t)}\sqrt{{{\sigma}_{V^{\star}_{t'+1}}(s_{t'},a_{t'})d^{\pi^\star}_{t:t'}(s_{t'},a_{t'}|s_t,a_t)}}\\
\explain{\leq}{CS\;Ineq}&\sqrt{\frac{1}{d_m}}\sum_{t'=t}^{H}\sqrt{\sum_{s_{t'},a_{t'}}d^{\pi^\star}_{t:t'}(s_{t'},a_{t'}|s_t,a_t)\sum_{s_{t'},a_{t'}}{{\sigma}_{V^{\star}_{t'+1}}(s_{t'},a_{t'})d^{\pi^\star}_{t:t'}(s_{t'},a_{t'}|s_t,a_t)}}\\
=&\sqrt{\frac{1}{d_m}}\sum_{t'=t}^{H}\sqrt{1\cdot\E^{\pi^\star}_{s_{t'},a_{t'}}\left[{\sigma}_{V^{\star}_{t'+1}}(s_{t'},a_{t'})\middle|s_t,a_t\right]}\\
\explain{\leq}{CS\;Ineq}&\sqrt{\frac{1}{d_m}}\sqrt{\sum_{t'=t}^{H}1\cdot \sum_{t'=t}^{H}\E^{\pi^\star}_{s_{t'},a_{t'}}\left[{\sigma}_{V^{\star}_{t'+1}}(s_{t'},a_{t'})\middle|s_t,a_t\right]}\\
\explain{\leq}{lem~\ref{lem:H3}}&\sqrt{\frac{1}{d_m}}\sqrt{\sum_{t'=t}^{H}1\cdot \Var_{\pi^\star}\left[\sum_{t'=t}^{H}r_{t'}\middle|s_t,a_t\right]}\leq \sqrt{\frac{H^3}{d_m}}\Rightarrow A_2\leq\sqrt{\frac{H^3}{d_m}},
	\end{align*}
	Plug all these numbers back, we have the simplified bound for $B$. 
\end{proof}
\begin{remark}
	The Assumption~\ref{assu2} comes into picture for the validity of the bound for $A_2$ since when $d^{\pi^\star}_{t:t'}(s'_t,a'_t|s_t,a_t)>0$, by Assumption~\ref{assu2} we always have $d_{t'}^\mu(s'_t,a'_t)>0$ so the bound will never be the trivial $+\infty$.
\end{remark}

\begin{corollary}
	Note choose any $m'>B$ (in Lemma~\ref{lem:complexity_OPVRT}) yields the similar complexity bound of 
	\[
	\frac{2m'\log(16HSAK/\delta)}{\epsilon^2}K,
	\]
	therefore by the simplified bound of $B$, we choose $m'=O(H^4/d_m)$ for stage1 and $m'=O(H^3/d_m)$ for stage2.
\end{corollary}

\paragraph{The doubling procedure.}
As we can see in Lemma~\ref{lem:complexity_OPVRT}, if the initial input $V^{(0)}_t$ in Algorithm~\ref{alg:OPDVR} has $\sup_t||V^{(0)}_t-V^\star_t||_\infty\geq \sqrt{H}$ (\emph{i.e.} $u^{(0)}>\sqrt{H}$), then it requires total of $\tilde{O}(H^4/d_m\epsilon^2)$ episodes to obtain $\epsilon$ accuracy, which is suboptimal. The doubling procedure helps resolve the problem. Concretely, for any final accuracy $0<\epsilon\leq1$:
\begin{itemize}
	\item \textbf{Stage1.} Denote $\epsilon'=\sqrt{H}\epsilon$ and $u^{(0)}=H$, then by the choice of $K$ and $m'_H=cH^4/d_m$ for the case of $u^{(0)}\geq\sqrt{H}$ in Lemma~\ref{lem:complexity_OPVRT}, it outputs $V_t^{\text{intermediate}}$, $\pi^{\text{intermediate}}$ which is $\epsilon'$ optimal with complexity:
	\[
	\frac{2m'_{H}\log(16HSAK_{\epsilon'}/\delta)}{\epsilon'^2}K_{\epsilon'}
	\] 
	where $K_{\epsilon'}=\log_2(H/\epsilon')$;
	\item \textbf{Stage2.} Use $V_t^{\text{intermediate}}$, $\pi^{\text{intermediate}}$ as input, since $\epsilon'=\sqrt{H}\epsilon\leq \sqrt{H}$, we can set $u^{(0)}=\sqrt{H}$. Now by Lemma~\ref{lem:complexity_OPVRT} again (with $m'_{\sqrt{H}}=cH^3/d_m$), Algorithm~\ref{alg:OPDVR} has the final output $V_t^{\text{final}}$, $\pi^{\text{final}}$ that is $\epsilon$ optimal with complexity 
	\[
	\frac{2m'_{\sqrt{H}}\log(16HSAK_{\epsilon}/\delta)}{\epsilon^2}K_{\epsilon}.
	\] 
	where $K_{\epsilon}=\log_2(\sqrt{H}/\epsilon)$.
\end{itemize}
Plug back $\epsilon'=\sqrt{H}\epsilon$, Algorithm~\ref{alg:OPDVR} guarantees $\epsilon$-optimal policy with probability $1-\delta$ using total complexity
\begin{equation}\label{eqn:analyze_complexity}
\begin{aligned}
&\frac{2m'_{H}\log(16HSAK_{\epsilon'}/\delta)}{\epsilon'^2}K_{\epsilon'}+\frac{2m'_{\sqrt{H}}\log(16HSAK_{\epsilon}/\delta)}{\epsilon^2}K_{\epsilon}\\
\leq&\frac{4\max[\frac{m'_{H}}{H},m'_{\sqrt{H}}]\log(16HSA\log_2(\sqrt{H}/\epsilon)/\delta)}{\epsilon^2}\log_2(\sqrt{H}/\epsilon)\\
\leq&O\left(\frac{H^3\log(16HSA\log_2(\sqrt{H}/\epsilon)/\delta)}{d_m\epsilon^2}\log_2(\sqrt{H}/\epsilon)\right)\\
\end{aligned}
\end{equation}
where the last inequality uses $m'_{\sqrt{H}}\leq cH^3/d_m$ and $m'_H\leq cH^4/d_m$ in Lemma~\ref{lem:complexity_OPVRT} and above holds with probability $1-\delta$ .

This provides the minimax optimality of $\tilde{O}(H^3/d_m\epsilon^2)$ for non-stationary setting.

\subsection{Practical OPDVR}\label{sec:app_practical} 
To go from non-implementable version to the practical version, the idea is to bound the event $\{n_{s_{t},a_{t}}\leq \frac{1}{2} m\cdot d^\mu_{t}(s_{t},a_{t})\}$ and $\{n'_{s_{t},a_{t}}\leq \frac{1}{2} l\cdot d^\mu_{t}(s_{t},a_{t})\}$ so that with high probability,  the non-implementable version is identical to the practical OPDVR in Algorithm~\ref{alg:OPDVR}. Specifically, when $m'\geq 8 H^2/d_m$ (this is satisfied since for each stage we set $m'$ to be at least $O(H^3/d_m)$), then 
\[
d_m\min_i m^{(i)}=d_m\min_i \frac{m'\log(16KHSA/\delta)}{(u^{(i-1)})^2} \geq d_m\frac{m'\log(16KHSA/\delta)}{H^2} \geq 8\log(16KHSA/\delta),
\]
so by Lemma~\ref{lem:chernoff_multiplicative} and a union bound
\begin{equation}\label{eqn:practical}
\begin{aligned}
&\P\left(\bigcup_{i\in[K]}\bigcup_{\{t,s_t,a_t\;:\;d^\mu_t(s_t,a_t)>0\}}\{n_{s_{t},a_{t}}^{(i)}\leq \frac{1}{2} m^{(i)}\cdot d^\mu_{t}(s_{t},a_{t})\}\cup\{n_{s_{t},a_{t}}^{\prime(i)}\leq \frac{1}{2} l^{(i)}\cdot d^\mu_{t}(s_{t},a_{t})\}\right)\\
\leq &2\P\left(\bigcup_{i\in[K]}\bigcup_{\{t,s_t,a_t\;:\;d^\mu_t(s_t,a_t)>0\}}\{n_{s_{t},a_{t}}^{(i)}\leq \frac{1}{2} m^{(i)}\cdot d^\mu_{t}(s_{t},a_{t})\}\right)\\
\leq &2KHSA\cdot\max_{{\{i,t,s_t,a_t\;:\;d^\mu_t(s_t,a_t)>0\}}} \P\left(n_{s_{t},a_{t}}^{(i)}\leq \frac{1}{2} m^{(i)}\cdot d^\mu_{t}(s_{t},a_{t})\right)\\
\leq &2KHSA \cdot e^{-d_m\min_i m^{(i)}/8}\leq\frac{2KHSA}{16KHSA/\delta}=\delta/8<\delta/4,
\end{aligned}
\end{equation}
and repeat this analysis for both stages, we have with probability $1-\delta/2$, Practical OPDVR is identical to the non-implementable version. 

\subsection{Proof of Theorem~\ref{thm:main_nonstationary}}

\begin{proof}
	The proof consists of two parts. The first part is to use \eqref{eqn:analyze_complexity} to show OPDVR in Algorithm~\ref{alg:OPDVR} outputs $\epsilon$-optimal policy using episode complexity 
	\[
	\frac{2\max[\frac{m'_{H}}{H},m'_{\sqrt{H}}]\log(32HSA\log_2(\sqrt{H}/\epsilon)/\delta)}{\epsilon^2}\log_2(\sqrt{H}/\epsilon)
	\]
	with probability $1-\delta/2$, and the second part is to use \eqref{eqn:practical} to let Practical OPDVR is identical to the non-implementable version with probability $1-\delta/2$. Apply a union bound of these two gives the stated results in Theorem~\ref{thm:main_nonstationary} with probability $1-\delta$.
\end{proof}

\section{Proofs for finite-horizon stationary setting}
\label{sec:proof_stationary}

Again, recall $\tilde{z}_t(s,a)$, $\tilde{\sigma}_{V^{\text{in}}_{t+1}}(s,a)$ \eqref{eqn:off_z_s} and $g_t$ \eqref{eqn:off_g_s} are three quantities deployed in Algorithm~\ref{alg:OPDVRT} that use off-policy data $\mathcal{D}$. We restate their definition as follows:

	\begin{align*}
	\tilde{z}_t(s,a)&=\begin{cases}P^\top(\cdot|s,a) V^{\text{in}}_{t+1}, &{if} \;n_{s,a}\leq \frac{1}{2} m\cdot \sum_{t=1}^H d^\mu_t(s,a),\\ 
	\frac{1}{n_{s,a}}\sum_{i=1}^m \sum_{u=1}^H V^{\text{in}}_{t+1}(s^{(i)}_{u+1})\cdot\mathbf{1}_{[s^{(i)}_u=s,a^{(i)}_u=a]},&{if} \;n_{s,a}> \frac{1}{2} m\cdot \sum_{t=1}^H d^\mu_t(s,a).\end{cases}\\
	\tilde{\sigma}_{V^{\text{in}}_{t+1}}(s,a)&=\begin{cases}{\sigma}_{V^{\text{in}}_{t+1}}(s,a),\qquad\qquad\qquad\qquad\qquad\qquad\,\qquad\quad {if} \;n_{s,a}\leq \frac{1}{2} m\cdot \sum_{t=1}^H d^\mu_t(s,a),\\ \frac{1}{n_{s,a}}\sum_{i=1}^m\sum_{u=1}^H[V^{\text{in}}_{t+1}(s^{(i)}_{u+1})]^2\cdot\mathbf{1}_{[s^{(i)}_u=s,a^{(i)}_u=a]}-\tilde{z}_t^2(s,a), \;\;\quad\qquad {otherwise}.
	\end{cases}
	\end{align*}

\[
	g_{t}(s,a)=\begin{cases}P^\top(\cdot|s,a)[V_{t+1}-V^{\text{in}}_{t+1}]-f(s,a),\hfill {if}\quad n'_{s,a}\leq \frac{1}{2} l\cdot \sum_{t=1}^H d^\mu_{t}(s,a),
	\\\frac{1}{n'_{s,a}}\sum_{j=1}^l\sum_{u=1}^H[V_{t+1}(s'^{(j)}_{u+1})-V_{t}^{\text{in}}(s'^{(j)}_{u+1})]\cdot\mathbf{1}_{[s'^{(j)}_{u},a'^{(j)}_{u}=s,a]}-f(s,a),\hfill o.w.\end{cases}
\]
where 
\begin{equation}\label{eqn:def_nsa}
n_{s,a}=\sum_{i=1}^m\sum_{t=1}^H\mathbf{1}{[s^{(i)}_t=s,a^{(i)}_t=a]},
\end{equation}
and recall $f(s,a)=4u^{\text{in}}\sqrt{{\log(2HSA/\delta)}/{l \sum_{t=1}^H d^\mu_t(s,a)}}$.

\begin{lemma}\label{lem:con_z_s}
	Let $\tilde{z_t}$ be defined as \eqref{eqn:off_z_s} in Algorithm~\ref{alg:OPDVRT}, where $\tilde{z_t}$ is the off-policy estimator of $P^\top(\cdot|s,a)V^{\text{in}}_{t+1}$ using $m$ episodic data. Then with probability $1-\delta$, we have 
	
	\begin{equation}\label{eqn:tilde_z_s}
\left|\tilde{z_t}-\PP_t V^{\text{in}}_{t+1}\right|\leq \left(\sqrt{\frac{16\cdot\sigma_{V^{\text{in}}_{t+1}}\cdot\log(HSA/\delta)}{m\sum_{t=1}^Hd^\mu_t}}+\sqrt{\frac{16 V_{\max}\cdot\log(2HSA/\delta)}{9m\sum_{t=1}^H d^\mu_t}}\cdot\log(HSA/\delta)\right), \quad \forall t\in[H]
	\end{equation}
	here $\tilde{z_t},\PP_t V^{\text{in}}_{t+1},\sigma_{V^{\text{in}}_{t+1}},d^\mu_t\in\mathbb{R}^{S\times A}$ are $S\times A$ column vectors and $\sqrt{\cdot}$ is elementwise operation.
\end{lemma}

\begin{proof}
	Consdier fixed $s,a$. Let $E_{s,a} :=\{n_{s,a}\geq \frac{1}{2}m\cdot \sum_{t=1}^H d^\mu_t(s,a)\}$, then by definition, 
	{\small
	\[
	\tilde{z_t}(s,a)-P^\top(\cdot|s,a)V^{\text{in}}_{t+1} = \left(\frac{1}{n_{s,a}}\sum_{i=1}^m\sum_{u=1}^H V^{\text{in}}_{t+1}(s^{(i)}_{u+1})\cdot\mathbf{1}[s^{(i)}_u=s,a^{(i)}_u=a]-P^\top(\cdot|s,a)V^{\text{in}}_{t+1}\right)\cdot\mathbf{1}(E_{s,a}).
	\]
} First note by \eqref{eqn:def_nsa}
\[
\E[n_{s,a}]=\sum_{i=1}^m\sum_{t=1}^H\E\left[\mathbf{1}[s^{(i)}_t=s,a^{(i)}_t=a]\right]=\sum_{i=1}^m\sum_{t=1}^H d^\mu_t(s,a)=m\sum_{t=1}^Hd^\mu_t(s,a).
\]
Next we conditional on $n_{s,a}$. Define $\mathcal{F}_k:=\{s_u^{(i)},a_u^{(i)}\}_{i\in[m]}^{u\in[k]}$ is an increasing filtration and denote
\[
X:=\sum_{i=1}^m\sum_{u=1}^H \left(V^{\text{in}}_{t+1}(s^{(i)}_{u+1})-P^\top(\cdot|s,a)V^{\text{in}}_{t+1}\right)\cdot\mathbf{1}[s^{(i)}_u=s,a^{(i)}_u=a],
\]
then by tower property $\E[X|\mathcal{F}(Y)]=\E[\E[X|\mathcal{F}(Y,Z)]|\mathcal{F}(Y)]$ (the fourth equal sign in below)
\begin{align*}
X_k:=&\E[X|\mathcal{F}_k]=\E\left[\sum_{i=1}^m\sum_{u=1}^H \left(V^{\text{in}}_{t+1}(s^{(i)}_{u+1})-P^\top(\cdot|s,a)V^{\text{in}}_{t+1}\right)\cdot\mathbf{1}[s^{(i)}_u=s,a^{(i)}_u=a]\middle|\mathcal{F}_k\right]\\
=&\sum_{i=1}^m\E\left[\sum_{u=1}^H \left(V^{\text{in}}_{t+1}(s^{(i)}_{u+1})-P^\top(\cdot|s,a)V^{\text{in}}_{t+1}\right)\cdot\mathbf{1}[s^{(i)}_u=s,a^{(i)}_u=a]\middle|\mathcal{F}_k\right]\\
=&\sum_{i=1}^m\E\left[\sum_{u=k}^H \left(V^{\text{in}}_{t+1}(s^{(i)}_{u+1})-P^\top(\cdot|s,a)V^{\text{in}}_{t+1}\right)\cdot\mathbf{1}[s^{(i)}_u=s,a^{(i)}_u=a]\middle|\mathcal{F}_k\right]\\
+&\sum_{i=1}^m\sum_{u=1}^{k-1} \left(V^{\text{in}}_{t+1}(s^{(i)}_{u+1})-P^\top(\cdot|s,a)V^{\text{in}}_{t+1}\right)\cdot\mathbf{1}[s^{(i)}_u=s,a^{(i)}_u=a]\\
=&\sum_{i=1}^m\sum_{u=k}^H\E\left[\E\big[ \left(V^{\text{in}}_{t+1}(s^{(i)}_{u+1})-P^\top(\cdot|s,a)V^{\text{in}}_{t+1}\right)\cdot\mathbf{1}[s^{(i)}_u=s,a^{(i)}_u=a]\big|\mathcal{F}_u\big]\middle|\mathcal{F}_k\right]\\
+&\sum_{i=1}^m\sum_{u=1}^{k-1} \left(V^{\text{in}}_{t+1}(s^{(i)}_{u+1})-P^\top(\cdot|s,a)V^{\text{in}}_{t+1}\right)\cdot\mathbf{1}[s^{(i)}_u=s,a^{(i)}_u=a]\\
=&\sum_{i=1}^m\sum_{u=k}^H\E\left[\mathbf{1}[s^{(i)}_u=s,a^{(i)}_u=a]\E\big[ \left(V^{\text{in}}_{t+1}(s^{(i)}_{u+1})-P^\top(\cdot|s,a)V^{\text{in}}_{t+1}\right)\big|s^{(i)}_u,a^{(i)}_u\big]\middle|\mathcal{F}_u\right]\\
+&\sum_{i=1}^m\sum_{u=1}^{k-1} \left(V^{\text{in}}_{t+1}(s^{(i)}_{u+1})-P^\top(\cdot|s,a)V^{\text{in}}_{t+1}\right)\cdot\mathbf{1}[s^{(i)}_u=s,a^{(i)}_u=a]\\
\end{align*}
Note if $\mathbf{1}[s^{(i)}_u=s,a^{(i)}_u=a]=1$, then
\begin{align*}
&\mathbf{1}[s^{(i)}_u=s,a^{(i)}_u=a]\E\big[ \left(V^{\text{in}}_{t+1}(s^{(i)}_{u+1})-P^\top(\cdot|s,a)V^{\text{in}}_{t+1}\right)\big|s^{(i)}_u,a^{(i)}_u\big]\\
=&1\cdot \E\big[ \left(V^{\text{in}}_{t+1}(s^{(i)}_{u+1})-P^\top(\cdot|s,a)V^{\text{in}}_{t+1}\right)\big|s^{(i)}_u=s,a^{(i)}_u=a\big]\\
=&\E\big[ V^{\text{in}}_{t+1}(s^{(i)}_{u+1})\big|s^{(i)}_u=s,a^{(i)}_u=a\big]-P^\top(\cdot|s,a)V^{\text{in}}_{t+1}\\
=&P^\top(\cdot|s,a)V^{\text{in}}_{t+1}-P^\top(\cdot|s,a)V^{\text{in}}_{t+1}=0
\end{align*}
if $\mathbf{1}[s^{(i)}_u=s,a^{(i)}_u=a]=0$, then still 
\[
\mathbf{1}[s^{(i)}_u=s,a^{(i)}_u=a]\E\big[ \left(V^{\text{in}}_{t+1}(s^{(i)}_{u+1})-P^\top(\cdot|s,a)V^{\text{in}}_{t+1}\right)\big|s^{(i)}_u,a^{(i)}_u\big]=0
\]
So plug back to obtain
\[
X_k=\sum_{i=1}^m\sum_{u=1}^{k-1} \left(V^{\text{in}}_{t+1}(s^{(i)}_{u+1})-P^\top(\cdot|s,a)V^{\text{in}}_{t+1}\right)\cdot\mathbf{1}[s^{(i)}_u=s,a^{(i)}_u=a].
\]
is a martingale.

First of all by Hoeffding's inequality, we have the martingale difference satisfies with probability $1-\delta/2$,
\begin{equation}\label{eqn:mart_diff}
\begin{aligned}
|X_{k+1}-X_{k}|=&\left|\sum_{i=1}^m \left(V^{\text{in}}_{t+1}(s^{(i)}_{k+1})-P^\top(\cdot|s,a)V^{\text{in}}_{t+1}\right)\cdot\mathbf{1}[s^{(i)}_k=s,a^{(i)}_k=a]\right|\\
=& \left|\sum_{i=1}^{n_{k,s,a}} \left(V^{\text{in}}_{t+1}(s^{(i)}_{k+1}|s,a)-P^\top(\cdot|s,a)V^{\text{in}}_{t+1}\right)\right|\\
\leq&\sqrt{2n_{k,s,a}\cdot V_{\max}\log(2/\delta)}\leq\sqrt{2n_{s,a}\cdot V_{\max}\log(2/\delta)}
\end{aligned}
\end{equation}
where we use shorthand notation $V^{\text{in}}_{t+1}(s^{(i)}_{k+1}|s,a)$ to denote the value of $V^{\text{in}}_{k+1}(s^{(i)}_{k+1})$ given $s^{(i)}_k=s$ and $a^{(i)}_k=a$ and $n_{k,s,a}=\sum_{i=1}^m\mathbf{1}[s^{(i)}_k=s,a^{(i)}_k=a]\leq n_{s,a}$.

Second, 
\begin{align*}
	\Var\left[X_{k+1}\middle|\mathcal{F}_k\right]&=\Var\left[\sum_{i=1}^m \left(V^{\text{in}}_{t+1}(s^{(i)}_{k+1})-P^\top(\cdot|s,a)V^{\text{in}}_{t+1}\right)\cdot\mathbf{1}[s^{(i)}_k=s,a^{(i)}_k=a]\middle|\mathcal{F}_k\right]\\
	&=\sum_{i=1}^m\Var\left[ \left(V^{\text{in}}_{t+1}(s^{(i)}_{k+1})-P^\top(\cdot|s,a)V^{\text{in}}_{t+1}\right)\cdot\mathbf{1}[s^{(i)}_k=s,a^{(i)}_k=a]\middle|\mathcal{F}_k\right]\\
	&=\sum_{i=1}^m\Var\left[ \left(V^{\text{in}}_{t+1}(s^{(i)}_{k+1})-P^\top(\cdot|s,a)V^{\text{in}}_{t+1}\right)\cdot\mathbf{1}[s^{(i)}_k=s,a^{(i)}_k=a]\middle|s^{(i)}_k,a^{(i)}_k\right]\\
	&=\sum_{i=1}^m\mathbf{1}[s^{(i)}_k=s,a^{(i)}_k=a]\Var\left[ V^{\text{in}}_{t+1}(s^{(i)}_{k+1})\middle|s^{(i)}_k,a^{(i)}_k\right]\\
	&=\sum_{i=1}^m\mathbf{1}[s^{(i)}_k=s,a^{(i)}_k=a]\Var\left[ V^{\text{in}}_{t+1}(s^{(i)}_{k+1})\middle|s^{(i)}_k=s,a^{(i)}_k=a\right]\\
	&=\sum_{i=1}^m\mathbf{1}[s^{(i)}_k=s,a^{(i)}_k=a]\cdot \sigma_{V^{\text{in}}_{t+1}}(s,a).
\end{align*}
where the second equal sign uses episodes are independent, the third equal sign uses Markov property, the fourth uses $\mathbf{1}[s^{(i)}_k=s,a^{(i)}_k=a]$ is measurable w.r.t $s^{(i)}_k,a^{(i)}_k$ and $P^\top(\cdot|s,a)V^{\text{in}}_{t+1}$ is constant, the fifth equal sign uses the identity
\[
\mathbf{1}[s^{(i)}_k=s,a^{(i)}_k=a]\Var\left[ V^{\text{in}}_{t+1}(s^{(i)}_{k+1})\middle|s^{(i)}_k,a^{(i)}_k\right]=\mathbf{1}[s^{(i)}_k=s,a^{(i)}_k=a]\Var\left[ V^{\text{in}}_{t+1}(s^{(i)}_{k+1})\middle|s^{(i)}_k=s,a^{(i)}_k=a\right]
\]
and sixth line is true since we have stationary transition, the underlying transition is always $P(\cdot|s,a)$ regardless of time step. This is the key for further reducing the dependence on $H$ and is \textbf{NOT} shared by non-stationary transition setting!

Therefore finally, 
\begin{equation}\label{eqn:var}
\sum_{k=1}^H\Var\left[X_{k+1}\middle|\mathcal{F}_k\right]=\sum_{k=1}^H\sum_{i=1}^m\mathbf{1}[s^{(i)}_k=s,a^{(i)}_k=a]\cdot \sigma_{V^{\text{in}}_{t+1}}(s,a)=n_{s,a}\cdot\sigma_{V^{\text{in}}_{t+1}}(s,a).
\end{equation}
Recall that $s,a$ is fixed and we conditional on $n_{s,a}$. Also note by tower property $\E[X]=0$. Therefore by \eqref{eqn:mart_diff}, \eqref{eqn:var} Freedman's inequality  (Lemma~\ref{lem:freedman}) with probability\footnote{To be mathematically rigorous, the difference bound is not with probability $1$ but in the high probability sense. Therefore essentially we are using a weaker version of freedman's inequality that with high probability bounded difference, \emph{e.g.} see \cite{chung2006concentration} Theorem~34,37. We do not present our result by explicitly writing in that way in order to prevent over-technicality and make the readers easier to understand. } $1-\delta$
\begin{align*}
|X|=&\left|\sum_{i=1}^m\sum_{u=1}^H \left(V^{\text{in}}_{t+1}(s^{(i)}_{u+1})-P^\top(\cdot|s,a)V^{\text{in}}_{t+1}\right)\cdot\mathbf{1}[s^{(i)}_u=s,a^{(i)}_u=a]\right|\\
\leq &\sqrt{{8n_{s,a}\cdot\sigma_{V^{\text{in}}_{t+1}}(s,a)\cdot\log(1/\delta)}}+\frac{2\sqrt{2n_{s,a}\cdot V_{\max}\log(2/\delta)}}{3}\cdot\log(1/\delta).
\end{align*}

which means with probability at least $1-\delta$
	\begin{align*}
	&\left|\tilde{z_t}(s,a)-P^\top(\cdot|s,a)V^{\text{in}}_{t+1}\right| \\
	= &\left|\frac{X}{n_{s,a}}\right|\cdot\mathbf{1}(E_{s,a})\\
	\leq & \left(\frac{\sqrt{{8n_{s,a}\cdot\sigma_{V^{\text{in}}_{t+1}}(s,a)\cdot\log(1/\delta)}}+\frac{2\sqrt{2n_{s,a}\cdot V_{\max}\log(2/\delta)}}{3}\cdot\log(1/\delta)}{n_{s,a}}\right)\cdot\mathbf{1}(E_{s,a})\\
	=& \left(\sqrt{\frac{8\cdot\sigma_{V^{\text{in}}_{t+1}}(s,a)\cdot\log(1/\delta)}{n_{s,a}}}+\sqrt{\frac{8\cdot V_{\max}}{9n_{s,a}}\cdot\log(2/\delta)}\cdot\log(1/\delta)\right)\cdot\mathbf{1}(E_{s,a})\\
	\leq& \left(\sqrt{\frac{16\cdot\sigma_{V^{\text{in}}_{t+1}}(s,a)\cdot\log(1/\delta)}{m\sum_{t=1}^Hd^\mu_t(s,a)}}+\sqrt{\frac{16 V_{\max}\cdot\log(2/\delta)}{9m\sum_{t=1}^H d^\mu_t(s,a)}}\cdot\log(1/\delta)\right)\\
	\end{align*}

	Now we get rid of the conditional on $n_{s_t,a_t}$. Denote $$A=\left\{\left|\tilde{z_t}(s,a)-P^\top(\cdot|s,a)V^{\text{in}}_{t+1}\right|\leq \left(\sqrt{\frac{16\cdot\sigma_{V^{\text{in}}_{t+1}}(s,a)\cdot\log(1/\delta)}{m\sum_{t=1}^Hd^\mu_t(s,a)}}+\sqrt{\frac{16 V_{\max}\cdot\log(2/\delta)}{9m\sum_{t=1}^H d^\mu_t(s,a)}}\cdot\log(1/\delta)\right)\right\},$$ then equivalently we can rewrite above result as $\P(A|n_{s,a})\geq 1-\delta$. Note this is the same as $\E[\mathbf{1}(A)|n_{s,a}]\geq 1-\delta$, therefore by law of total expectation we have
	\[
	\P(A)=\E[\mathbf{1}(A)]=\E[\E[\mathbf{1}(A)|n_{s,a}]]\geq \E[1-\delta]=1-\delta,
	\]
	Finally, apply the union bound over all $t,s,a$, we obtain
	\[
	\left|\tilde{z_t}-\PP_t V^{\text{in}}_{t+1}\right|\leq \left(\sqrt{\frac{16\cdot\sigma_{V^{\text{in}}_{t+1}}\cdot\log(HSA/\delta)}{m\sum_{t=1}^Hd^\mu_t}}+\sqrt{\frac{16 V_{\max}\cdot\log(2HSA/\delta)}{9m\sum_{t=1}^H d^\mu_t}}\cdot\log(HSA/\delta)\right),
	\]
	where the inequality is  element-wise and this is \eqref{eqn:tilde_z}.
\end{proof}

\begin{lemma}\label{lem:con_sig_s}
	Let $\tilde{\sigma}_{V^{\text{in}}_{t+1}}$ be defined as \eqref{eqn:off_z_s} in Algorithm~\ref{alg:OPDVRT}, the off-policy estimator of $\sigma_{V^{\text{in}}_{t+1}}(s,a)$ using $m$ episodic data. Then with probability $1-\delta$, we have 
	
	\begin{equation}\label{eqn:tilde_sig_s}
	\left|\tilde{\sigma}_{V^{\text{in}}_{t+1}}-\sigma_{V_{t+1}^{\text{in}}} \right|\leq 6 V_{\max}^2\sqrt{\frac{\log(4HSA/\delta)}{m\cdot\sum_{t=1}^H d^\mu_t}}+\frac{4V_{\max}^2\log(4HSA/\delta)}{m\cdot \sum_{t=1}^H d^\mu_t},\quad \forall t=1,...,H.
	\end{equation}
	
\end{lemma}

\begin{proof}
	From the definition we have for fixed $(s,a)$
	\begin{align*}
	&\tilde{\sigma}_{V^{\text{in}}_{t+1}}(s,a)-\sigma_{V_{t+1}^{\text{in}}}(s,a)\\ 
	=& \left(\frac{1}{n_{s,a}}\sum_{i=1}^{m}\sum_{u=1}^H\left[ V^{\text{in}}_{t+1}(s^{(i)}_{u+1})^2-P^\top(\cdot|s,a)(V^{\text{in}}_{t+1})^2\right]\mathbf{1}[s^{(i)}_u=s,a^{(i)}=a]\right)\mathbf{1}(E_{s,a})\\
	+&\left(\left[\frac{1}{n_{s,a}}\sum_{i=1}^{m}\sum_{u=1}^H V^{\text{in}}_{t+1}(s^{(i)}_{u+1})\mathbf{1}[s^{(i)}_u=s,a^{(i)}=a]\right]^2-\left[P^\top(\cdot|s,a)V^{\text{in}}_{t+1}\right]^2\right)\mathbf{1}(E_{s,a})\\
	\end{align*}
	Now we conditional on $n_{s, a}$. The key point is we can regroup $m$ episodic data into $mH$ data pieces, in order. (This is valid since within each episode data is generated by time and between different episodes are independent, so we can concatenate one episode after another and end up with $mH$ pieces that comes in sequentially.) This key reformulation allows us to apply Azuma Hoeffding's inequality and obtain with probability $1-\delta/2$,
	\begin{equation}\label{eqn:first_diff_s}
	\begin{aligned}
	&\left(\frac{1}{n_{s,a}}\sum_{i=1}^{m}\sum_{u=1}^H\left[V^{\text{in}}_{t+1}(s^{(i)}_{u+1})^2-P^\top(\cdot|s,a)(V^{\text{in}}_{t+1})^2\right]\mathbf{1}[s^{(i)}_u=s,a^{(i)}_u=a]\right)\mathbf{1}(E_{s,a})\\
	=&\left(\frac{1}{n_{s,a}}\sum_{u'=1}^{n_{s,a}}\left[V^{\text{in}}_{t+1}(s^{(i)}_{u'+1})^2-P^\top(\cdot|s,a)(V^{\text{in}}_{t+1})^2\right]\mathbf{1}[s^{(i)}_{u'}=s,a^{(i)}_{u'}=a]\right)\mathbf{1}(E_{s,a})\\
	\leq& V_{\max}^2\sqrt{\frac{2\log(4/\delta)}{n_{s,a}}}\cdot\mathbf{1}(E_{s,a})\leq 2V_{\max}^2\sqrt{\frac{\log(4/\delta)}{m\cdot \sum_{t=1}^H d^\mu_t(s,a)}},
	\end{aligned}
	\end{equation}
	where the first equal sign comes from the reformulation trick and the first inequality is by $X_k:=\sum_{u'=1}^{k}\left[V^{\text{in}}_{t+1}(s^{(i)}_{u'+1})^2-P^\top(\cdot|s_t,a_t)(V^{\text{in}}_{t+1})^2\right]\mathbf{1}[s^{(i)}_{u'}=s,a^{(i)}_{u'}=a]$ is martingale. Similarly with probability $1-\delta/2$,
	\begin{equation}\label{eqn:sig_diff_s}
	\left(\frac{1}{n_{s,a}}\sum_{u'=1}^{n_{s_t,a_t}}V^{\text{in}}_{t+1}(s^{(i)}_{u'+1}|s,a)-P^\top(\cdot|s,a)V^{\text{in}}_{t+1}\right)\mathbf{1}(E_{s,a})\leq 2V_{\max}\sqrt{\frac{\log(4/\delta)}{m\cdot\sum_{t=1}^H d^\mu_t(s,a)}}.
	\end{equation}
	Note for $a,b,c>0$, if $|a-b|\leq c$, then $|a^2-b^2|=|a-b|\cdot|a+b|\leq |a-b|\cdot(|a|+|b|)\leq|a-b|\cdot(2|b|+c)\leq c\cdot(2|b|+c)=2bc+c^2 $, therefore by \eqref{eqn:sig_diff_s} we have 
	\begin{equation}\label{eqn:second_diff_s}
	\begin{aligned}
	&\left(\left[\frac{1}{n_{s,a}}\sum_{u'=1}^{n_{s,a}}V^{\text{in}}_{t+1}(s^{(i)}_{u'+1}|s,a)\right]^2-\left[P^\top(\cdot|s,a)V^{\text{in}}_{t+1}\right]^2\right)\mathbf{1}(E_{s,a})\\
	\leq &4P^\top(\cdot|s,a)V^{\text{in}}_{t+1}\cdot V_{\max}\sqrt{\frac{\log(4/\delta)}{m\cdot \sum_{t=1}^Hd^\mu_t(s,a)}}+\frac{4V_{\max}^2\log(4/\delta)}{m\cdot \sum_{t=1}^H d^\mu_t(s,a)}\\
	\leq &4 V_{\max}^2\sqrt{\frac{\log(4/\delta)}{m\cdot \sum_{t=1}^H d^\mu_t(s,a)}}+\frac{4V_{\max}^2\log(4/\delta)}{m\cdot \sum_{t=1}^H d^\mu_t(s,a)}
	\end{aligned}
	\end{equation}
	where the last inequality comes from $|P^\top(\cdot|s,a)V^{\text{in}}_{t+1}|\leq ||P(\cdot|s,a)||_1 ||V^{\text{in}}_{t+1}||_\infty\leq V_{\max}$. Combining \eqref{eqn:first_diff_s}, \eqref{eqn:second_diff_s} and a union bound, we have with probability $1-\delta$, 
	\[
	\left|\tilde{\sigma}_{V^{\text{in}}_{t+1}}(s,a)-\sigma_{V_{t+1}^{\text{in}}}(s,a) \right|\leq 6 V_{\max}^2\sqrt{\frac{\log(4/\delta)}{m\cdot \sum_{t=1}^H d^\mu_t(s,a)}}+\frac{4V_{\max}^2\log(4/\delta)}{m\cdot \sum_{t=1}^H d^\mu_t(s,a)},
	\]
	apply again the union bound over $t,s,a$ gives the desired result.
	
\end{proof}

\begin{lemma}\label{lem:con_g_s}
	Fix time $t\in[H]$. Let $g_t$ be the estimator in \eqref{eqn:off_g} in Algorithm~\ref{alg:OPDVRT}. Then if $||V_{t+1}-V^{\text{in}}_{t+1}||_\infty\leq 2u^{\text{in}}$, then with probability $1-\delta/H$, 
	\[
	\mathbf{0}\leq \PP[V_{t+1}-V^{\text{in}}_{t+1}]-g_t\leq 8u^{\text{in}}\sqrt{\frac{\log(2HSA/\delta)}{l \sum_{t=1}^H d^\mu_t}}
	\]
\end{lemma}

\begin{proof}
	Recall $g_t,d^\mu_t$ are vectors. By definition of $g_t(s,a)$, use similar regrouping trick and apply Azuma Hoeffding's inequality we obtain with probability $1-\delta/H$
	\begin{align*}
	&g_t(s,a)+f(s,a)-P^\top(\cdot|s,a)[V_{t+1}-V^{\text{in}}_{t+1}]\\
	=&\left(\frac{1}{n'_{s,a}}\sum_{u'=1}^{n'_{s,a}}\left[V_{t+1}(s'^{(j)}_{u'+1}|s,a)-V_{t+1}^{\text{in}}(s'^{(j)}_{u'+1}|s,a)\right]-P^\top(\cdot|s,a)[V_{t+1}-V^{\text{in}}_{t+1}]\right)\cdot\mathbf{1}(E_t)\\
	\leq & \left(||V_{t+1}-V_{t+1}^{\text{in}}||_\infty\sqrt{\frac{2\log(2H/\delta)}{n'_{s,a}}}\right)\cdot\mathbf{1}(E_t)\\
	\leq & ||V_{t+1}-V_{t+1}^{\text{in}}||_\infty\sqrt{\frac{4\log(2H/\delta)}{l\cdot \sum_{t=1}^Hd^\mu_t(s,a)}}
	\end{align*}
	Now use assumption $||V_{t+1}-V_{t+1}^{\text{in}}||_\infty\leq 2u^{\text{in}}$ and a union bound over $s,a$, we have with probability $1-\delta/H$,
	\begin{equation}
	\left|g_t+f-\PP[V_{t+1}-V^{(0)}_{t+1}]\right|\leq 4u^{\text{in}}\sqrt{\frac{\log(2HSA/\delta)}{l \sum_{t=1}^H d^\mu_t}}
	\end{equation}
	use $f=4u^{\text{in}}\sqrt{{\log(2HSA/\delta)}/{l \sum_{t=1}^H d^\mu_t}}$, we obtain the stated result.
\end{proof}

\paragraph{Proof of Theorem~\ref{thm:main_stationary}} Note that Lemma~\ref{lem:con_z_s},\ref{lem:con_sig_s},\ref{lem:con_g_s} updates Lemma~\ref{lem:con_z},\ref{lem:con_sig},\ref{lem:con_g} by replacing $d^\mu_t$ with $\sum_{t=1}^H d^\mu_t$ and keeping the rest the same except the second order term $\sqrt{\frac{16 V_{\max}\cdot\log(2HSA/\delta)}{9m\sum_{t=1}^H d^\mu_t}}\cdot\log(HSA/\delta)$ in Lemma~\ref{lem:con_z_s} is different from Lemma~\ref{lem:con_z_s}. However, this is still lower order term since it is of order $\widetilde{O}(\sqrt{\frac{H}{m\sum_{t=1}^H d^\mu_t}})$. To avoid redundant reasoning, by following the identical logic as Section~\ref{sec:iterative_non} we have a similar expression of \eqref{eqn:lower_zt} as follows:

\begin{equation}\label{eqn:lower_zt_s}
\begin{aligned}
&z_t\geq \PP V^{\tin}_{t+1}-\sqrt{\frac{16\cdot{\sigma}_{V^{\star}_{t+1}}\cdot\log(4HSA/\delta)}{m\cdot\sum_{t=1}^H d_t^\mu}}-\sqrt{\frac{16\cdot\log(4HSA/\delta)}{m\cdot \sum_{t=1}^H d_t^\mu}}\cdot u^{\tin}\\
&- V_{\max}\left[8\sqrt{6}\cdot\left(\frac{\log(16HSA/\delta)}{m\cdot \sum_{t=1}^H d^\mu_t}\right)^{3/4}+\frac{56\log(16HSA/\delta)}{3m\cdot \sum_{t=1}^H d^\mu_t}\right]-c\sqrt{\frac{V_{\max}\cdot\log(16HSA/\delta)}{m\cdot \sum_{t=1}^H d^\mu_t}}\log(HSA/\delta).
\end{aligned}
\end{equation}
where the last term is additional. However, note when $u^{\tin}\leq\sqrt{H}$,  then 
\[
\sqrt{\frac{16\cdot\log(4HSA/\delta)}{m\cdot \sum_{t=1}^H d_t^\mu}}\cdot u^{\tin}\leq\tilde{O}\left(\sqrt{\frac{H}{m\sum_{t=1}^H d^\mu_t}}\right),\qquad c\sqrt{\frac{V_{\max}\cdot\log(16HSA/\delta)}{m\cdot \sum_{t=1}^H d^\mu_t}}\leq \tilde{O}\left(\sqrt{\frac{H}{m\sum_{t=1}^H d^\mu_t}}\right)
\]
so the last term $c\sqrt{\frac{V_{\max}\cdot\log(16HSA/\delta)}{m\cdot \sum_{t=1}^H d^\mu_t}}$ can be assimilated by previous one. If $u^{\tin}>\sqrt{H}$, it is of even lower order. Therefore following the same reasoning we can complete the proof for Algorithm~\ref{alg:OPDVRT}.

From non-implementable version to the practical version, we need to bound the event of $\{n_{s,a}\leq \frac{1}{2} m\cdot \sum_{t=1}^H d^\mu_t(s,a) \}$, where $n_{s,a}=\sum_{i=1}^m\sum_{t=1}^H\mathbf{1}{[s^{(i)}_t=s,a^{(i)}_t=a]}$. In this case, $n_{s,a}$ is no longer binomial random variable so Lemma~\ref{lem:chernoff_multiplicative} cannot be applied. However, the trick we use for resolving this issue is the following decomposition
\[
\left\{n_{s,a}\leq \frac{1}{2} m\cdot \sum_{t=1}^H d^\mu_t(s,a)\right\}\subset \bigcup_{t=1}^H\left\{n_{t,s,a}\leq\frac{1}{2}md^\mu_t(s,a)\right\},
\]
where $n_{s,a}=\sum_{t=1}^H n_{t,s,a}$ and $n_{t,s,a}=\sum_{i=1}^m\mathbf{1}{[s^{(i)}_t=s,a^{(i)}_t=a]}$ are binomial random variables. Lemma~\ref{lem:chernoff_multiplicative} can then be used together with union bounds to finish the proof.

\section{Proofs for infinite-horizon discounted setting}
\label{sec:proof_i}

First recall data $\mathcal{D}=\{s^{(i)},a^{(i)},r^{(i)},s^{\prime(i)}\}_{i\in[n]}$ are i.i.d off-policy pieces with $(s^{(i)},a^{(i)})\sim d^\mu$ and $s'^{(i)}\sim P(\cdot|s^{(i)},a^{(i)})$. Moreover, $d^\mu$ is defined as:
\[
d^\mu(s)=(1-\gamma)\sum_{t=0}^\infty\gamma^t\P[s_t=s|s_0\sim d_0,\mu],\quad d^\mu(s,a)=d^\mu(s)\mu(a|s).
\]

The corresponding \emph{off-policy estimators} in Algorithm~\ref{alg:OPDVRT_in} are defined as:
{
	\begin{equation}\label{eqn:off_z_app}
	\begin{aligned}
	\tilde{z}(s,a)&=\begin{cases}P^\top(\cdot|s,a) V^{\text{in}}, &{if} \;n_{s,a}\leq \frac{1}{2} m\cdot d^\mu(s,a),\\ 
	\frac{1}{n_{s,a}}\sum_{i=1}^m V^{\text{in}}(s^{\prime(i)})\cdot\mathbf{1}_{[s^{(i)}=s,a^{(i)}=a]},&{if} \;n_{s,a}> \frac{1}{2} m\cdot d^\mu(s,a).\end{cases}\\
	\tilde{\sigma}_{V^{\text{in}}}(s,a)&=\begin{cases}{\sigma}_{V^{\text{in}}}(s,a),\qquad\qquad\qquad\qquad\qquad\,\qquad\quad {if} \;n_{s,a}\leq \frac{1}{2} m\cdot d^\mu(s,a),\\ \frac{1}{n_{s,a}}\sum_{i=1}^m[V^{\text{in}}(s^{\prime(i)})]^2\cdot\mathbf{1}_{[s^{(i)}=s,a^{(i)}=a]}-\tilde{z}^2(s,a), \;\;{otherwise}.
	\end{cases}
	\end{aligned}
	\end{equation}
}where $n_{s,a}:=\sum_{i=1}^n \mathbf{1}[s^{(i)}=s,a^{(i)}=a]$ is the number of samples start at $(s,a)$. Similarly, $P^\top(\cdot|s,a)[V-V^{\text{in}}]$ is later updated using different $l$ episodes ($n'_{s,a}$ is the number count from $l$ episodes):
{\begin{equation}\label{eqn:off_g_app}
	g^{(i)}(s,a)=\begin{cases}P^\top(\cdot|s,a)[V^{(i)}-V^{\text{in}}]-f(s,a),\hfill {if}\quad n'_{s,a}\leq \frac{1}{2} l\cdot d^\mu(s,a),
	\\\frac{1}{n'_{s,a}}\sum_{j=1}^l[V^{(i)}(s'^{(j)})-V^{\text{in}}(s'^{(j)})]\cdot\mathbf{1}_{[s'^{(j)},a'^{(j)}=s,a]}-f(s,a),\hfill o.w.\end{cases}
	\end{equation}}
where $f=4u^{\text{in}}\sqrt{{\log(2RSA/\delta)}/{l d^\mu}}$ and $R=\ln (4/u^{\tin}(1-\gamma))$.


\begin{algorithm}[tb]
	\caption{OPVRT: A Prototypical Off-Policy Variance Reduction Template ($\infty$-horizon)}
	\label{alg:OPDVRT_in}
	\small{
		\begin{algorithmic}[1]
			\STATE {\bfseries Functional input:}  Integer valued function $\mathbf{m}:  \R_+ \rightarrow \mathbb{N}$.   Off-policy estimator $\mathbf{z}_t,\mathbf{g}_t$ in function forms that provides lower confidence bounds (LCB) of the two terms in the bootstrapped value function \eqref{eq:VR_decomposition}.
			
			\STATE {\bfseries Static input:} Initial value function $V^{(0)}$ and $\pi^{(0)}$  (which satisfy $V^{(0)}\leq \mathcal{T}_{\pi^{(0)}}V^{(0)}$).
			A scalar $u^{(0)}$ satisfies $u^{(0)}\geq ||V^\star-V^{(0)}||_\infty$.   Outer loop iterations $K$. 
			Offline dataset $\mathcal D = \{s^{(i)},a^{(i)},r^{(i)},s'^{(i)}\}_{i = 1}^{n}$ from the behavior policy $\mu$ as a data-stream where
			$n \geq \sum_{i=1}^K  (1+R) \cdot \mathbf{m}(u^{(0)} \cdot 2^{-(i-1)}). $
			
			\STATE 			------------------\textsc{Inner loop} ---------------------\\
			\FUNCTION{\textsc{QVI-VR-inf} ($\cD_1,[\cD_2^{(i)}]_{i=1}^K$, $V_t^{\text{in}}, \pi^{\text{in}}, \mathbf{z}_t,\mathbf{g}_t, u^\text{in}$)}
			\STATE  \mypink{$\diamond$ Computing reference with $\cD_1$:}
			\STATE Initialize $Q^{(0)}\leftarrow \mathbf{0}\in\R^{\mathcal{S}\times\mathcal{A}}$ and $V^{(0)} = V_{\text{in}}$.
			\FOR{each pair $(s,a)\in\mathcal{S}\times\mathcal{A}$}
			\STATE \mypink{$\diamond$ Compute an LCB of $P^\top(\cdot|s,a) V^{in}$:}
			\STATE $z\leftarrow  \mathbf{z}(\cD_1, V^{\text{in}}, u^\text{in} )$
			\ENDFOR
			\STATE \mypink{$\diamond$ Value Iterations with $\cD_2$:}
			\FOR{$i=1,...,R$}
			\STATE \mypink{$\diamond$ Update value function:}  $V^{(i)}=\max(V_{Q^{(i-1)}},V^{\text{in}})$, 
			\STATE \mypink{$\diamond$ Update policy according to value function:} 
			\STATE $\forall s$, if $V^{(i)}(s)=V^{(i-1)}(s)$ set $\pi(s)=\pi^{\text{in}}(s)$; else set $\pi(s)=\pi_{Q^{(i-1)}}(s)$. 
			\IF{$t\geq 1$}
			\STATE \mypink{$\diamond$  LCB of $P^\top(\cdot|s_{t-1},a_{t-1})[V_{t}-V^{\text{in}}_{t}]$:} 
			\STATE $g^{(i)} \leftarrow  \mathbf{g}(\cD_2^{(i)},V^{(i)}, V^{\text{in}},u^\text{in})$. 
			\STATE \mypink{$\diamond$ Update $Q$ function:}  $Q^{(i)}\leftarrow r+\gamma z+\gamma g^{(i)}$  
			\ENDIF
			\ENDFOR 
			\STATE {\bfseries Return:} $V^{(R)}$ and $\pi^{(R)}$.
			\ENDFUNCTION
			\STATE 			------------------\textsc{outer loop} ---------------------\\
			\FOR{$j=1,...,K$}
			\STATE$m^{(j)} \rightarrow \mathbf{m}(u^{(j-1)})$
			\STATE Get $\cD_1$ and $\cD_2^{(i)}$ for $i=1,...,K$ each with size $m^{(j)}$ from the stream $\cD$.
			\STATE $V^{(j)},\pi^{(j)}\leftarrow$\textsc{QVI-VR-inf}($\cD_1,[\cD_2^{(i)}]_{i=1}^K,V^{(i-1)},\pi^{(i-1)},\mathbf{z},\mathbf{g},u^{(i-1)}$).
			\STATE $u^{(j)}\leftarrow u^{(j-1)}/2$.
			\ENDFOR		
			\STATE {\bfseries Output: $V^{(K)}$, $\pi^{(K)}$ } 
			
		\end{algorithmic}
	}
\end{algorithm}
\normalsize{}
\begin{algorithm}[t]
	\caption{OPDVR: Off-Policy Doubled Variance Reduction ($\infty$-horizon)}
	\label{alg:OPDVR_in}
	\small{
		%
		\begin{algorithmic}[1]
			\INPUT Offline Dataset $\cD$ of size $n$ as a stream.  Target accuracy $\epsilon,\delta$ such that the algorithm does not use up $\cD$.
			\INPUT Estimators $\mathbf{z}, \mathbf{g} $ in function forms, $m'_1, m'_2, K_1, K_2$. 
			\STATE  \mypink{$\diamond$ Stage $1$. coarse learning: a ``warm-up'' procedure } 
			\STATE Set initial values $V^{(0)}:=\mathbf{0}$ and any policy $\pi^{(0)}$. 
			\STATE Set initial $u^{(0)}:=(1-\gamma)^{-1}$.
			\STATE Set $\mathbf{m}(u)=m'_1\log(16(1-\gamma)^{-1}RSA)/u^2$. 
						\STATE Run  Algorithm~\ref{alg:OPDVRT_in} with   $\mathbf{m}, \mathbf{z}, \mathbf{g},  V^{(0)}, \pi^{(0)},u^{(0)}, K_1,\cD$ and return
			$V^{\text{intermediate}},\pi^{\text{intermediate}}$. 
			\STATE \mypink{$\diamond$ Stage $2$. fine learning: reduce error to given accuracy }						
			\STATE Reset initial values $V^{(0)}:=V^{\text{intermediate}}$ and policy $\pi^{(0)}:=\pi^{\text{intermediate}}$. Set $u^{(0)}:=\sqrt{(1-\gamma)^{-1}}$.
			\STATE Reset $\mathbf{m}(u)$ by replacing $m'_1$ with $m'_2$, $K_1$ with $K_2$. 
			\STATE Run  Algorithm~\ref{alg:OPDVRT_in} with   $\mathbf{m}, \mathbf{z}, \mathbf{g},  V_t^{(0)}, \pi^{(0)},u^{(0)}, K_2,\cD$ and return $V_t^{\text{final}},\pi^{\text{final}}$.
			\OUTPUT  $V_t^{\text{final}},\pi^{\text{final}}$
		\end{algorithmic}
	}
	%
	%
\end{algorithm}
\normalsize{}

\begin{lemma}\label{lem:mono_in}
	Suppose $V$ and $\pi$ is any value and policy satisfy $V\leq \mathcal{T}_{\pi}V$. Then it holds $V\leq V^\pi\leq V^\star$.
	
\end{lemma}
\begin{proof}
This is similar to Lemma~\ref{lem:mono} and the key is to use Bellman equation $V^\pi= \mathcal{T}_{\pi}V^\pi$.
\end{proof}


	\begin{lemma}\label{lem:con_z_in}
		Let $\tilde{z}$ be defined as \eqref{eqn:off_z_app} in Algorithm~\ref{alg:OPDVRT_in}, where $\tilde{z}$ is the off-policy estimator of $P^\top(\cdot|s,a)V^{\text{in}}$ using $m$ episodic data. Then with probability $1-\delta$, we have 
		
		\begin{equation}\label{eqn:tilde_z_in}
		\left|\tilde{z}-\PP V^{\text{in}}\right|\leq \sqrt{\frac{4\cdot\sigma_{V^{\text{in}}}\cdot\log(SA/\delta)}{m\cdot d^\mu}}+\frac{4V_{\max}}{3m\cdot d^\mu}\log(SA/\delta).
		\end{equation}
		here $\tilde{z},\PP V^{\text{in}},\sigma_{V^{\text{in}}},d^\mu\in\mathbb{R}^{S\times A}$ are $S\times A$ column vectors and $\sqrt{\cdot}$ is elementwise operation.
	\end{lemma}
	
	\begin{proof}
		First fix $s,a$. Let $E_{s,a} :=\{n_{s,a}\geq \frac{1}{2}m\cdot d^\mu(s,a)\}$, then by definition, 
		\[
		\tilde{z}(s,a)-P^\top(\cdot|s,a)V^{\text{in}}_{t+1} = \left(\frac{1}{n_{s,a}}\sum_{i=1}^m V^{\text{in}}(s^{\prime(i)})\cdot\mathbf{1}[s^{(i)}=s,a^{(i)}=a]-P^\top(\cdot|s,a)V^{\text{in}}\right)\cdot\mathbf{1}(E_{s,a}).
		\]
		Next we conditional on $n_{s,a}$. Then from above expression and Bernstein inequality \ref{lem:bernstein_ineq} we have with probability at least $1-\delta$
		\begin{align*}
		&\left|\tilde{z}(s,a)-P^\top(\cdot|s,a)V^{\text{in}}\right| \\
		= &\left|\frac{1}{n_{s,a}}\sum_{i=1}^{n_{s,a}}V^{\text{in}}(s^{\prime (i)}|s,a)-P^\top(\cdot|s,a)V^{\text{in}}\right|\cdot\mathbf{1}(E_{s,a})\\
		\leq & \left(\sqrt{\frac{2\cdot\sigma_{V^{\text{in}}}(s,a)\cdot\log(1/\delta)}{n_{s,a}}}+\frac{2V_{\max}}{3n_{s,a}}\log(1/\delta)\right)\cdot\mathbf{1}(E_{s,a})\\
		\leq & \sqrt{\frac{4\cdot\sigma_{V^{\text{in}}}(s,a)\cdot\log(1/\delta)}{m\cdot d^\mu(s,a)}}+\frac{4V_{\max}}{3m\cdot d^\mu(s,a)}\log(1/\delta)
		\end{align*}
		where again notation $V^{\text{in}}(s^{\prime(i)}|s,a)$ denotes the value of $V^{\text{in}}(s^{\prime(i)})$ given $s^{(i)}=s$ and $a^{(i)}=a$. The condition $V^\text{in}\leq V_{\max}$ is guaranteed by Lemma~\ref{lem:mono}.
		Now we get rid of the conditional on $n_{s,a}$. Denote $A=\{\tilde{z}(s,a)-P^\top(\cdot|s,a)V^{\text{in}}\leq \sqrt{4\cdot\sigma_{V^{\text{in}}}(s,a)\cdot\log(1/\delta)/m\cdot d^\mu(s,a)}+\frac{4V_{\max}}{3m\cdot d^\mu(s,a)}\log(1/\delta)\}$, then equivalently we can rewrite above result as $\P(A|n_{s,a})\geq 1-\delta$. Note this is the same as $\E[\mathbf{1}(A)|n_{s,a}]\geq 1-\delta$, therefore by law of total expectation we have
		\[
		\P(A)=\E[\mathbf{1}(A)]=\E[\E[\mathbf{1}(A)|n_{s,a}]]\geq \E[1-\delta]=1-\delta,
		\]
		\textit{i.e.} for fixed $(s,a)$ we have with probability at least $1-\delta$, 
		\[
		\left|\tilde{z}(s,a)-P^\top(\cdot|s,a)V^{\text{in}}\right|\leq \sqrt{\frac{4\cdot\sigma_{V^{\text{in}}}(s,a)\cdot\log(1/\delta)}{m\cdot d^\mu(s,a)}}+\frac{4V_{\max}}{3m\cdot d^\mu(s,a)}\log(1/\delta)
		\]
		Apply the union bound over all $s,a$, we obtain
		\[
		\left|\tilde{z}-\PP V^{\text{in}}\right|\leq \sqrt{\frac{4\cdot\sigma_{V^{\text{in}}}\cdot\log(SA/\delta)}{m\cdot d^\mu}}+\frac{4V_{\max}}{3m\cdot d^\mu}\log(SA/\delta),
		\]
		where the inequality is  element-wise and this is \eqref{eqn:tilde_z_in}.
	\end{proof}	
	
	\begin{lemma}\label{lem:con_sig_in}
		Let $\tilde{\sigma}_{V^{\text{in}}}$ be defined as \eqref{eqn:off_z_app} in Algorithm~\ref{alg:OPDVRT_in}, the off-policy estimator of $\sigma_{V^{\text{in}}}(s,a)$ using $m$ episodic data. Then with probability $1-\delta$, we have 
		
		\begin{equation}\label{eqn:tilde_sig_in}
		\left|\tilde{\sigma}_{V^{\text{in}}}-\sigma_{V^{\text{in}}} \right|\leq 6 V_{\max}^2\sqrt{\frac{\log(4SA/\delta)}{m\cdot d^\mu}}+\frac{4V_{\max}^2\log(4SA/\delta)}{m\cdot d^\mu}.
		\end{equation}
		
	\end{lemma}
	
	\begin{proof}
		From the definition we have for fixed $(s,a)$
		\begin{align*}
		\tilde{\sigma}_{V^{\text{in}}}(s,a)-\sigma_{V^{\text{in}}}(s,a) =& \left(\frac{1}{n_{s,a}}\sum_{i=1}^{n_{s,a}}V^{\text{in}}(s^{\prime(i)}|s,a)^2-P^\top(\cdot|s,a)(V^{\text{in}})^2\right)\mathbf{1}(E_{s,a})\\
		+&\left(\left[\frac{1}{n_{s,a}}\sum_{i=1}^{n_{s,a}}V^{\text{in}}(s^{\prime(i)}|s,a)\right]^2-\left[P^\top(\cdot|s,a)V^{\text{in}}\right]^2\right)\mathbf{1}(E_{s,a})\\
		\end{align*}
		By using the same conditional on $n_{s, a}$ as in Lemma~\ref{lem:con_z_in}, applying Hoeffding's inequality and law of total expectation, we obtain with probability $1-\delta/2$,
		\begin{equation}\label{eqn:first_diff_in}
		\begin{aligned}
		&\left(\frac{1}{n_{s,a}}\sum_{i=1}^{n_{s,a}}V^{\text{in}}(s^{\prime(i)}|s,a)^2-P^\top(\cdot|s,a)(V^{\text{in}})^2\right)\mathbf{1}(E_{s,a})\\
		&\leq V_{\max}^2\sqrt{\frac{2\log(4/\delta)}{n_{s,a}}}\cdot\mathbf{1}(E_{s,a})\leq 2V_{\max}^2\sqrt{\frac{\log(4/\delta)}{m\cdot d^\mu(s,a)}},
		\end{aligned}
		\end{equation}
		and similarly with probability $1-\delta/2$,
		\begin{equation}\label{eqn:sig_diff_in}
		\left(\frac{1}{n_{s,a}}\sum_{i=1}^{n_{s,a}}V^{\text{in}}(s^{\prime(i)}|s,a)-P^\top(\cdot|s,a)V^{\text{in}}\right)\mathbf{1}(E_{s,a})\leq 2V_{\max}\sqrt{\frac{\log(4/\delta)}{m\cdot d^\mu(s,a)}}.
		\end{equation}
		Again note for $a,b,c>0$, if $|a-b|\leq c$, then $|a^2-b^2|=|a-b|\cdot|a+b|\leq |a-b|\cdot(|a|+|b|)\leq|a-b|\cdot(2|b|+c)\leq c\cdot(2|b|+c)=2bc+c^2 $, therefore by \eqref{eqn:sig_diff_in} we have 
		\begin{equation}\label{eqn:second_diff_in}
		\begin{aligned}
		&\left(\left[\frac{1}{n_{s,a}}\sum_{i=1}^{n_{s,a}}V^{\text{in}}(s^{(i)}|s,a)\right]^2-\left[P^\top(\cdot|s,a)V^{\text{in}}\right]^2\right)\mathbf{1}(E_{s,a})\\
		\leq &4P^\top(\cdot|s,a)V^{\text{in}}\cdot V_{\max}\sqrt{\frac{\log(4/\delta)}{m\cdot d^\mu(s,a)}}+\frac{4V_{\max}^2\log(4/\delta)}{m\cdot d^\mu(s,a)}\\
		\leq &4 V_{\max}^2\sqrt{\frac{\log(4/\delta)}{m\cdot d^\mu(s,a)}}+\frac{4V_{\max}^2\log(4/\delta)}{m\cdot d^\mu(s,a)}
		\end{aligned}
		\end{equation}
		where the last inequality comes from $|P^\top(\cdot|s,a)V^{\text{in}}|\leq ||P(\cdot|s,a)||_1 ||V^{\text{in}}||_\infty\leq V_{\max}$. Combining \eqref{eqn:first_diff_in}, \eqref{eqn:second_diff_in} and a union bound, we have with probability $1-\delta$, 
		\[
		\left|\tilde{\sigma}_{V^{\text{in}}}(s,a)-\sigma_{V^{\text{in}}}(s,a) \right|\leq 6 V_{\max}^2\sqrt{\frac{\log(4/\delta)}{m\cdot d^\mu(s,a)}}+\frac{4V_{\max}^2\log(4/\delta)}{m\cdot d^\mu(s,a)},
		\]
		apply again the union bound over $s,a$ gives the desired result.
		
	\end{proof}
	
	\begin{lemma}\label{lem:con_g_in}
		Fix $i\in[R]$. Let $g^{(i)}$ be the estimator in \eqref{eqn:off_g_app} in Algorithm~\ref{alg:OPDVRT_in}. Then if $||V^{(i)}-V^{\text{in}}||_\infty\leq 2u^{\text{in}}$, then with probability $1-\delta/R$, 
		\[
		\mathbf{0}\leq \PP[V^{(i)}-V^{\text{in}}]-g^{(i)}\leq 8u^{\text{in}}\sqrt{\frac{\log(2RSA/\delta)}{l d^\mu}}
		\]
	\end{lemma}

	\begin{proof}
		Recall $g^{(i)},d^\mu$ are vectors. By definition of $g^{(i)}(s,a)$, applying Hoeffding's inequality we obtain with probability $1-\delta/R$,
		\begin{align*}
		&g^{(i)}(s,a)+f(s,a)-P^\top(\cdot|s,a)[V^{(i)}-V^{\text{in}}]\\
		=&\left(\frac{1}{n'_{s,a}}\sum_{j=1}^l\left[V^{(i)}(s'^{(j)}|s,a)-V^{\text{in}}(s'^{(j)}|s,a)\right]-P^\top(\cdot|s,a)[V^{(i)}-V^{\text{in}}]\right)\cdot\mathbf{1}(E_{s,a})\\
		\leq & \left(||V^{(i)}-V^{\text{in}}||_\infty\sqrt{\frac{2\log(2R/\delta)}{n'_{s,a}}}\right)\cdot\mathbf{1}(E_{s,a})\\
		\leq & ||V^{(i)}-V^{\text{in}}||_\infty\sqrt{\frac{4\log(2R/\delta)}{l\cdot d^\mu(s,a)}}
		\end{align*}
		Now use assumption $||V^{(i)}-V^{\text{in}}||_\infty\leq 2u^{\text{in}}$ and a union bound over $s,a$, we have with probability $1-\delta/R$,
		\begin{equation}
		\left|g^{(i)}+f-\PP[V^{(i)}-V^{\tin}]\right|\leq 4u^{\text{in}}\sqrt{\frac{\log(2RSA/\delta)}{l d^\mu}}
		\end{equation}
		use $f=4u^{\text{in}}\sqrt{{\log(2RSA/\delta)}/{l d^\mu}}$, we obtain the stated result.
	\end{proof}

	\subsection{Iterative update analysis for infinite horizon discounted setting}\label{sec:iterative_in}
	
	The goal of iterative update is to obtain the recursive relation: $Q^\star-Q^{(i)}\leq \gamma\PP^{\pi^\star}[Q^\star-Q^{(i-1)}]+{\xi}$. 
	
	\begin{lemma}\label{lem:VR_incremental_in}
		Let $Q^\star$ be the optimal $Q$-value satisfying $Q^\star=r+\gamma \PP V^\star$ and $\pi^\star$ is one optimal policy satisfying Assumption~\ref{assu2}. Let $\pi$ and $V_t$ be the \textbf{Return} of inner loop in Algorithm~\ref{alg:OPDVRT_in}. We have with probability $1-\delta$, for all $i\in[R]$,
		\begin{align*}
		V^{\text{in}}\leq V^{(i)}\leq\mathcal{T}_{\pi^{(i)}}V^{(i)}\leq V^\star,\quad Q^{(i)}\leq r+\gamma \PP V^{(i)},\quad \text{and}\quad Q^\star-Q^{(i)}\leq\gamma \PP^{\pi^\star}[Q^\star-Q^{(i-1)}]+{\xi},
		\end{align*} 
		where \begin{align*}
		{\xi}\leq&8u^{\tin}\sqrt{\frac{\log(2RSA/\delta)}{l d^\mu}}+\sqrt{\frac{16\cdot{\sigma}_{V^{\star}}\cdot\log(4SA/\delta)}{m\cdot d^\mu}}+\sqrt{\frac{16\cdot\log(4SA/\delta)}{m\cdot d^\mu}}\cdot u^{\tin}\\
		+& V_{\max}\left[8\sqrt{6}\cdot\left(\frac{\log(16SA/\delta)}{m\cdot d^\mu}\right)^{3/4}+\frac{56\log(16SA/\delta)}{3m\cdot d^\mu}\right].
		\end{align*} Here $\PP^{\pi^\star}\in\R^{S\cdot A\times S\cdot A}$ with $\PP^{\pi^\star}_{(s,a),(s',a')}=d^{\pi^\star}(s',a'|s,a)$.
	\end{lemma}
	
	\begin{proof}
		\textbf{Step1:} For any $a,b\geq0$, we have the basic inequality $\sqrt{a+b}\leq\sqrt{a}+\sqrt{b}$, and apply to Lemma~\ref{lem:con_sig_in} we have with probability $1-\delta/4$,
		\begin{equation}\label{eqn:diff_sig_in}
		\sqrt{\left|\tilde{\sigma}_{V^{\text{in}}}-\sigma_{V^{\text{in}}} \right|}\leq  V_{\max}\cdot\left(\frac{36\log(16SA/\delta)}{m\cdot d^\mu}\right)^{1/4}+2V_{\max}\cdot\sqrt{\frac{\log(16SA/\delta)}{m\cdot d^\mu}}.
		\end{equation}
		Next, similarly for any $a,b\geq 0$, we have $\sqrt{a}\leq\sqrt{|a-b|}+\sqrt{b}$,  conditional on above then apply to Lemma~\ref{lem:con_z_in} (with probability $1-\delta/4$) and we obtain with probability $1-\delta/2$,
		\begin{align*}
		&\left|\tilde{z}-\PP V^{\text{in}}\right|\\
		\leq &\sqrt{\frac{4\cdot\sigma_{V^{\text{in}}}\cdot\log(4SA/\delta)}{m\cdot d^\mu}}+\frac{4V_{\max}}{3m\cdot d^\mu}\log(4SA/\delta)\\
		\leq&\left(\sqrt{\tilde{\sigma}_{V^{\text{in}}}}+\sqrt{\left|\tilde{\sigma}_{V^{\text{in}}}-\sigma_{V^{\text{in}}}\right|}\right)\sqrt{\frac{4\cdot\log(4SA/\delta)}{m\cdot d^\mu}}+\frac{4V_{\max}}{3m\cdot d^\mu}\log(4SA/\delta)\\
		= &\sqrt{\frac{4\cdot\tilde{\sigma}_{V^{\text{in}}}\cdot\log(4SA/\delta)}{m\cdot d^\mu}}+\left(\sqrt{\left|\tilde{\sigma}_{V^{\text{in}}}-\sigma_{V^{\text{in}}}\right|}\right)\sqrt{\frac{4\cdot\log(4SA/\delta)}{m\cdot d^\mu}}+\frac{4V_{\max}}{3m\cdot d^\mu}\log(4SA/\delta)\\
		\leq &\sqrt{\frac{4\cdot\tilde{\sigma}_{V^{\text{in}}}\cdot\log(4SA/\delta)}{m\cdot d^\mu}}+2\sqrt{6}\cdot V_{\max}\cdot\left(\frac{\log(16SA/\delta)}{m\cdot d^\mu}\right)^{3/4}+\frac{16V_{\max}}{3m\cdot d^\mu}\log(16SA/\delta).
		\end{align*}
		Since $e = \sqrt{4\cdot\tilde{\sigma}_{V^{\tin}}\cdot\log(4SA/\delta)/(m\cdot d^\mu)}+2\sqrt{6}\cdot V_{\max}\cdot\left(\log(16SA/\delta)/(m\cdot d^\mu)\right)^{3/4}+16V_{\max}\log(16SA/\delta)/(3m\cdot d^\mu)$,
		 from above we have 
		\begin{equation}\label{eqn:upper_z_in}
		z=\tilde{z}-e\leq \PP V^{\tin},
		\end{equation}
		and 
		\begin{equation}\label{eqn:lower_z_in}
		z\geq \PP V^{\tin}-2e.
		\end{equation}
		
		Next note $\sqrt{\sigma_{(\cdot)}}$ is a norm, so by norm triangle inequality (for the second inequality) and $\sqrt{a}\leq\sqrt{b}+\sqrt{|b-a|}$ with \eqref{eqn:diff_sig_in} (for the first inequality) we have 
		\begin{align*}
		\sqrt{\tilde{\sigma}_{V^{\tin}}}\leq&\sqrt{\sigma_{V^{\tin}}} +  V_{\max}\left[\left(\frac{36\log(16SA/\delta)}{m\cdot d^\mu}\right)^{1/4}+\sqrt{\frac{4\log(16SA/\delta)}{m\cdot d^\mu}}\right]\\
		\leq&\sqrt{\sigma_{V^{\star}}} +\sqrt{\sigma_{V^{\star}-V^{\tin}}} +  V_{\max}\left[\left(\frac{36\log(16SA/\delta)}{m\cdot d^\mu}\right)^{1/4}+\sqrt{\frac{4\log(16SA/\delta)}{m\cdot d^\mu}}\right]\\
		\leq&\sqrt{\sigma_{V^{\star}}} +\sqrt{\PP(V^{\star}-V^{\tin})^2} +  V_{\max}\left[\left(\frac{36\log(16SA/\delta)}{m\cdot d^\mu}\right)^{1/4}+\sqrt{\frac{4\log(16SA/\delta)}{m\cdot d^\mu}}\right]\\
		\leq&\sqrt{\sigma_{V^{\star}}} +||V^{\star}-V^{\tin}||_\infty\cdot\mathbf{1} +  V_{\max}\left[\left(\frac{36\log(16SA/\delta)}{m\cdot d^\mu}\right)^{1/4}+\sqrt{\frac{4\log(16SA/\delta)}{m\cdot d^\mu}}\right]\\
		\leq&\sqrt{\sigma_{V^{\star}}} +u^{\tin}\cdot\mathbf{1} +  V_{\max}\left[\left(\frac{36\log(16SA/\delta)}{m\cdot d^\mu}\right)^{1/4}+\sqrt{\frac{4\log(16SA/\delta)}{m\cdot d^\mu}}\right]\\
		\end{align*}
		Plug this back to \eqref{eqn:lower_z_in} we get
		\begin{equation}\label{eqn:lower_zt_in}
		\begin{aligned}
		z\geq& \PP V^{\tin}-\sqrt{\frac{16\cdot{\sigma}_{V^{\star}}\cdot\log(4SA/\delta)}{m\cdot d^\mu}}-\sqrt{\frac{16\cdot\log(4SA/\delta)}{m\cdot d^\mu}}\cdot u^{\tin}\\
		-& V_{\max}\left[8\sqrt{6}\cdot\left(\frac{\log(16SA/\delta)}{m\cdot d^\mu}\right)^{3/4}+\frac{56\log(16SA/\delta)}{3m\cdot d^\mu}\right].
		\end{aligned}
		\end{equation}
		To sum up, so far we have shown that \eqref{eqn:upper_z_in}, \eqref{eqn:lower_zt_in} hold with probability $1-\delta/2$ and we condition on that. 
		
		\textbf{Step2:} Next we prove 
		\begin{equation}\label{eqn:inter_in}
		Q^{(i)}\leq r+\gamma\PP V^{(i)},\quad V^{\tin}\leq V^{(i)}\leq V^\star,\quad\forall i\in[R]
		\end{equation}
		using backward induction. 
		
		First of all, $V^{(0)}=V^{\tin}$ implies $V^{\tin}\leq V^{(0)}\leq V^\star$ and 
		$Q^{(0)}:=\mathbf{0}\leq r+\gamma \PP V^{(0)}$ so the results hold for the base case. 
		
		Now for certain $i$, using induction assumption we can assume with probability at least $1-(i-1)\delta/R$, for all $i'=0,...,i-1$,  
		\begin{equation}\label{eqn:intermediate_in}
		Q^{(i')}\leq r+\gamma\PP V^{(i')}\qquad	V^{\tin}\leq V^{(i')}\leq V^\star
		\end{equation}
		In particular, since $V^{\tin}\leq V^\star\leq V^{\tin}+u^{\tin}\mathbf{1}$, so combine this and \eqref{eqn:intermediate_in} for $i'=i-1$ we get
		\[
		V^\star-V^{(i-1)}\leq V^\star- V^{\tin}\leq u\mathbf{1}.
		\]
		By Lemma~\ref{lem:con_g_in}, with probability $1-\delta/R$,
		\begin{equation}\label{eqn:diff_g_t_in}
		\PP[V^{(i)}-V^{\tin}]- 8u^{\tin}\sqrt{\frac{\log(2RSA/\delta)}{l d^\mu}}\leq g^{(i)}\leq \PP[V^{(i)}-V^{\tin}].
		\end{equation}
		By the right hand side of this and \eqref{eqn:upper_z_in} we acquire with probability $1-i\delta/R$,
		\[
		Q^{(i)}=r+\gamma z+\gamma g^{(i)}\leq  r+ \gamma\PP V^{\tin}+ \gamma \PP[V^{(i)}-V^{\tin}] =  r+\gamma\PP V^{(i)}
		\]
		where the second equality already gives the proof of the first part of claim~\eqref{eqn:inter_in}. Moreover, by induction assumption $V^{(i-1)}\leq V^\star$ we have 
		\[
		Q^{(i-1)}\leq  r+\gamma\PP V^{(i-1)}\leq r+\gamma\PP V^\star=Q^\star,
		\]
		 which implies $V_{Q^{(i-1)}}\leq V_{Q^\star}=V^\star$, therefore we have 
		\[
		V^{(i)}=\max(V_{Q^{(i-1)}},V^{(i-1)})\leq V^\star_t,
		\]	
		this completes the proof of the second part of claim~\eqref{eqn:inter_in}. 
		
		\textbf{Step3:} Next we prove $V^{(i)}\leq\mathcal{T}_{\pi^{(i)}}V^{(i)}$. 
		
		For a particular $s$, on one hand, if $\pi^{(i)}(s)=\argmax_{a} Q^{(i-1)}(s,a)$, by $Q^{(i-1)}\leq r+\gamma\PP V^{(i-1)}$ we have in this case:  
		\begin{align*}
		V^{(i)}(s)=&\max_{a} Q^{(i-1)}(s,a)=Q^{(i-1)}(s,\pi^{(i)}(s))\leq r(s,\pi^{(i)}(s))+\gamma P^\top(\cdot|s,\pi^{(i)}(s))V^{(i-1)}\\
		\leq&r(s,\pi^{(i)}(s))+\gamma P^\top(\cdot|s,\pi^{(i)}(s))V^{(i)}=(\mathcal{T}_{\pi^{(i)}}V^{(i)})(s),
		\end{align*}
		where the first equal sign comes from the definition of $V^{(i)}$ when $V_{Q^{(i-1)}}(s)\geq V^{\tin}(s)$ and the first inequality is from Step2.
		
		On the other hand, if $\pi^{(i)}(s)=\pi^{(i-1)}(s)$, then
		\[
		V^{(i)}(s)=V^{(i-1)}(s)\leq (\mathcal{T}_{\pi^{(i-1)}}V^{(i-1)})(s)\leq(\mathcal{T}_{\pi^{(i-1)}}V^{(i)})(s)=(\mathcal{T}_{\pi^{(i)}}V^{(i)})(s).
		\]

		\textbf{Step4:}	It remains to check $Q^\star-Q^{(i)}\leq \gamma\PP^{\pi^\star}[Q^\star-Q^{(i-1)}]+{\xi}$. Indeed, using the construction of $Q^{(i)}$, we have
		\begin{equation}\label{eqn:diff_q_t_in}
		\begin{aligned}
		&Q^\star-Q^{(i)}=Q^\star-r-\gamma z-\gamma g^{(i)}=\gamma\PP V^\star-\gamma z-\gamma g^{(i)}\\
		=&\gamma[\PP V^\star-\PP(V^{(i)}-V^{\tin})-\PP V^{\tin}]+{\xi}=\gamma\PP V^\star-\gamma\PP V^{(i)}+{\xi},
		\end{aligned}
		\end{equation}
		where the second equation uses Bellman optimality equation and the third equation uses the definition of ${\xi}=\gamma[\PP(V^{(i)}-V^{\tin})-g^{(i)}+\PP V^{\tin}-z]$. By \eqref{eqn:lower_zt_in} and \eqref{eqn:diff_g_t_in},
		\begin{align*}
		{\xi}\leq&8 u^{\tin}\sqrt{\frac{\log(2RSA/\delta)}{l d^\mu}}+\sqrt{\frac{16\cdot{\sigma}_{V^{\star}}\cdot\log(4SA/\delta)}{m\cdot d^\mu}}+\sqrt{\frac{16\cdot\log(4SA/\delta)}{m\cdot d^\mu}}\cdot u^{\tin}\\
		+&  V_{\max}\left[8\sqrt{6}\cdot\left(\frac{\log(16SA/\delta)}{m\cdot d^\mu}\right)^{3/4}+\frac{56\log(16SA/\delta)}{3m\cdot d^\mu}\right].
		\end{align*}
		Lastly, note 
		$
		\PP V^\star=\PP^{\pi^\star} Q^\star
		$
		and from $V^{(i)}\geq V_{Q^{(i-1)}}$, we have $\PP V^{(i)}\geq\PP V_{Q^{(i-1)}}=\PP^{\pi_{Q^{i-1}}} Q^{(i-1)}\geq \PP^{\pi^\star} Q^{(i-1)}$, the last inequality holds true since $\pi_{Q^{(i-1)}}$ is the greedy policy over $Q^{(i-1)}$. Threfore \eqref{eqn:diff_q_t_in} becomes $Q^\star-Q^{(i)}=\gamma\PP V^\star-\gamma\PP V^{(i)}+{\xi}\leq\gamma\PP^{\pi^\star} Q^\star-\gamma\PP^{\pi^\star} Q^{(i-1)}+{\xi}$. This completes the proof.

	\end{proof}

	\begin{lemma}\label{lem:complexity_in}
		Suppose the input $V^{\tin}$ of Algorithm~\ref{alg:OPDVRT_in} satisfies $V^{\tin}\leq\mathcal{T}_{\pi^{\tin}}V^{\tin}$ and $V^{\tin}\leq V^\star\leq V^{\tin}+u^{\tin}\mathbf{1}$. Let $V^\text{out}$, $\pi^\text{out}$ be the return of inner loop of Algorithm~\ref{alg:OPDVRT_in} and choose $m=l^{(i)}:=m' \cdot \log(16RSA)/(u^{\tin})^2$, where $m'$ is a parameter will be decided later. Then in addition to the results of Lemma~\ref{lem:VR_incremental_in}, we have with probability $1-\delta$, 
		\begin{itemize}
			\item if $u^{\tin}\in[\sqrt{1/(1-\gamma)},1/(1-\gamma)]$, then: 
			\begin{align*}
			&\mathbf{0}\leq V^\star-V^\text{out}\leq\\
			\leq &\bigg(\frac{12/(1-\gamma)^2}{\sqrt{m' }}\norm{\d^{\pi^\star}_{t}\sqrt{\frac{1}{d^\mu}}}_{\infty}+\frac{4}{\sqrt{m'}}\norm{\sum_{t=0}^{\infty}\gamma^t\d^{\pi^\star}_{t}\sqrt{\frac{{\sigma}_{V^{\star}}}{d^\mu}}}_{\infty} + \frac{8\sqrt{6}(1/(1-\gamma))^{\frac{10}{4}}}{(m')^{3/4}}\norm{\d^{\pi^\star}_{t}\left[\frac{1}{d^\mu}\right]^{\frac{3}{4}}}_{\infty}\\
			+& \frac{56 /(1-\gamma)^3}{3m'}   \norm{\d^{\pi^\star}_{t}\frac{1}{d^\mu}}_{\infty}\bigg)u^{\tin}\cdot\mathbf{1}+\frac{u^{\tin}}{4}\mathbf{1}.
			\end{align*}
			\item if $u^{\tin}\leq \sqrt{1/(1-\gamma)}$, then 
			\begin{align*}
			&\mathbf{0}\leq V^\star-V^\text{out}\leq\\
			\leq &\bigg(\frac{12\sqrt{(1/(1-\gamma))^3}}{\sqrt{m' }}\norm{\d^{\pi^\star}_{t}\sqrt{\frac{1}{d^\mu}}}_{\infty}+\frac{4}{\sqrt{m'}}\norm{\sum_{t=0}^{\infty}\gamma^t\d^{\pi^\star}_{t}\sqrt{\frac{{\sigma}_{V^{\star}}}{d^\mu}}}_{\infty} + \frac{8\sqrt{6}(1/(1-\gamma))^{\frac{9}{4}}}{(m')^{3/4}}\norm{\d^{\pi^\star}_{t}\left[\frac{1}{d^\mu}\right]^{\frac{3}{4}}}_{\infty}\\
			+& \frac{56 (1/(1-\gamma))^{\frac{5}{2}}}{3m'}   \norm{\d^{\pi^\star}_{t}\frac{1}{d^\mu}}_{\infty}\bigg)u^{\tin}\cdot\mathbf{1}+\frac{u^{\tin}}{4}\mathbf{1}.
			\end{align*}
		\end{itemize}
		
		where  $\d^{\pi^\star}_{t}\in\R^{S\cdot A\times S\cdot A}$ is a matrix represents the multi-step transition from time $0$ to $t$, \emph{i.e.} 
		$\d^{\pi^\star}_{(s,a),(s',a')}=d^{\pi^\star}_{0:t}(s',a'|s,a)$ and recall $1/d^\mu$ is a vector. $\d^{\pi^\star}_{t}\frac{1}{d^\mu}$ is a matrix-vector multiplication. For a vector $d_t\in \R^{S\times A}$, norm $||\cdot||_{\infty}$ is defined as $||d_t||_{\infty}=\max_{t,s,a}d_t(s,a)$. 
	\end{lemma}

	\begin{proof}
		By Lemma~\ref{lem:VR_incremental_in}, we have with probability $1-\delta$, for all $t\in[H]$,
		\begin{align*}
		V^{\text{in}}\leq V^{(i)}\leq\mathcal{T}_{\pi^{(i)}}V^{(i)}\leq V^\star,\quad Q^{(i)}\leq r+\gamma\PP V^{(i)},\quad \text{and}\quad Q^\star-Q^{(i)}\leq \gamma\PP^{\pi^\star}[Q^\star-Q^{(i-1)}]+{\xi},
		\end{align*} 
		where \begin{align*}
		{\xi}\leq&8 u^{\tin}\sqrt{\frac{\log(2RSA/\delta)}{l d^\mu}}+\sqrt{\frac{16\cdot{\sigma}_{V^{\star}}\cdot\log(4SA/\delta)}{m\cdot d^\mu}}+\sqrt{\frac{16\cdot\log(4SA/\delta)}{m\cdot d^\mu}}\cdot u^{\tin}\\
		+&  V_{\max}\left[8\sqrt{6}\cdot\left(\frac{\log(16SA/\delta)}{m\cdot d^\mu}\right)^{3/4}+\frac{56\log(16SA/\delta)}{3m\cdot d^\mu}\right].
		\end{align*}
		
		Applying the recursion repeatedly, we obtain
		\[
		Q^\star-Q^{(R)}\leq \gamma^R\PP^{\pi^\star}[Q^\star-Q^{(0)}]+\sum_{i=0}^R \gamma^i\left(\PP^{\pi^\star}\right)^i\xi\leq \gamma^R\PP^{\pi^\star}[Q^\star-Q^{(0)}]+\sum_{i=0}^\infty \gamma^i\left(\PP^{\pi^\star}\right)^i\xi
		\]
		Note $(\PP^{\pi^\star})^i\in\R^{S\cdot A\times S\cdot A}$ represents the multi-step transition from time $0$ to $i$, \emph{i.e.} 
		$(\PP^{\pi^\star})^i_{(s,a),(s',a')}=d^{\pi^\star}_i(s',a'|s,a)$. Recall $R=\ln (4/u^{\tin}(1-\gamma))$, then
		\[
		\gamma^R\PP^{\pi^\star}[Q^\star-Q^{(0)}]\leq \gamma^R ||Q^\star-Q^{(0)}||_\infty\leq\gamma^R V_{\max}=\gamma^R/(1-\gamma)\leq u^{\tin}/4.
		\]

		Therefore 
		{
			\begin{equation}\label{eqn:q_diff_decomp_in}
			\begin{aligned}
			&Q^\star-Q^{(R)}\leq\frac{u^{\tin}}{4}+\sum_{t=0}^\infty \gamma^t\d^{\pi^\star}_t\xi\\
			\leq&\sum_{t=0}^{\infty}\gamma^t\d^{\pi^\star}_{t}\bigg(8 u^{\tin}\sqrt{\frac{\log(2RSA/\delta)}{l d^\mu}}+\sqrt{\frac{16\cdot{\sigma}_{V^{\star}}\cdot\log(4SA/\delta)}{m\cdot d^\mu}}+\sqrt{\frac{16\cdot\log(4SA/\delta)}{m\cdot d^\mu}}\cdot u^{\tin}\\
			+&  V_{\max}\left[8\sqrt{6}\cdot\left(\frac{\log(16SA/\delta)}{m\cdot d^\mu}\right)^{3/4}+\frac{56\log(16SA/\delta)}{3m\cdot d^\mu}\right]\bigg)+\frac{u^{\tin}}{4},
			\end{aligned}
			\end{equation}
		}
		\normalsize{}
		Now by our choice of $m=l^{(i)}:=m' \cdot \log(16RSA/\delta)/(u^{\tin})^2$, then the first term of \eqref{eqn:q_diff_decomp_in} is further less than

		\begin{equation}\label{eqn:decomp_genereal_in}
		\begin{aligned}
		\leq &\sum_{t=0}^{\infty}\gamma^t\d^{\pi^\star}_{t}\left(\frac{12u^{\tin}}{\sqrt{m' d^\mu}}u^{\tin}+\sqrt{\frac{16\cdot{\sigma}_{V^{\star}}}{m'\cdot d^\mu}}u^{\tin}
		+ V_{\max}\left[8\sqrt{6}\cdot\left(\frac{(u^{\tin})^{ 2/3}}{m'\cdot d^\mu}\right)^{3/4}+\frac{56u^{\tin}}{3m'\cdot d^\mu}\right]\cdot u^{\tin}\right)\\
		\end{aligned}
		\end{equation}
		
		\textbf{Case1.} If $u^{\tin}\leq \sqrt{1/(1-\gamma)}$, then \eqref{eqn:decomp_genereal_in} is less than 
		\begin{equation}\label{eqn:decomp_genereal_case1_in}
		\begin{aligned}
		\leq &\sum_{t=0}^{\infty}\gamma^t\d^{\pi^\star}_{t}\left(\frac{12\sqrt{1/(1-\gamma)}}{\sqrt{m' d^\mu}}+\sqrt{\frac{16\cdot{\sigma}_{V^{\star}}}{m'\cdot d^\mu}}
		+ V_{\max}\left[8\sqrt{6}\cdot\left(\frac{(1/(1-\gamma))^{ 1/3}}{m'\cdot d^\mu_{t'}}\right)^{3/4}+\frac{56(1/(1-\gamma))^{1/2}}{3m'\cdot d^\mu_{t'}}\right] \right)u^{\tin}\\
		\leq &\bigg(\frac{12\sqrt{(1/(1-\gamma))^3}}{\sqrt{m' }}\norm{\d^{\pi^\star}_{t}\sqrt{\frac{1}{d^\mu}}}_{\infty}+\frac{4}{\sqrt{m'}}\norm{\sum_{t=0}^{\infty}\gamma^t\d^{\pi^\star}_{t}\sqrt{\frac{{\sigma}_{V^{\star}}}{d^\mu}}}_{\infty} + \frac{8\sqrt{6}(1/(1-\gamma))^{\frac{9}{4}}}{(m')^{3/4}}\norm{\d^{\pi^\star}_{t}\left[\frac{1}{d^\mu}\right]^{\frac{3}{4}}}_{\infty}\\
		+& \frac{56 (1/(1-\gamma))^{\frac{5}{2}}}{3m'}   \norm{\d^{\pi^\star}_{t}\frac{1}{d^\mu}}_{\infty}\bigg)u^{\tin}\cdot\mathbf{1}.
		\end{aligned}
		\end{equation}
		
		\textbf{Case2.} If $u^{\tin}\geq \sqrt{1/(1-\gamma)}$, then \eqref{eqn:decomp_genereal_in} is less than 
		\begin{equation}\label{eqn:decomp_genereal_case2_in}
		\begin{aligned}
		\leq &\sum_{t=0}^{\infty}\gamma^t\d^{\pi^\star}_{t}\left(\frac{12/(1-\gamma)}{\sqrt{m' d^\mu}}+\sqrt{\frac{16\cdot{\sigma}_{V^{\star}}}{m'\cdot d^\mu}}
		+ V_{\max}\left[8\sqrt{6}\cdot\left(\frac{(1/(1-\gamma))^{ 2/3}}{m'\cdot d^\mu_{t'}}\right)^{3/4}+\frac{56/(1-\gamma)}{3m'\cdot d^\mu_{t'}}\right] \right)u^{\tin}\\
		\leq &\bigg(\frac{12/(1-\gamma)^2}{\sqrt{m' }}\norm{\d^{\pi^\star}_{t}\sqrt{\frac{1}{d^\mu}}}_{\infty}+\frac{4}{\sqrt{m'}}\norm{\sum_{t=0}^{\infty}\gamma^t\d^{\pi^\star}_{t}\sqrt{\frac{{\sigma}_{V^{\star}}}{d^\mu}}}_{\infty} + \frac{8\sqrt{6}(1/(1-\gamma))^{\frac{10}{4}}}{(m')^{3/4}}\norm{\d^{\pi^\star}_{t}\left[\frac{1}{d^\mu}\right]^{\frac{3}{4}}}_{\infty}\\
		+& \frac{56 /(1-\gamma)^3}{3m'}   \norm{\d^{\pi^\star}_{t}\frac{1}{d^\mu}}_{\infty}\bigg)u^{\tin}\cdot\mathbf{1}.
		\end{aligned}
		\end{equation}

	\end{proof}

	Next, let us first finish the proof the Algorithm~\ref{alg:OPDVRT_in}.
	
	\begin{lemma}\label{lem:complexity_OPVRT_in}
		For convenience, define:
		\[
		A_{\frac{1}{2}}=\sup_t \norm{\d^{\pi^\star}_{0:t}\sqrt{\frac{1}{d^\mu}}}_{\infty}, \;A_2=\sup_t\norm{\sum_{t=0}^{\infty}\gamma^t\d^{\pi^\star}_{0:t}\sqrt{\frac{{\sigma}_{V^{\star}}}{d^\mu}}}_{\infty},\;A_{\frac{3}{4}}=\sup_t\norm{\d^{\pi^\star}_{0:t}\left[\frac{1}{d^\mu}\right]^{\frac{3}{4}}}_{\infty},\;\;A_1=\norm{\d^{\pi^\star}_{0:t}\frac{1}{d^\mu}}_{\infty}.
		\] 
		Recall $\epsilon$ is the target accuracy in the outer loop of Algorithm~\ref{alg:OPDVRT_in} and $R=\ln (4/\epsilon(1-\gamma))$. Then:
		\begin{itemize}
						\item If $u^{(0)}> \sqrt{(1-\gamma)^{-1}}$, then let $m^{(j)}=l^{(i,j)}=m'\log(16(1-\gamma)^{-1}SARK)/(u^{(i-1)})^2$, where 
			\begin{align*}
			m_1'&=\max\left[ 96^2(1-\gamma)^{-4} A_{\frac{1}{2}}^2,32^2A_2^2,\left(64\sqrt{6}A_{\frac{3}{4}}\right)^{\frac{4}{3}}(1-\gamma)^{-\frac{10}{3}},\frac{448}{3}(1-\gamma)^{-3}A_1\right], \\ K_1&=\log_2({(1-\gamma)^{-1}}/\epsilon),
			\end{align*}
			\item If $u^{(0)}\leq \sqrt{(1-\gamma)^{-1}}$, then let $m^{(j)}=l^{(i,j)}=m'\log(16(1-\gamma)^{-1}SARK/\delta)/(u^{(j-1)})^2$, where 
			\begin{align*}
			m_2'&=\max\left[ 96^2(1-\gamma)^{-3}A_{\frac{1}{2}}^2,32^2A_2^2,\left(64\sqrt{6}A_{\frac{3}{4}}\right)^{\frac{4}{3}}(1-\gamma)^{-3},\frac{448}{3}(1-\gamma)^{-5/2}A_1\right],\\
			 \quad K_2&=\log_2(\sqrt{(1-\gamma)^{-1}}/\epsilon),
			\end{align*}
		\end{itemize}
		Algorithm~\ref{alg:OPDVRT_in} obeys that, with probability $1-\delta$,the output $\pi^{(K)}$ is an $\epsilon$-optimal policy, i.e. $||V_1^\star-V_1^{\pi^{(K)}}||_\infty <\epsilon$ with total sample complexity:
		\[
		O\left(\frac{m'\log(16(1-\gamma)^{-1}SARK/\delta)}{\epsilon^2}RK\right)
		\]
		for both cases. Moreover, $m'$ can be simplified as:
		\begin{itemize}
			\item If $u^{(0)}\leq \sqrt{(1-\gamma)^{-1}}$,then $m'\leq c(1-\gamma)^{-3}/d_m$;
			\item If $u^{(0)}> \sqrt{(1-\gamma)^{-1}}$, then $m'\leq c (1-\gamma)^{-4}/d_m$.
		\end{itemize}
	\end{lemma}

	\begin{proof}
	The proof of this lemma follows the same logic as Lemma~\ref{lem:complexity_OPVRT}. Note there is additional logarithmic factor $R$ since the Inner loop of Algorithm~\ref{alg:OPDVRT_in} has an extra \textbf{For} loop. Also, $A_2$ can be bounded by $O((1-\gamma)^{-3/2})$ due to the following counterpart result of Lemma~\ref{lem:H3}: 
		\begin{align*}
	\sum_{t'=t}^{\infty}\gamma^{t'}\E^{\pi^\star}_{s_{t'},a_{t'}}\left[{\sigma}_{V^{\star}}(s_{t'},a_{t'})\middle|s_t,a_t\right]\leq \Var_{\pi^\star}\left[\sum_{t'=t}^{\infty}\gamma^{t'} r_{t'}\middle|s_t,a_t\right]
		\end{align*}
	which reduces the dependence from $(1-\gamma)^{-3}$ to $(1-\gamma)^{-2}$.
	\end{proof}

	\subsection{Proof of Theorem~\ref{thm:main_infinite}}
	
	\begin{proof}
	Again the proof relies on the two stages of Algorithm~\ref{alg:OPDVR_in} where the first stage reduces the error to the level below $\sqrt{(1-\gamma)^{-1}}$ and the next stage decrease the error to given accuracy. Moreover, bounding the event $\{n_{s,a}\leq \frac{1}{2}m d^\mu(s,a)\}$ using Lemma~\ref{lem:chernoff_multiplicative} is valid since data $\mathcal{D}$ is i.i.d. and $n_{s,a}=\sum_{i=1}^m\mathbf{1}[s^{(i)}=s,a^{(i)}=a]$ follows binomial distribution with $\E[n_{s,a}]=m d^\mu(s,a)$.
	\end{proof}

\section{Proof of Theorem~\ref{thm:learn_low_bound_s}}
\label{sec:lower}

We prove the offline learning lower bound (best policy identification in the offline regime) of $\Omega(H^2/d_m\epsilon^2)$ for stationary transition case. Our proof consists of two steps: we will first show a minimax lower bound ({over all MDP instances}) for learning $\epsilon$-optimal policy is $\Omega(H^2SA/\epsilon^2)$; next we can further improve the lower bound ({over problem class $\mathcal{M}_{d_m}$}) for learning $\epsilon$-optimal policy to $\Omega(H^2/d_m\epsilon^2)$ by a reduction of the first result.

There are numerous literature that provide information theoretical lower bounds under different setting, \textit{e.g.} \cite{dann2015sample,jiang2017contextual,krishnamurthy2016pac,jin2018q,sidford2018near,domingues2020episodic,yin2020near,zanette2020exponential,duan2020minimax,wang2020statistical,jin2020pessimism}. However, to the best of our knowledge, \citet{yin2020near} is the only one that gives the lower bound for explicit parameter dependence in offline case. Concretely, their lower bound $\Omega(H^3/d_m\epsilon^2)$ (for non-stationary setting) includes $d_m$ which is an inherent measure of offline problems. In the stationary transition setting, by a modification of their construction (which again originated from \citet{jiang2017contextual}) we can prove the lower bound of $\Omega(H^2/d_m\epsilon^2)$.

\subsection{Information theoretical lower sample complexity bound over all MDP instances for identifying $\epsilon$-optimal policy.}  

\begin{theorem}\label{thm:learn_low_bound_all}
	Given $H\geq 2$, $A\geq2$, $0<\epsilon<\frac{1}{48\sqrt{8}}$ and $S\geq c_1$ where $c_1$ is a universal constant. Then for any algorithm and any $n\leq cH^2SA/\epsilon^2$, there exists a non-stationary $H$ horizon MDP with probability at least $p$, the algorithm outputs a policy $\hpi$ 
	with $v^\star- v^{\hpi}\geq\epsilon $. 
\end{theorem}

The proof relies on embedding $\Theta(S)$ independent multi-arm bandit problems into a family of hard-to-learn MDP instances so that any algorithm that wants to output a near-optimal policy needs to identify the best action in $\Omega(S)$ problems. By standard multi-arm bandit identification result Lemma~\ref{lem:bandit} we need $O(SA)$ episodes. To recover the $H^2$ factor, we only assign reward $1$ to ``good'' states in the latter half of the MDP and all other states have reward $0$.

\begin{proof}[Proof of Theorem~\ref{thm:learn_low_bound_all}]
	We construct a non-stationary MDP with $S$ states per level, $A$ actions per state and has horizon $2H$. States are categorized into three types with two special states $g$, $b$ and the remaining $S-2$ ``bandit'' states denoted by $s_{i}$, $i\in[S-2]$. Each bandit state has an unknown best action $a^\star_{i}$ that provides the highest expected reward comparing to other actions.

	The transition dynamics are defined as follows:
	\begin{itemize}
		
		\item for $h=1,...,2H-1$, 
		\begin{itemize}
			\item For bandit states $b_{i}$, there is probability $1-\frac{1}{H}$ to transition back to itself ($b_{i}$) regardless of the action chosen. For the rest of $\frac{1}{H}$ probability, optimal action $a^\star_{i}$ have probability $\frac{1}{2}+\tau$ or $\frac{1}{2}-\tau$ transition to $g$ or $b$ respectively and all other actions $a$ will have equal probability $\frac{1}{2}$ for either $g$ or $b$, where $\tau$ is a parameter will be decided later. Or equivalently,
			
			\[
			\P(\cdot|s_{i},a^\star_{i})=\begin{cases} 1-\frac{1}{H}\quad &\text{if}\;\cdot=s_{i}\\
			(\frac{1}{2}+\tau)\cdot\frac{1}{H}\quad &\text{if}\;\cdot=g\\
			(\frac{1}{2}-\tau)\cdot\frac{1}{H} \quad &\text{if}\;\cdot=b\end{cases}\quad 	\P(\cdot|s_{i},a)=\begin{cases} 1-\frac{1}{H}\quad &\text{if}\;\cdot=s_{i}\\
			\frac{1}{2}\cdot\frac{1}{H}\quad &\text{if}\;\cdot=g\\
			\frac{1}{2}\cdot\frac{1}{H} \quad &\text{if}\;\cdot=b\end{cases}
			\] 
			\item $g$ always transitions to $g$ and $b$ always transitions to $b$, \textit{i.e.} for all $a\in\mathcal{A}$,
			\[
			\P(g|g,a)=1,\quad \P(b|b,a)=1.
			\]
			We will determine parameter $\tau$ at the end of the proof.

		\end{itemize}

		\item Reward assignment: the instantaneous reward is $1$ if and only if state $s=g$ and the current time $t\in\{H,\ldots,2H-1\}$. In all other cases, the reward is $0$. \emph{i.e.},
		\[
		\begin{cases}
		r(s_t,a)=1 \;\;iff\;\;s_t=g\;\;and\;\;t\geq H,\\
		r(s_t,a)=0\;\;o.w.
		\end{cases}
		\]
		
		\item The initial distribution is decided by:
		\begin{equation}\label{eqn:initial_c_MDP}
		\P(s_{i})=\frac{1}{S},\; \forall i\in[S-2],\;	\P(g)=\frac{1}{S},\;\;\P(b)=\frac{1}{S}
		\end{equation}
		
	\end{itemize}
	
	By this construction the optimal policy must take $a^\star_{i}$ for each bandit state $s_{i}$ for at least the first half of the MDP (when $t\leq H$). In other words, this construction embeds $(S-2)$ independent best arm identification problems that are identical to the stochastic multi-arm bandit problem in Lemma~\ref{lem:bandit} into the MDP for the following two reasons: \textbf{1.} the transition is stationary (the optimal arm $a_i^\star$ for state $s_i$ is identical across all time $t$) so instead of $H(S-2)$ (for non-stationary case) MAB problems we only have $S-2$ of them; \textbf{2.} all $S-2$ problems are independent since each state $s_i$ can only transition to themselves or $g$, $b$.

	Notice for any time $h$ with $h\leq H$, any bandit state $s_{i}$, the difference of the expected reward between optimal action $a_{i}^\star$ and other actions is:
	\begin{equation}\label{eqn:diff_reward}
	\begin{aligned}
	&(\frac{1}{2}+\tau)\cdot\frac{1}{H}\cdot \E[r_{{(h+1)}:2H}|g]+(\frac{1}{2}-\tau)\cdot\frac{1}{H}\cdot \E[r_{{(h+1)}:2H}|b]+(1-\frac{1}{H})\cdot \E[r_{{(h+1)}:2H}|s_{i}]\\
	&-\frac{1}{2H}\cdot \E[r_{{(h+1)}:2H}|g]-\frac{1}{2H}\cdot \E[r_{{(h+1)}:2H}|b]-(1-\frac{1}{H})\cdot \E[r_{{(h+1)}:2H}|s_{i}]\\
	=&(\frac{1}{2}+\tau)\cdot\frac{1}{H}\cdot \E[r_{{(h+1)}:2H}|g]+(\frac{1}{2}-\tau)\cdot\frac{1}{H}\cdot \E[r_{{(h+1)}:2H}|b]\\
	&-\frac{1}{2H}\cdot \E[r_{{(h+1)}:2H}|g]-\frac{1}{2H}\cdot \E[r_{{(h+1)}:2H}|b]\\
	=&(\frac{1}{2}+\tau)\frac{1}{H}\cdot H+(\frac{1}{2}-\tau)\frac{1}{H}\cdot 0-\frac{1}{2H}\cdot H+\frac{1}{2H}\cdot 0=\tau 
	\end{aligned}
	\end{equation}
	so it seems by Lemma~\ref{lem:bandit} one suffices to use the least possible $\frac{A}{72(\tau)^2}$ samples to identify the best action $a_{i}^\star$. However, note observing $\sum_{t=1}^{2H}r_t=H$ is equivalent as observing $\sum_{t=1}^{H}r_t=1$ (since $\sum_{t=1}^{H}r_t=1$ is equivalent to $s_H=g$ and is equivalent to $\sum_{t=1}^{H}r_t=1$). Therefore, for the bandit states in the first half the samples that provide information for identifying the best arm is up to time $H$. Or in other words, identify best arm in stationary transition setting can be decided in each single stage after $t\geq H$. As a result, the difference of the expected reward between optimal action $a_{h,i}^\star$ and other action for identifying the best arm should be corrected as: 
	\begin{align*}
	&(\frac{1}{2}+\tau)\cdot\frac{1}{H}\cdot \E[r_{{(h+1)}:H}|g]+(\frac{1}{2}-\tau)\cdot\frac{1}{H}\cdot \E[r_{{(h+1)}:H}|b]+(1-\frac{1}{H})\cdot \E[r_{{(h+1)}:H}|s_{i}]\\
	&-\frac{1}{2H}\cdot \E[r_{{(h+1)}:H}|g]-\frac{1}{2H}\cdot \E[r_{{(h+1)}:H}|b]-(1-\frac{1}{H})\cdot \E[r_{{(h+1)}:H}|s_{i}]\\
	=&(\frac{1}{2}+\tau)\frac{1}{H}\cdot 1+(\frac{1}{2}-\tau)\frac{1}{H}\cdot 0-\frac{1}{2H}\cdot 1+\frac{1}{2H}\cdot 0= \frac{\tau}{H}
	\end{align*}
	or one can compute any bandit state in latter half ($h\geq H$):
	\begin{align*}
	&(\frac{1}{2}+\tau)\cdot\frac{1}{H}\cdot \E[r_{h:h+1}|g]+(\frac{1}{2}-\tau)\cdot\frac{1}{H}\cdot \E[r_{h:h+1}|b]+(1-\frac{1}{H})\cdot \E[r_{h:h+1}|s_i]\\
	&-\frac{1}{2H}\cdot \E[r_{h:h+1}|g]-\frac{1}{2H}\cdot \E[r_{h:h+1}|g]-(1-\frac{1}{H})\cdot \E[r_{{h:h+1}}|s_{i}]\\
	=&(\frac{1}{2}+\tau)\frac{1}{H}\cdot 1+(\frac{1}{2}-\tau)\frac{1}{H}\cdot 0-\frac{1}{2H}\cdot 1+\frac{1}{2H}\cdot 0= \frac{\tau}{H},
	\end{align*}
	which yields the same result. Now by Lemma~\ref{lem:bandit}, unless $\frac{A}{72(\tau/H)^2}$ samples are collected from that bandit state, the learning algorithm fails to identify the optimal action $a^\star_{i}$ with probability at least $1/3$.
	
	After running any algorithm, let $C$ be the set of bandit states for which the algorithm identifies
	the correct action. Let $D$ be the set of bandit states for which the algorithm collects fewer than $\frac{A}{72(\tau/H)^2}$ samples. Then by Lemma~\ref{lem:bandit} we have
	\begin{align*}
	\E[|C|]&=\E\left[\sum_{i}\mathds{1}[a_{i}=a_{i}^\star]\right]\leq (S-2)-|D|+\E\left[\sum_{i\in D}\mathds{1}[a_{i}=a_{i}^\star]\right]\\
	&\leq ((S-2)-|D|)+\frac{2}{3}|D|=(S-2)-\frac{1}{3}|D|.
	\end{align*}
	
	If we have $n\leq \frac{(S-2)}{2}\times\frac{A}{72(\tau/H)^2}$, by pigeonhole principle the algorithm can collect $\frac{A}{72(\tau/H)^2}$ samples for at most half of the bandit problems, \textit{i.e.} $|D|\geq (S-2)/2$. Therefore we have
	\[
	\E[|C|]\leq (S-2)-\frac{1}{3}|D|\leq \frac{5}{6}(S-2).
	\]
	Then by Markov inequality 
	\[
	\P\left[|C|\geq \frac{11}{12}(S-2)\right]\leq \frac{5/6}{11/12}=\frac{10}{11}
	\]
	so the algorithm failed to identify the optimal action on 1/12 fraction of the bandit problems with probability at least $1/11$. Note for each failure in identification, the reward is differ by at least $\tau $ in terms of the value for $\hat{v}^\pi$ (see \eqref{eqn:diff_reward}), therefore under the event $\{|C'|\geq \frac{1}{12}(S-2)\}$, the suboptimality of the policy produced by the algorithm is
	\begin{equation}\label{eqn:sub_gap}
	\begin{aligned}
	\epsilon:&=v^\star-v^{\widehat{\pi}}=\P[\text{visit} \;C']\times \tau +\P[\text{visit} \;C]\times 0\geq \P[\bigcup_{i\in C'}\text{visit}(i)]\times\tau \\
	&=\sum_{i\in C'} \P[\text{visit}(i)]\times\tau =\sum_{i\in C'} \frac{1}{S}\tau=\frac{S-2}{S}\tau:=c_1\tau \\
	\end{aligned}
	\end{equation}
	where the third equal sign uses all best arm identification problems are independent. Now we set $\tau=\min(\sqrt{1/8},\epsilon/c_1)$ and under $n\leq cH^2SA/\epsilon^2$, we have
	\[
	n\leq cH^2SA/\epsilon^2\leq c'H^2SA/\tau^2=c'{ 72} S\cdot \frac{A}{72(\tau/H)^2}:=c{''}S\cdot \frac{A}{72(\tau/H)^2}\leq \frac{S-2}{2}\cdot \frac{A}{72(\tau/H)^2},
	\]
	the last inequality holds as long as $S\geq 2/(1-2c'')$. Therefore in this situation, with probability at least $1/11$, $v^\star-v^{\widehat{\pi}}\geq \epsilon$. Finally, we can use scaling to reduce the horizon from $2H$ to $H$.
	
\end{proof}

\begin{remark}
	The suboptimality gap calculation \eqref{eqn:sub_gap} does not use the construction that each $s_i$ has $1-\frac{1}{H}$ probability going back to itself so if we only need Theorem~\ref{thm:learn_low_bound_all} then one can assign all the probability to just $g$ or $b$, which reduces to the construction of Theorem~2 in \citet{dann2015sample}. However, our construction is essential for proving the following offline lower bound.
\end{remark}

\subsection{Information theoretical lower sample complexity bound over problems in $\mathcal{M}_{d_m}$ for identifying $\epsilon$-optimal policy.}  

For all $0<d_m\leq\frac{1}{SA}$, let the class of problems be

$$\mathcal{M}_{d_m} :=\big\{(\mu,M) \; \big| \;\min_{t,s_t,a_t} d_t^\mu(s_t,a_t) \newline \geq d_m\big\}.$$

\begin{theorem}[Restate Theorem~\ref{thm:learn_low_bound_s}]\label{thm:learn_low_bound}
	Under the condition of Theorem~\ref{thm:learn_low_bound_all}. In addition assume $0<d_m\leq\frac{1}{SA}$. There exists another universal constant $c$ such that when $n\leq cH^2/d_m\epsilon^2$, we always have 
	\[
	\inf_{{v}^{\pi_{alg}}}\sup_{(\mu,M)\in\mathcal{M}_{d_m}}\P_{\mu,M}\left(v^*-v^{\pi_{alg}}\geq \epsilon\right)\geq p.
	\]
\end{theorem}

	\begin{figure}[H]
	\centering     
	\includegraphics[width=0.6\linewidth]{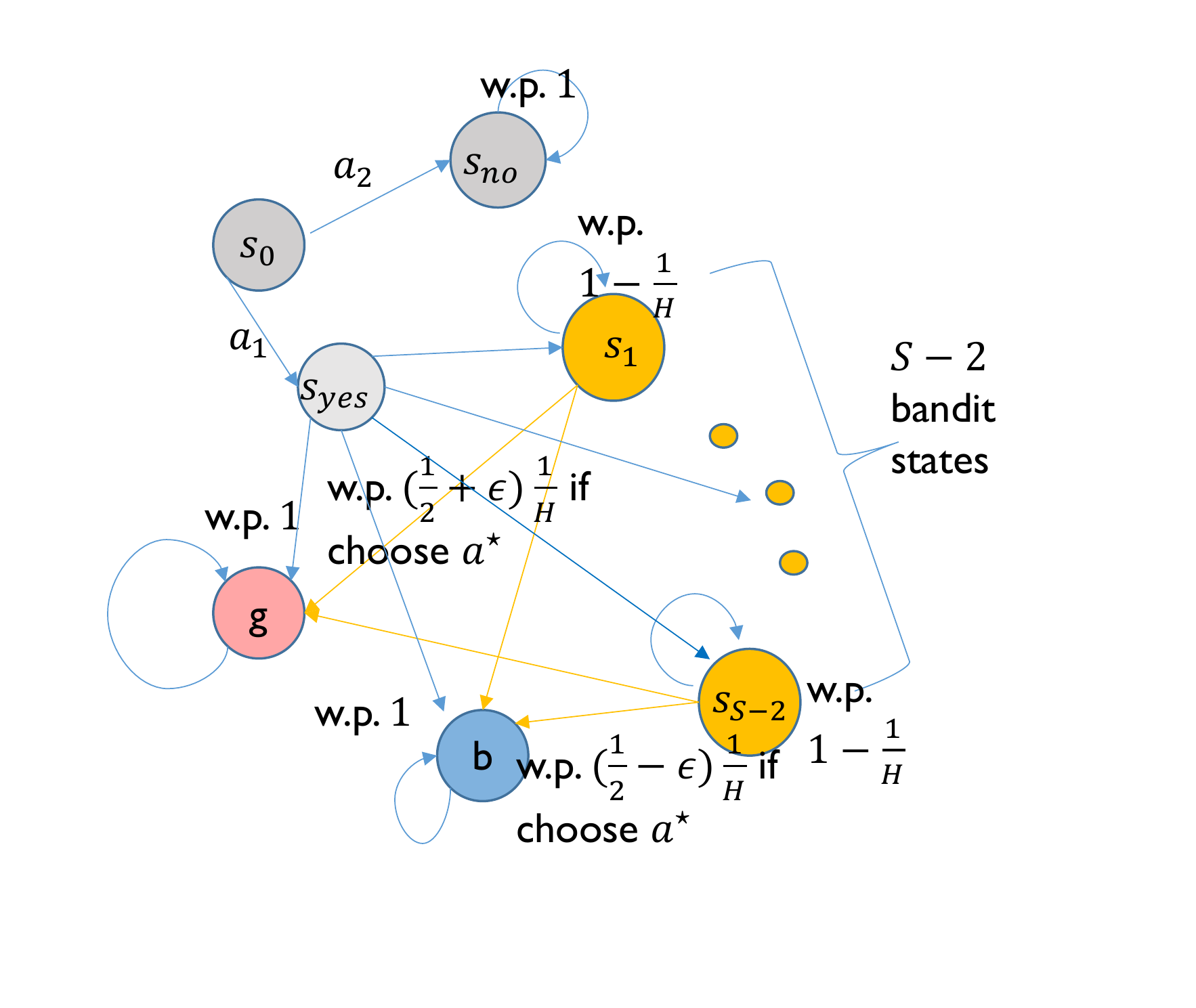}
	\caption{An illustration of transition diagram for Theorem~\ref{thm:learn_low_bound}}
	\label{fig:ill}
\end{figure}

\begin{proof} 
	
	The proof is mostly identical to \cite{yin2020near} except we concatenate all state together to ensure transition is stationary. The hard instances $(\mu,M)$ we used rely on Theorem~\ref{thm:learn_low_bound_all} as follow:
	
	\begin{itemize}
		
		\item for the MDP $M=(\mathcal{S}+3,\mathcal{A}, r,P,d_1,2H)$, 
		\begin{itemize}
			
			\item There are three extra states $s_0, s_\text{yes},s_\text{no}$ in addition to Theorem~\ref{thm:learn_low_bound_all}. Initial distribution $d_1$ will always enter state $s_0$, and there are two actions with action $a_1$ always transitions to $s_\text{yes}$ and action $a_2$ always transitions to $s_\text{no}$. The reward at the first time $r_1(s,a)=0$ for any $s,a$. 
			
			\item For state $s_\text{no}$, it will always transition back to itself regardless of the action and receive reward $0$, \emph{i.e.} 
			\[
			P_t(s_\text{no}|s_\text{no},a)=1,\;r_t(s_\text{no},a)=0,\;\forall t,\;\forall a.
			\] 
			\item For state $s_\text{yes}$, it will transition to the MDP construction in Theorem~\ref{thm:learn_low_bound_all} with horizon $2H$ and $s_\text{yes}$ always receives reward zero (see Figure~\ref{fig:ill}).
			\item For $t=1$, choose $\mu(a_1|s_0)=\frac{1}{2}d_mSA$ and $\mu(a_2|s_0)=1-\frac{1}{2}d_m SA$. For all other states, choose $\mu$ to be uniform policy, \emph{i.e.} $\mu(a_t|s_t)=1/A$.

		\end{itemize}
		
	\end{itemize}
	
	Based on this construction, the optimal policy has the form $\pi^\star=(a_1,\ldots)$ and therefore the MDP branch that enters $s_\text{no}$ is uninformative. Hence, data collected by that part is uninformed about the optimal policy and there is only $\frac{1}{2}d_m SA$ proportion of data from $s_\text{yes}$ are useful. Moreover, by Theorem~\ref{thm:learn_low_bound_all} the rest of Markov chain succeeded from $s_\text{yes}$ requires $\Omega(H^2SA/\epsilon^2)$ episodes (regardless of the exploration strategy/logging policy), so the actual data complexity needed for the whole construction $(\mu,M)$ is $\frac{\Omega(H^2SA/\epsilon^2)}{d_mSA}=\Omega(H^2/d_m\epsilon^2)$. 
	
	It remains to check this construction $\mu,M$ stays within $\mathcal{M}_{d_m}$. The checking is mostly the same as Theorem~G.2. in \citet{yin2020near} so we don't state here. We only highlight the checking for bandit state at different time steps. Indeed, for all $i\in[S-2]$,
		\begin{align*}
		d^\mu_{t+1}(s_{i})&\geq \P^\mu(\underbrace{s_{i},s_{i},\ldots,s_{i}}_{t \;\;times},s_\text{yes},s_0)= \left(\prod_{u=1}^{t}\P^\mu(s_{i}|s_i)\right)\P^\mu(s_{i}|s_\text{yes})\P^\mu(s_\text{yes}|s_0)\\
		&=(1-\frac{1}{H})^t \left(\frac{1}{S}\right)\left(\frac{1}{2}d_m SA\right)\geq c d_m A,\\
		\end{align*}
		now by $\mu$ is uniform we have $d^\mu_{t+1}(s_{t+1,i},a)\geq \Omega(d_mA)\cdot\frac{1}{A}=\Omega(d_m)$ for all $a$. So the condition is satisfied in the stationary transition case. This concludes the proof.

\end{proof}

\section{More details for Discussion Section~\ref{sec:discussion}}\label{sec:app_discussion}

\subsection{Proof of Lemma~\ref{thm:est_d_m}}\label{sec:data_adaptive}

\begin{proof}
  Note data $\mathcal{D}$ comes from the logging policy $\mu$, therefore we can use extra $n(\geq 1/d_m\cdot \log(HSA/\delta))$ episodes to construct direct on-policy estimator as: 
  \[
    \widehat{d}^\mu_t = n_{s_t,a_t}/n.
  \]
  Since $n_{s_t,a_t}$ is binomial, by the multiplicative Chernoff bound (Lemma~\ref{lem:chernoff_multiplicative}), we have 
  \[
    P\left[n_{s_t,a_t}<\frac{1}{2}d^\mu_t(s_t,a_t)n\right]\leq e^{-\frac{d^\mu_t(s_t, a_t) \cdot n}{8}}, \qquad P\left[n_{s_t,a_t}\geq \frac{3}{2}d^\mu_t(s_t,a_t)n\right]\leq e^{-\frac{d^\mu_t(s_t, a_t) \cdot n}{12}}.
  \]
  this implies that for any $(s_t, a_t)$ such that $d^\mu_t(s_t, a_t)>0$, when $n\geq 1/d_m\cdot \log(1/\delta) \ge 1/d^\mu_t(s_t, a_t) \cdot \log(1/\delta)$, we have with probability $1-\delta$ that
  \[
    \frac{1}{2}d^\mu_t(s_t,a_t)\leq \widehat{d}^\mu_t(s_t,a_t)\leq \frac{3}{2}d^\mu_t(s_t,a_t).
  \]
  Applying a union bound, we have the above is true for all $(t,s_t, a_t)$ when $n\geq 1/d_m\cdot \log(HSA/\delta)$. Finally, take $\widehat{d}_m\defeq \min_{(t,s_t, a_t):\widehat{d}^\mu_t(s_t, a_t)>0} \widehat{d}^\mu_t(s_t, a_t)$. On the above concentration event, we get
  \begin{align*}
    \frac{1}{2}d_m \le \widehat{d}_m \le \frac{3}{2}d_m,
  \end{align*}
 by taking $\min$ on all sides.
\end{proof}

\subsection{On relationship between $1/d_m$ and $\beta_\mu,C$}
\label{sec:comparison_dm}


In the function approximation regime, roughly speaking, the \emph{concentration coefficient} assumption requires \cite{munos2003error,le2019batch,chen2019information,xie2020q}
\[
\beta_\mu=\sup_{\pi\in\mathcal{F}}\norm{\frac{d^\pi(s,a)}{d^\mu(s,a)}}_\infty<\infty,
\]
where $\mathcal{F}$ is the policy class induced by approximation functions. In the tabular case, since we want to maximize over all policies, $\mathcal{F}=\{all \;\;policies\}$, therefore above should be interpreted as:
\[
\sup_{\pi \;\text{arbitrary}}\norm{\frac{d^\pi_t(s,a)}{d^\mu_t(s,a)}}_\infty <\infty\Rightarrow {||d^\mu_t(s,a)||_\infty}>0,
\] 

since $\mathcal{F}$ is the largest possible class, if the transition kernel $P(s'|s,a)$ is able to reach some $s'\in\mathcal{S}$ given $s,a$, then that implies $d^\pi_t(s')>0$. Next one can always pick $\pi_{t+1}(s')=a'$ such that $d^\pi_{t+1}(s',a')=d^\pi_t(s')>0$,  for all $a'\in\mathcal{A}$. This means $\mu$ has the chance to explore all states and actions whenever the transition $P$ can transition to all states (from some previous $s,a$).

On the other hand, our Assumption~\ref{assu2} only require $\mu$ to trace at least one optimal policy $\pi^\star$ and it is fine for $\mu$ to never visit certain state-action $s,a$ that is not related to $\mu$.

As a result, since $\beta_\mu$ or $C$ are explicitly incorporated, the upper bounds in \cite{le2019batch,chen2019information,xie2020q} may degenerate to $+\infty$ under our setting (Assumption~\ref{assu2}), regardless of the dependence on horizon.

Nevertheless, we point out that function approximation$+$concentrability assumption is a powerful framework for handling realizability/agnostic case and related concepts (\emph{e.g.} inherent Bellman error) and easier to scale the setting to general continuous case.

\subsection{Improved dependence on $(1-\gamma)^{-1}$ than prior work}
The sample complexity bound $\widetilde{O}((1-\gamma)^{-3}/d_m\epsilon^2)$ in Theorem~\ref{thm:main_infinite} can be compared with the line of recent works on offline RL with function approximation. For example,
\citet{le2019batch}~consider doing batch learning based on fitted Q-iteration with constraints and in their Theorem~4.3 the sample complexity should be translated as $\tilde{O}((1-\gamma)^{-6}\beta_\mu/\epsilon^2)$, where $\beta_\mu$ is the ``concentration factor'' similar to $1/d_m$, but with stronger assumption that $\mu$ explores all $s,a$ that can be visited by the function approximation class. \cite{chen2019information,xie2020q} also consider using FQI in different ways and prove $\epsilon V_{\max}$-optimal policy with sample complexity $\tilde{O}((1-\gamma)^{-4}C/\epsilon^2)$ and $\tilde{O}((1-\gamma)^{-2}C/\epsilon^2)$, where $C$ is again the ``concentration-type coefficient''. Their result should be translated as $\tilde{O}((1-\gamma)^{-6}C/\epsilon^2)$ and $\tilde{O}((1-\gamma)^{-4}C/\epsilon^2)$ for $\epsilon$-optimal policy.

\subsection{The doubling procedure overcomes the proofing defect in \cite{sidford2018near}}
\label{sec:sidford_proof}

\cite{sidford2018near}  first uses \emph{variance reduction} technique to provides provable guarantee for the $\epsilon$-optimal policy. However, their complexity may actually become suboptimal under their initialization. In fact, in their Proof of Proposition 5.4.1. 
(page $23$ of \href{https://arxiv.org/pdf/1806.01492.pdf}{https://arxiv.org/pdf/1806.01492.pdf}), they claim the inequality
\[
\left(\frac{(1-\gamma)^3u^2}{C^{\prime\prime}(1-\gamma)^{8/3}}\right)^{3/4}=C^{\prime\prime}(1-\gamma)^{1/4}u^{1/2}\cdot\leq\frac{u}{16},
\]
which is equivalent to $u\leq O(\sqrt{1/(1-\gamma)})$ (or $u\leq\sqrt{H}$) and based on their initialization $\bm{v}^{(0)}=\mathbf{0}$ they cannot guarantee $\norm{\bm{v}^\star}=\norm{\bm{v}^{(0)}-\bm{v}^\star}_\infty:=u\leq (1-\gamma)^{-1/2}$.
We fix this issue using the doubled Variance Reduction so that minimaxity is preserved for the offline learning with arbitrary initialization.


\section{Technical lemmas }

\begin{lemma}[Best arm identification lower bound \cite{krishnamurthy2016pac}]\label{lem:bandit}
	For any $A \geq 2$ and $\tau \leq \sqrt{1/8}$ and any best arm identification algorithm that produces an estimate $\hat{a}$, there exists a multi-arm bandit problem for which the best arm $a^\star$ is $\tau$  better than all others, but $\P[\hat{a}\neq a^\star]\geq 1/3$ unless the number of samples $T$ is at least $\frac{A}{72\tau^2}$ .
\end{lemma}

\begin{lemma}[Multiplicative Chernoff bound \cite{chernoff1952measure}]\label{lem:chernoff_multiplicative}
	Let $X$ follows Binomial distribution, {i.e.} $X\sim Binom(n,p)$. For any $1\geq\delta>0$, we have that 
	$$
	\P[X < (1-\delta)pn] <  e^{-\frac{\delta^2 pn}{2}}.\qquad \text{and}\qquad \P[X\geq (1+\delta)pn] <  e^{-\frac{\delta^2 pn}{3}}
	$$	
\end{lemma}

\begin{lemma}[Bernstein’s Inequality]\label{lem:bernstein_ineq}
	Let $X_1,...,X_n$ be independent random variables such that $\E[X_i]=0$ and $|X_i|\leq C$. Let $\sigma^2 = \frac{1}{n}\sum_{i=1}^n \mathrm{Var}[X_i]$, then we have 
	$$
	\frac{1}{n}\sum_{i=1}^n X_i\leq \sqrt{\frac{2\sigma^2\cdot\log(1/\delta)}{n}}+\frac{2C}{3n}\log(1/\delta)
	$$
	holds with probability $1-\delta$.
\end{lemma}

\begin{lemma}[Freedman's inequality \cite{tropp2011freedman}]\label{lem:freedman}
	Let $X$ be the martingale associated with a filter $\mathcal{F}$ (\textit{i.e.} $X_i=\E[X|\mathcal{F}_i]$) satisfying $|X_i-X_{i-1}|\leq M$ for $i=1,...,n$. Denote $W:=\sum_{i=1}^n\Var(X_i|\mathcal{F}_{i-1})\leq\sigma^2$  then we have 
	\[
	\P(|X-\E[X]|\geq\epsilon)\leq 2 e^{-\frac{\epsilon^2}{2(\sigma^2+M\epsilon/3)}}.
	\]
	Or equivalently, with probability $1-\delta$,
	\[
	|X-\E[X]|\leq \sqrt{{8\sigma^2\cdot\log(1/\delta)}}+\frac{2M}{3}\cdot\log(1/\delta).
	\]
\end{lemma}

\begin{lemma}\label{lem:H3}
	Let $r^{(1)}_t,s^{(1)}_t,a^{(1)}_t$ denotes random variables. Then the following decomposition holds:
	\begin{equation}\label{eqn:11}
	\begin{aligned}
	&\mathrm{Var}_\pi\left[\sum_{t=h}^H r^{(1)}_t\middle|s^{(1)}_h=s_h,a^{(1)}_h=a_h\right] = \sum_{t=h}^H \Big(\E_\pi\left[ \mathrm{Var}\left[r^{(1)}_t+v^\pi_{t+1}(s_{t+1}^{(1)}) \middle|s^{(1)}_t,a^{(1)}_t\right] \middle|s^{(1)}_h=s_h,a^{(1)}_h=a_h\right]\\
	&\quad +  \E_\pi\left[ \mathrm{Var}\left[  \E[r^{(1)}_t+v^\pi_{t+1}(s_{t+1}^{(1)}) | s^{(1)}_t, a^{(1)}_t]  \middle|s^{(1)}_t\right]\middle|s^{(1)}_h=s_h,a^{(1)}_h=a_h \right]\Big).
	\end{aligned}
	\end{equation}
\end{lemma}

\begin{remark}
	This is a conditional version of Lemma~3.4 in \cite{yin2020asymptotically}. It can be proved using the identical trick as Lemma~3.4 in \cite{yin2020asymptotically} except the law of total variance is replaced by the law of total conditional variance.
\end{remark}

\end{document}